\documentclass[10pt,journal,compsoc]{IEEEtran}

\usepackage{microtype}
\usepackage{graphicx}
\usepackage{bm}
\usepackage{algorithm}
\usepackage{algorithmic}
\usepackage{multirow}
\usepackage{amsthm}
\usepackage{amsmath}
\usepackage{subfigure}
\usepackage{amssymb}
\usepackage{booktabs} 

\usepackage{balance}
\usepackage{amsmath}
\usepackage{float}
\usepackage{subeqnarray}
\usepackage{cases}
\usepackage{tabularx}
\usepackage{color}
\usepackage{lscape}

\usepackage{threeparttable}
\usepackage{arydshln}

\newtheorem{thm}{Theorem}

\newtheorem{prop}{Proposition}
\newtheorem{assum}{Assumption}
\newtheorem{lemma}{Lemma}

\newtheorem{remark}{Remark}

\newtheorem{example}{Example}

\def\x{\mathbf{x}}
\def\y{\mathbf{y}}
\def\z{\mathbf{z}}

\def\X{\mathcal{X}}
\def\Y{\mathcal{Y}}
\def\S{\mathcal{S}}
\def\T{\mathcal{T}}

\newcommand{\tabincell}[2]{\begin{tabular}{@{}#1@{}}#2\end{tabular}}

\usepackage[colorlinks,linkcolor=black]{hyperref}

\begin{document}

\title{A General Descent Aggregation Framework for Gradient-based Bi-level Optimization}

\author{Risheng~Liu,~\IEEEmembership{Member,~IEEE,}
        Pan~Mu,~
        Xiaoming~Yuan,~
        Shangzhi~Zeng
        and~Jin~Zhang
\IEEEcompsocitemizethanks{\IEEEcompsocthanksitem Risheng Liu is with the DUT-RU International School of Information Science $\&$ Engineering, Dalian University of Technology, Dalian 116024, China. E-mail: {rsliu}@dlut.edu.cn.\protect
\IEEEcompsocthanksitem P. Mu is with the College of Computer Science and Technology, Zhejiang University of Technology, Hangzhou 310023, China, and also with the School of Mathematical Sciences, Dalian University of Technology, Dalian 116024, China. E-mail: {panmu1}@outlook.com. \protect
\IEEEcompsocthanksitem X. Yuan is with the Department of Mathematics, The University of Hong Kong, Hong Kang, China. E-mail: {xmyuan}@hku.hk.\protect
\IEEEcompsocthanksitem S. Zeng is with the Department of Mathematics and Statistics, University of Victoria, Canada. E-mail: {zengshangzhi}@uvic.ca.\protect
\IEEEcompsocthanksitem J. Zhang is with the Department of Mathematics, SUSTech International Center for Mathematics, Southern University of Science and Technology, National Center for Applied Mathematics Shenzhen, Shenzhen, Guangdong, China. (Corresponding author, E-mail: {zhangj9}@sustech.edu.cn.)\protect\\
}

\thanks{Manuscript received April 19, 2020; revised August 26, 2015.}}

\markboth{Journal of \LaTeX\ Class Files,~Vol.~14, No.~8, August~2015}%
{Shell \MakeLowercase{\textit{et al.}}: A Generic Descent Aggregation Framework for Gradient-based Bi-level Optimization}

\IEEEtitleabstractindextext{
\begin{abstract}
In recent years, a variety of gradient-based methods have been developed to solve Bi-Level Optimization (BLO) problems in machine learning and computer vision areas. However, the theoretical correctness and practical effectiveness of these existing approaches always rely on some restrictive conditions (e.g., Lower-Level Singleton, LLS), which could hardly be satisfied in real-world applications. Moreover, previous literature only proves theoretical results based on their specific iteration strategies, thus lack a general recipe to uniformly analyze the convergence behaviors of different gradient-based BLOs. In this work, we formulate BLOs from an optimistic bi-level viewpoint and establish a new gradient-based algorithmic framework, named Bi-level Descent Aggregation (BDA), to partially address the above issues. Specifically, BDA provides a modularized structure to hierarchically aggregate both the upper- and lower-level subproblems to generate our bi-level iterative dynamics. Theoretically, we establish a general convergence analysis template and derive a new proof recipe to investigate the essential theoretical properties of gradient-based BLO methods. Furthermore, this work systematically explores the convergence behavior of BDA in different optimization scenarios, i.e., considering various solution qualities (i.e., global/local/stationary solution) returned from solving approximation subproblems. Extensive experiments justify our theoretical results and demonstrate the superiority of the proposed algorithm for hyper-parameter optimization and meta-learning tasks. Source code is available at \href{https://github.com/vis-opt-group/BDA}{https://github.com/vis-opt-group/BDA}. 
\end{abstract}

}

\maketitle

\IEEEraisesectionheading{\section{Introduction}\label{sec:introduction}}

\IEEEPARstart{B}{i-Level} Optimization (BLO) are a class of mathematical programs with optimization problems in their constraints. Recently, thanks to the powerful modeling capabilities, BLO have been recognized as important tools for a variety of machine learning and computer vision applications~\cite{liu2020generic,rajeswaran2019meta,mackay2019self,kunisch2013bilevel,liu2021investigating}. Mathematically, BLO can be formulated as 
\begin{equation}
\min\limits_{\mathbf{x}\in\mathcal{X},\mathbf{y}\in\mathcal{Y}}F(\x,\y), \ s.t. \ \mathbf{y}\in\S(\x):=\arg\min\limits_{\y} f(\x,\y),\label{eq:blp}
\end{equation}
where the Upper-Level (UL) objective $F$ and the Lower-Level (LL) objective $f$ both are jointly continuous function, the UL constraint $\mathcal{X}$ is a compact set, the set-valued mapping $\mathcal{S}(\mathbf{x})$ indicates the solution set of the LL subproblem parameterized by $\mathbf{x}$, and $\mathcal{Y} \subseteq \mathrm{dom} F$ is a compact convex set. Indeed, the BLO model in Eq.~\eqref{eq:blp} is a hierarchical optimization problem with two coupled variables $(\x,\y)\in\mathbb{R}^n\times\mathbb{R}^{m}$ which need to be optimized simultaneously. This makes the computation of an optimal solution a challenging task~\cite{dempe2018bilevel}. To overcome such an unpleasant situation, from an optimistic BLO viewpoint\footnote{For more theoretical details of optimistic BLO, we refer to~\cite{dempe2018bilevel,dempe2019two} and the references therein.}, we decouple Eq.~\eqref{eq:blp} as optimizing UL variable $\x$ and LL variable $\y$ separately. Specifically, for any given $\x$, we expect that the LL solution $\y\in\S(\x)$ also leads to the best UL objective value (i.e., $F(\x,\cdot)$) simultaneously. For this purpose, following the optimistic BLO idea, we incorporate some taste of hierarchy regarding the LL variable $\y$, and Eq.~\eqref{eq:blp} is thus reformulated as
\begin{equation}
\begin{array}{l}
\min\limits_{\x\in\X}\varphi(\x), \
\mbox{with} \ \varphi(\x) := \inf\limits_{\y\in \Y \cap \S(\x)}F(\x,\y).\label{eq:oblp}
\end{array}
\end{equation}
Actually, the above stated optimistic viewpoint is general and has received extensive attentions in BLO literature~\cite{dempe2007new,kohli2012optimality,lampariello2020numerically,Liu2021TowardsGB}. Such reformulation reduces BLO to a single-level problem $\min_{\mathbf{x}\in\mathcal{X}}\varphi(\x)$ w.r.t. the UL variable $\x$. Although early works on BLO can date back to the nineteen seventies~\cite{dempe2018bilevel}, it was not until the last decade that a large amount of bi-level optimization models were established to capture vision and machine learning applications, including meta learning~\cite{franceschi2018bilevel,rajeswaran2019meta,zugner2019adversarial}, hyper-parameter optimization~\cite{franceschi2017forward,okuno2018hyperparameter,mackay2019self}, reinforcement learning~\cite{yang2019provably}, neural architecture search~\cite{liu2018darts,wu2019fbnet,nakai2020att,hu2020tf} and image processing~\cite{kunisch2013bilevel,de2017bilevel,chen2020flexible,Liu2020investigating,liu2021investigating}, and etc.

Due to the hierarchical structure and the sophisticated dependency between UL and LL variables, solving BLO is challenging in general, especially when the LL solution set $\S(\x)$ is not a singleton~\cite{jeroslow1985polynomial,dempe2018bilevel}. Actually, the most straightforward idea in existing learning and vision literature is to assume that $\S(\x)$ is a singleton. Formally, we call the BLO model is with the Lower-Level Singleton (LLS) condition if $\forall\x\in\X$. Under this condition, a variety of Gradient-based Bi-level Methods (GBMs) have been developed to solve BLOs in different machine learning and computer vision applications. 

The key idea behind these existing GBMs is to solve BLOs with an approximated Best Response (BR) Jacobian (i.e., the gradient of the best response mapping w.r.t. the UL variable $\x$). From this perspective, we can roughly categorize existing GBMs into two groups, i.e., explicit and implicit BR methods. For explicit BR methods, the BR gradients are obtained by automatic differentiation~\cite{weinan2017proposal} through iterations of the LL gradient descent. This explicit structure mainly includes: recurrence-based~\cite{maclaurin2015gradient,franceschi2017forward,franceschi2018bilevel,shaban2018truncated}, initialization-based~\cite{finn2017model,nichol2018first} and proxy-based scheme~\cite{lorraine2018stochastic,mackay2019self}. Specifically, recurrence-based BR first calculate gradient representations of the LL objective and then perform either reverse or forward gradient computations (a.k.a., automatic differentiation, based on the LL gradients) for the UL subproblem. In~\cite{finn2017model,nichol2018first}, known for its simplicity and state of the art performance, initialization-based structure estimated a good initialization of model parameters for the fast adaptation to new tasks purely by a gradient-based search. For proxy-based scheme~\cite{lorraine2018stochastic,mackay2019self}, a so-called hyper-network is trained to map LL gradients for their hierarchical optimization. These explicit methods only rely on the gradient information of the LL subproblem to update LL variable that cannot cover the UL descent information. On the other hand, implicit BR methods~(\cite{rajeswaran2019meta,grazzi2020iteration,lorraine2020optimizing,Bertrand2020implicit,liu2021value} and~\cite{pedregosa2016hyperparameter}) are designed based on the observation that it is possible to replace the LL subproblem by an implicit equation. These implicit methods derive their BR gradients but involve computing a Hessian matrix and its inverse, which could be computationally expensive and unstable for large-scale problems.

Note that, for existing methods within the explicit BR category, the explicit BR approximation by optimization iteration dynamics raises an issue regarding approximation quality. In fact, without the LLS assumption, the dynamics procedures of existing methods, in general may not be good approximations. This is because in this case, the optimization dynamics converge to some minimizers of the LL objective, but not necessarily to the one that also minimizes the UL objective. This unpleasant situation was noticed by both the machine learning and the optimization communities; see, e.g., \cite[Section 3]{franceschi2018bilevel}. In theory, research on the theoretical convergence is still in its infancy (as summarized in Table~\ref{tab:result}). Indeed, all the mentioned GBMs require the LLS condition in LL subproblem to simplify their optimization processes and gain theoretical guarantees. For example, the works in~\cite{franceschi2018bilevel,shaban2018truncated} enforce the strong convexity assumption to the LL subproblem. Unfortunately, it has been demonstrated that such LLS assumption is too restrictive to be satisfied in most real-world learning and vision applications. Further, these existing methods only concern the convergence towards stationary or global/local minimum, thus lack comprehensive convergence analyses.

In response to these limitations, this work proposes a novel framework termed Bi-level Descent Aggregation (BDA). Specifically, we propose a gradient type method for solving BLOs by aggregating UL and LL objectives. Theoretically, this work provides a general proof recipe as a basic template for the convergence analysis. In particular, in the absence of LLS, the BDA convergence was strictly guaranteed as long as the embedded inner simple bi-level dynamics meet the so-called \emph{UL objective convergence property} and \emph{LL objective convergence property}; see Section~\ref{sec:conv_recipes} for details. Specifically, we construct dynamics for optimizing the inner simple bi-level subproblem and hence achieve a justified good approximation. By using some variational analysis techniques sophisticatedly, the new optimization dynamics are shown to meet \emph{UL objective convergence property} and \emph{LL objective convergence property} without imposing any strong convexity assumptions in either UL or LL subproblems. Thanks to the new proof recipe, we provide the convergence results, which are classified by global/local solution cases returned from solving the approximation subproblems (i.e., $\min_{\x}\varphi_{K}(\x)$). Besides, if solving an approximation subproblems (i.e., $\min_{\x}\varphi_{K}(\x)$) returns (approximate) stationarity, we demonstrate the stationarity convergence result under the designed algorithm scheme (i.e., BDA). Moreover, as can be seen in Table~\ref{tab:result}, a striking feature of our study is that all the sufficient conditions we use to meet the desired convergence are easily verifiable for practical learning applications. We designed a high-dimensional counter-example with a series of complex experiments to verify our theoretical investigations and explore the intrinsic principles of the proposed algorithms. Extensive experiments also show the superiority of our method for different tasks, including hyper-parameter optimization and meta learning. We summarize the contributions of this work as follows.
\begin{itemize}
	\item By designing a gradient-aggregation strategy to formulate the inner simple bi-level dynamics, we provide a new algorithmic framework to handle the LLS issue, which has been widely witnessed, but related research is still missing among existing gradient-based BLO approaches. 
	\item We establish a general convergence analysis template together with an associated proof recipe for BDA. This new proof technique enhances our understanding of the essence of gradient-based method's convergence, hence helps to eliminate the UL strong convexity assumption, which is required in~\cite{liu2020generic}.\footnote{A preliminary version of this work has been published in [1].} 
	\item We provide a comprehensive theoretical convergence analysis of the developed algorithm. Focusing on different solution qualities (namely, global/local/stationary solutions), we elaborate the convergence properties respectively, thus significantly extend results in~\cite{liu2020generic}. 
	\item As a nontrivial byproduct, the iterative gradient-aggregation dynamics (i.e., Eq.~\eqref{eq:improved-lower}) are of independent interest in convex optimization. They can be identified as a new iterative optimization scheme for solving the simple bi-level problem without the UL strong convexity. 
\end{itemize}

\section{Gradient-based BLOs: A Brief Review}

	As for the BLO model in Eq.~\eqref{eq:oblp}, it is worthwhile noting that the LL solution set $\S(\x)$ may have multiple solutions for every (or some) fixed $\x$. However, it is challenging to solve BLO, especially when the LL solution set is not a singleton. Thus, in learning and vision application scenarios, the most straightforward idea of designing GBMs is to enforce the singleton assumption on  $\S(\x)$. With such LLS condition, the BLO model in Eq.~\eqref{eq:oblp} actually can be simplified as follows:
	\begin{equation}\label{eq:blp_lls}
		\min_{\x\in\X}\varphi(\x):=F(\x,\y^{\ast}(\x)),\ s.t.,\ \y^{\ast}=\ \underset{\y \in \Y}{\mathrm{argmin}} f(\x,\y).
	\end{equation} 
	Thus the optimization task reduces to solve a single-level problem (i.e., $\min_{\x\in\X}\varphi(\x)=F(\x,\y^{\ast}(\x))$) with an optimal LL variable $\y^{\ast}$. In this way, the gradient of $\varphi$ (w.r.t., $\x$) can be written as
	\begin{equation}\label{eq:grad_x}
		\underbrace{\nabla\varphi(\x)}_{\text{grad. of}\ \x} = \underbrace{\nabla_{\x} F(\x,\y^{\ast}(\x))}_{\text{direct grad of}\ \x} +  \underbrace{\left(\frac{\partial\y^{\ast}(\x)}{\partial\x}\right)^{\top}}_{\text{BR Jacobian}}\underbrace{\nabla_{\y} F(\x,\y^{\ast})}_{\text{direct grad. of}\ \y},
	\end{equation} 
	where ``grad." denotes the abbreviation of gradient and $(\cdot)^{\top}$ means the transpose operation. In existing GBMs, they actually first numerically approximate $\y^*$ by $\y_K$ and thus define $\varphi_K(\x)=F(\x,\y_K(\x))$. Then the UL variable $\x$ can be updated based on the following practical formulation\footnote{Please refer to a recent survey in~\cite{liu2021investigating} for more details on GBMs in leaning and vision areas.} 
	\begin{equation}\label{eq:grad_varphiK_x}
	\nabla \varphi_K(\x)\!=\!\nabla_{\x} F(\x,\y_K(\x))\!+\! \left(\frac{\partial\y_K(\x)}{\partial\x}\right)^{\top}\!\nabla_{\y}F(\x,\y_K(\x)).
	\end{equation}

In particular, a variety of techniques~\cite{franceschi2017forward,franceschi2018bilevel,shaban2018truncated,finn2017model,nichol2018first,lorraine2018stochastic,mackay2019self} have been developed to explicitly formulate $\y_K$ using dynamic systems. For example, by enforcing the LLS assumption on the BLO problem and considering $\x$ as the recurrent parameters of a gradient-based dynamic system, i.e., $\y_k = \T_k(\x,\y_{k-1})$ with $\T_k(\x,\y_{k-1})=\y_{k-1} - \eta\nabla_{\y}f(\x,\y_{k-1}(\x))$, these methods first calculate gradient representations of the LL objective and then perform either reverse/forward or automatic differentiation to obtain Eq.~\eqref{eq:grad_varphiK_x}. However, the dynamic system generated by these GBMs can only reveal gradients of the LL subproblem, but completely miss descent information from the UL objective.

Theoretically, these existing convergence results all require that the LL dynamics $\{\y_K(\x)\}$ is uniformly bounded on $\X$ and $\y_K(\x)$ uniformly converges to $\y^{\ast}(\x)$ as $K\to\infty$. We should also point out that the LLS condition actually plays the key role for most of existing GBMs (e.g., \cite{franceschi2018bilevel,pedregosa2016hyperparameter}). These approaches often require restrictive assumptions (e.g., strong convexity) to meet this assumption. Besides, some researches~\cite{mackay2019self,shaban2018truncated,fallah2020convergence} prove that UL value-function converges to a first-order stationary point, i.e., $\lim_{T\to\infty}\nabla\varphi(\x_T)\to 0$. To achieve the stationarity, they require the first-order Lipshitz assumption for the UL and LL objectives, the twice continuously differentiable property for the LL objective and some additional restrictive assumptions, such as nonsingular Hessian assumption in~\cite{mackay2019self}.

\section{Bi-level Descent Aggregation}\label{sec:alg}

In this section, we establish a general algorithmic framework to solve BLOs formulated in ~Eq.~\eqref{eq:oblp}. In particular, by incorporating the numerical BR mapping $\y_K(\x)$ into Eq.~\eqref{eq:oblp}, we actually aim to solve the following approximated single-level optimization model:
\begin{equation}
\min_{\mathbf{x}\in\X} \varphi_K(\x). \label{eq:upper_varphiK}
\end{equation}
It should be emphasized that different from these existing GBMs stated above, which only use the information of the LL subproblem to generate $\varphi_K(\x)$ (i.e., obtain $\y_K(\x)$), here we formulate it as the value function of the following inner simple bi-level model:
\begin{equation}\label{eq:simple-blp}
\min_{\y \in \Y}F(\x,\y), \ s.t. \ \y\in\S(\x). 
\end{equation} 
Let $\mathcal{T}_k(\x,\cdot)$\footnote{In fact, our theoretical analysis in Section~\ref{sec:conv_recipes} will introduce two essential properties, which can be used as guidance for designing $\T_k$. In other words, any $\T_k$ satisfying these two properties all can be used as our fundamental modules.} stand for a schematic iterative module originated from a certain simple bi-level solution strategy on Eq.~\eqref{eq:simple-blp} (with a fixed UL variable $\x$). Then we can write the general updating rule of $\y$ as follows:
	\begin{equation}
		\mathbf{y}_{k+1}(\x)=\mathcal{T}_{k+1}(\x,\y_{k}(\x)), \ k=0,\cdots,K-1,\label{eq:update-t}
	\end{equation}
where $\mathbf{y}_{0}(\x) = \mathbf{y}_{0}$ is the initialization based on $\mathbf{y}_{0} \in \Y$.  
For the particular form of $\mathcal{T}_k(\x,\cdot)$, here we would like to aggregate both the UL and LL subproblems to define it. Specifically, for a given $\x$, we write the descent directions of the UL and LL objectives as 
\begin{equation*}
\begin{array}{c}
		\mathbf{d}^{{F}}_k(\mathbf{x})=s_u\nabla_{\mathbf{y}} F(\mathbf{x},\mathbf{y}_{k}(\mathbf{x})),\
	\mathbf{d}^{{f}}_k(\mathbf{x})=s_l\nabla_{\mathbf{y}} f(\mathbf{x},\mathbf{y}_{k}(\mathbf{x})),
\end{array}
\end{equation*}
where $s_u$, $s_l$ denote the corresponding step size parameters. Then we consider the following aggregated updating scheme as $\mathcal{T}_k(\x,\cdot)$, i.e., 
\begin{equation}\label{eq:improved-lower}
	\begin{array}{l}
		\T_{k+1}\left(\x,\y_k(\x)\right) \\
		\quad = \mathtt{Proj}_{\Y}\left( \y_k(\mathbf{x}) -\left( \mu \alpha_k\mathbf{d}^{{F}}_k(\mathbf{x})+(1-\mu)\beta_k\mathbf{d}^{{f}}_k(\mathbf{x})\right) \right),
	\end{array}
\end{equation}
where $\mathtt{Proj}_{\Y}$ denotes the projection on $\Y$,  $\mu \in (0,1)$ and $\alpha_k,\beta_k\in(0,1]$ are the aggregation parameters and $k = 0,\ldots, K-1$. 
Here we should point out that the iteration scheme in~\cite[Eq.~(10)]{liu2020generic} is just a specific case of Eq.~\eqref{eq:improved-lower} with $\beta_k = (1-\mu\alpha_k)/(1-\mu)$. 

As for solving the single-level problem in Eq.~\eqref{eq:upper_varphiK}, we state that this UL optimization step straightforwardly follows the standard (stochastic) gradient scheme, which has been widely investigated in literature; see, e.g.,~\cite{maclaurin2015gradient,franceschi2017forward,franceschi2018bilevel,shaban2018truncated}. To close this section, we summarize the overall BDA scheme in the following Algorithm~\ref{alg:BDA}.\footnote{When the identity of UL iteration step $t$ is clear from the context, we omit the superscript $t$ and write $\x$ instead of $\x^t$. To present the algorithm steps in a more explicit form, we provide detailed iterations with the superscript $t$ of $\x$ in Algorithm 1; see the experiment in Sections~\ref{subsec:numerical_results}. }

\begin{algorithm}[htb!]
	\renewcommand{\algorithmicrequire}{\textbf{Input:}}
	\renewcommand{\algorithmicensure}{\textbf{Output:}}
	\caption{Bi-level Descent Aggregation Framework}\label{alg:BDA}
	\begin{algorithmic}[1]
		\REQUIRE The necessary parameters and initialization.
		\ENSURE The optimized $\x$, $\y$.
		\STATE $t=0$.  
		\WHILE{Not Converge}
		\FOR{$k=0$ to $K-1$}
		\STATE $\%$ LL updating (line 5--8)
		\STATE   $\mathbf{d}^{{F}}_k(\x^t)=s_u\nabla_{\mathbf{y}} F(\mathbf{x}^t,\mathbf{y}_{k}(\x^t)),$\label{step:d_F}
		\STATE $\mathbf{d}^{{f}}_k(\x^t)=s_l\nabla_{\mathbf{y}} f(\mathbf{x}^t,\mathbf{y}_{k}(\x^t)),$\label{step:d_f}
		\STATE  $\hat{\y}_{k+1}(\x^t) =\y_k(\x^t) -( \mu \alpha_k\mathbf{d}^{{F}}_k(\x^t)+(1-\mu)\beta_k\mathbf{d}^{{f}}_k(\x^t)),$
		\STATE $\y_{k+1}(\x^t) = \mathtt{Proj}_{\Y}\left( \hat{\y}_{k+1}(\x^t) \right)$.\label{step:guide-d}
		\ENDFOR	
		\STATE $\%$ UL updating (line 11)
		\STATE $\x^{t+1} = \mathtt{Proj}_{\X}(\x^t - \lambda\nabla\varphi_K(\x^t))$.
		\STATE $t=t+1$.
		\ENDWHILE
	\end{algorithmic}
\end{algorithm}

\begin{remark}
		First of all, we emphasize that it will be demonstrated in the following sections that the main scope of introducing set constraints (i.e., $\x\in\X$ and $\y\in\Y$) in Eq.~\eqref{eq:oblp} is to guarantee the completeness of our theoretical analysis. Thus in most optimization scenarios, we can straightforwardly define large enough $\X$ and $\Y$ (e.g., the whole space) to make the projection operation $\mathtt{Proj}$ inactive during our iterations. 
		Besides, even if it requires to explicitly consider the set constraints for some specific applications, we actually simply introduce Clarke subdifferential (see~\cite{bolte2020mathematical} for detailed definition) for the projection operation during iterations.
\end{remark}

\section{A General Convergence Analysis Recipe}\label{sec:conv_recipes}

This part aims to provide a general convergence analysis recipe for GBMs (not only BDA, but also these existing approaches). That is, we first introduce two essential convergence properties for the UL and LL subproblems and then establish a general convergence analysis template to investigate the theoretical properties of gradient-based bi-level iterations\footnote{We suggest readers to refer to~\cite{liu2020generic} and the references therein for necessary definitions used in our convergence analysis. }. 

To conduct the convergence analysis, we first make the following standing assumption. 
\begin{assum}\label{assum:F}
	$F(\x,\y)$, $\nabla_y F(\x,\y)$, $f(\x,\y)$ and $\nabla_y f(\x,\y)$ are continuous on $\X\times\mathbb{R}^m$.	For any $\x \in \X$, $F(\x,\cdot) : \mathbb{R}^m \rightarrow \mathbb{R}$ is $L_F$-smooth, convex and bounded below by $M_0$, $f(\x,\cdot) : \mathbb{R}^m \rightarrow \mathbb{R}$ is $L_f$-smooth and convex.
\end{assum}

\subsection{Two Essential Convergence Properties}

Now we are ready to establish the new convergence analysis template, which describes the main steps to achieve the converge guarantees for our bi-level updating scheme (stated in Eqs.~\eqref{eq:upper_varphiK}-\eqref{eq:update-t}, with a schematic $\T_k$). Basically, our proof recipe is based on the following two essential properties:
\begin{enumerate}
	\item[(1)]\textbf{UL objective convergence property:} For each $\x \in \mathcal{X}$,
	\begin{equation*}
	\lim\limits_{K \rightarrow \infty}\varphi_K(\x) \rightarrow \varphi(\x).\label{eq:dist-varphi}
	\end{equation*}
	\item[(2)]\textbf{LL objective convergence property:} $\{\y_{K}(\x)\}$ is uniformly bounded on $\X$, and for any $\epsilon>0$, there exists $k(\epsilon)>0$ such that whenever $K>k(\epsilon)$, 	
	\begin{equation*}
	\sup_{\x \in \X}\left\{ f(\x,\y_K(\x)) - f^{\ast}(\x) \right\} \le \epsilon.
	\end{equation*}
\end{enumerate}

Indeed, the general recipe provides us a criterion to design different stable algorithms. Under these two essential properties, we thoroughly analyze the bi-level optimization problem and provide comprehensive theoretical results. Based on the developed BDA algorithm scheme, we first provide convergence results towards global and local minimum in Section~\ref{subsec:global} and Section~\ref{subsec:local} respectively. Specifically, if a series of global/local solutions of approximation subproblems are found, then a global solution of the original bi-level problem can be approximately achieved.

\subsection{Towards Global Minimum}\label{subsec:global}

Thanks to the continuity of $f(\x,\y)$, we have the same semi-continuity over partial minimization as in \cite{liu2020generic}. In other words, with the continuity of $f(\x,\y)$, we have that $f^{\ast}(\x):=\min_{\y}f(\x,\y)$ is Upper Semi-Continuous (USC for short) on $\X$. Equipped with the above two properties (i.e., \emph{UL objective convergence property} and \emph{LL objective convergence property}), we establish our general convergence results in the following theorem for the schematic bi-level scheme in Eqs.~\eqref{eq:upper_varphiK}-\eqref{eq:update-t}. 
	\begin{thm}\label{thm:general}(Convergence towards Global Minimum)
		Suppose both the above UL and LL objective convergence properties hold and $f(\x,\y)$ is continuous on $\X\times\mathbb{R}^m$. 
		Let $\x_K$\footnote{This subscript $K$ just corresponds to the subscript of $\varphi_K$.} be a $\varepsilon_K$-minimum of $\varphi_{K}(\x)$, i.e.,
		\begin{equation*}
		\begin{array}{c}
		\varphi_K(\x_K) \le \varphi_K(\x) + \varepsilon_K, \quad \forall \x \in \X.
		\end{array}
		\end{equation*}
		Then if $\varepsilon_K \rightarrow 0$, we have
		\begin{itemize}
			\item[(1)] Any limit point $\bar{\x}$ of the sequence $\{\x_K\}$ satisfies that $\bar{\x}\in\arg\min_{\x\in\X}\varphi(\x)$. 
			\item[(2)] $\inf_{\x \in \X}\varphi_K(\x) \rightarrow \inf_{\x \in \X} \varphi(\x)$ as $K \rightarrow \infty$. 
		\end{itemize} 
	\end{thm}
	\begin{proof}
		For any limit point $\bar{\x}$ of the sequence $\{\x_K\}$, let $\{\x_{l}\}$ be a subsequence of $\{\x_K\}$ such that $\x_{l} \rightarrow \bar{\x} \in \X$. As $\{\y_K(\x)\}$ is uniformly bounded on $\X$, we can have a subsequence $\{\x_{m}\}$ of $\{\x_{l}\}$ satisfying $\y_m(\x_m) \rightarrow \bar{\y}$ for some $\bar{\y}$. It follows from the \emph{LL objective convergence property} that for any $\epsilon > 0$, there exists $M(\epsilon) > 0$ such that for any $m > M(\epsilon)$, we have
		\begin{equation*}
		f(\x_m,\y_m(\x_m)) - f^{\ast}(\x_m) \le \epsilon.
		\end{equation*}
		By letting $m \rightarrow \infty$, and since $f$ is continuous and $f^{\ast}(\x)$ is USC on $\X$, we have $f(\bar{\x},\bar{\y}) - f^{\ast}(\bar{\x}) \le \epsilon$. As $\epsilon$ is arbitrarily chosen, we have $f(\bar{\x},\bar{\y}) - f^{\ast}(\bar{\x}) \le 0$ and thus $\bar{\y} \in \S(\bar{\x})$.
		Next, as $F$ is continuous at $(\bar{\x},\bar{\y})$, for any $\epsilon > 0$, there exists $M(\epsilon) > 0$ such that for any $m > M(\epsilon)$, it holds
		\begin{equation*}
		F(\bar{\x},\bar{\y}) \le F(\x_m,\y_m(\x_m)) + \epsilon.
		\end{equation*}
		Then, we have, for any $m > M(\epsilon)$ and $\x \in \X$,
		\begin{equation}\label{eq1}
		\begin{array}{r}
		\ \varphi(\bar{\x}) = \inf_{\y \in \S(\bar{\x}) } F(\bar{\x}, \y)
		\le F(\x_m,\y_m(\x_m)) + \epsilon \\
		\quad\le \varphi_m(\x) + \epsilon + \varepsilon_m\ .\quad\quad
		\end{array}
		\end{equation}
		Taking $m \rightarrow \infty$ and by the \emph{UL objective convergence property} and $\varepsilon_m \rightarrow 0$, we have
		\begin{equation*}
		\begin{aligned}
		\varphi(\bar{\x}) \le \lim_{m \rightarrow \infty}\varphi_m(\x) + \epsilon + \varepsilon_m = \varphi(\x) + \epsilon, \ \forall \x \in \X.
		\end{aligned}
		\end{equation*}
		By taking $\epsilon \rightarrow 0$, we have $\varphi(\bar{\x}) \le \varphi(\x), \ \forall \x \in \X$ which implies $\bar{\x} \in \arg\min_{\x \in \X} \varphi(\x)$.
		
		We next show that $\inf_{\x \in \X}\varphi_K(\x) \rightarrow \inf_{\x \in \X} \varphi(\x)$ as $K \rightarrow \infty$. For any $\x \in \X$, $\inf_{\x \in \X}\varphi_K(\x) \le \varphi_K(\x)$, by taking $K \rightarrow \infty$ and 
		with the \emph{UL objective convergence property}, we have
		\begin{equation*}
		\limsup_{K \rightarrow \infty} \left\{\inf_{\x \in \X}\varphi_K(\x) \right\} \le \varphi(\x), \ \forall \x \in \X,
		\end{equation*}
		and thus
		\begin{equation*}
		\limsup_{K \rightarrow \infty} \left\{\inf_{\x \in \X}\varphi_K(\x) \right\}\le \inf_{\x \in \X} \varphi(\x).
		\end{equation*}
		So, if $\inf_{\x \in \X}\varphi_K(\x) \rightarrow \inf_{\x \in \X} \varphi(\x)$ does not hold, then there exist $\delta > 0$ and subsequence  $\{\x_{l}\}$ of $\{\x_K\}$ such that
		\begin{equation}\label{eq:contra_dist}
		\inf_{\x \in \X}\varphi_{l}(\x)  < \inf_{\x \in \X} \varphi(\x) -  \delta, \  \forall l.
		\end{equation}
		Since $\X$ is compact, we can assume without loss of generality that $\x_{l} \rightarrow \bar{\x} \in \X$ by considering a subsequence. Then, as shown in above, we have $\bar{\x} \in \arg\min_{\x \in \X} \varphi(\x)$. And, by the same arguments for deriving Eq.~\eqref{eq1}, we can show that $\forall\epsilon > 0$, there exists $k(\epsilon) > 0$ such that $\forall l > k(\epsilon)$, it holds
		\begin{equation*}
		\varphi(\bar{\x}) \le \varphi_l(\x_l) + \epsilon.
		\end{equation*}
		By letting $l \rightarrow \infty$, $\epsilon \rightarrow 0$ and the definition of $\x_l$, we have
		\begin{equation*}
		\inf_{\x \in \X} \varphi(\x) = \varphi(\bar{\x}) \le \liminf_{l \rightarrow \infty}\left\{ \inf_{\x \in \X} \varphi_{l}(\x) \right\},
		\end{equation*}
		which implies a contradiction to Eq.~\eqref{eq:contra_dist}. Thus we have $\inf_{\x \in \X}\varphi_K(\x) \rightarrow \inf_{\x \in \X} \varphi(\x)$ as $K \rightarrow \infty$.
	\end{proof}

\subsection{Towards Local Minimum}\label{subsec:local}

	\begin{thm}\label{thm:local}(Convergence towards Local Minimum)
		Suppose both the LL and UL objective convergence properties hold and let $\x_K$ be a local $\varepsilon_K$-minimum of $\varphi_{K}(\x)$ with uniform neighborhood modulus $\delta > 0$, i.e.,
		\begin{equation*}
		\varphi_K(\x_K) \le \varphi_K(\x) + \varepsilon_K, \quad \forall \x \in \mathbb{B}_{\delta} (\x_K)\cap\X.
		\end{equation*}
		Then we have that any limit point $\bar{\x}$ of the sequence $\{\x_K\}$ is a local minimum of $\varphi$, i.e., there exists $\tilde{\delta} > 0$ such that
		\begin{equation*}
		\varphi(\bar{\x}) \le \varphi(\x), \quad \forall \x \in \mathbb{B}_{\tilde{\delta}} (\bar{\x})\cap \X.
		\end{equation*}
	\end{thm}
	\begin{proof}
		For any limit point $\bar{\x}$ of the sequence $\{\x_K\}$, let $\{\x_{l}\}$ be a subsequence of $\{\x_K\}$ such that $\x_{l} \rightarrow \bar{\x} \in \X$ and $\x_l \in \mathbb{B}_{\delta/2} (\bar{\x})$. 
		As $\{\y_K(\x)\}$ is uniformly bounded on $\X$, we can have a subsequence $\{\x_{m}\}$ of $\{\x_{l}\}$ satisfying $\y_m(\x_m) \rightarrow \bar{\y}$ for some $\bar{\y}$. It follows from the \emph{LL objective convergence property} that for any $\epsilon > 0$, there exists $M(\epsilon) > 0$ such that for any $m > M(\epsilon)$, we have
		\begin{equation*}
		f(\x_m,\y_m(\x_m)) - f^{\ast}(\x_m) \le \epsilon.
		\end{equation*}
		By letting $m \rightarrow \infty$, and since $f$ is continuous and $f^{\ast}(\x)$ is USC on $\X$, we have
		\begin{equation*}
		f(\bar{\x},\bar{\y}) - f^{\ast}(\bar{\x}) \le \epsilon.
		\end{equation*}
		As $\epsilon$ is arbitrarily chosen, we have $f(\bar{\x},\bar{\y}) - f^{\ast}(\bar{\x}) \le 0$ and thus $\bar{\y} \in \S(\bar{\x})$. 
		Next, as $F$ is continuous at $(\bar{\x},\bar{\y})$, for any $\epsilon > 0$, there exists $M(\epsilon) > 0$ such that for any $m > M(\epsilon)$, it holds
		\begin{equation*}
		F(\bar{\x},\bar{\y}) \le F(\x_m,\y_m(\x_m)) + \epsilon.
		\end{equation*}
		Then, we have, for any $m > M(\epsilon)$ and $\x \in \X$,
		\begin{equation*}
		\begin{aligned}
		\varphi(\bar{\x}) =\! \inf_{\y \in \S(\bar{\x}) }\! F(\bar{\x}, \y) 
		\le F(\x_m,\y_m(\x_m)) + \epsilon\! = \varphi_m(\x_m) + \epsilon.
		\end{aligned}
		\end{equation*}
		Next, as $\x_m$ is a local $\varepsilon_m$-minimum of $\varphi_{m}(\x)$ with uniform neighborhood modulus $\delta$, it follows
		\begin{equation*}
		\varphi_m(\x_m) \le \varphi_m(\x) + \varepsilon_m, \ \forall \x \in \mathbb{B}_{\delta} (\x_m)\cap \X.
		\end{equation*}
		Since $\mathbb{B}_{\delta/2} (\bar{\x}) \subseteq \mathbb{B}_{\delta/2+\|\x_m - \bar{\x}\|}(\x_m) \subseteq \mathbb{B}_{\delta} (\x_m)$, we have that for any $\epsilon > 0$, $\forall \x \in \mathbb{B}_{\delta/2} (\bar{\x}) \cap \X$, there exists $M(\epsilon) > 0$ such that whenever $m > M(\epsilon) $, 
		\begin{equation*}
		\begin{aligned}
		\varphi_m(\x_m) + \epsilon \le \varphi_m(\x) + \epsilon + \varepsilon_m.
		\end{aligned}
		\end{equation*}	
		Taking $m \rightarrow \infty$ and by the UL objective convergence property and $\varepsilon_m \rightarrow 0$, $\forall \x \in \mathbb{B}_{\delta/2} (\bar{\x}) \cap \X$ we have
		\begin{equation*}
		\begin{aligned}
		\varphi(\bar{\x}) \le \lim_{m \rightarrow \infty}\varphi_m(\x)+ \varepsilon_m + \epsilon = \varphi(\x) + \epsilon.
		\end{aligned}
		\end{equation*}
		By taking $\epsilon \rightarrow 0$, we have
		\begin{equation*}
		\varphi(\bar{\x}) \le \varphi(\x), \ \forall \x \in \mathbb{B}_{\delta/2} (\bar{\x}) \cap \X,
		\end{equation*}
		which implies $\bar{\x} \in \arg\min_{\x \in \mathbb{B}_{\delta/2} (\bar{\x})\cap \X} \varphi(\x)$, i.e, $\bar{x}$ is a local minimum of $\varphi$.
	\end{proof}

Note that this work provides a series of approximate optimization problems to the bi-level problem, and we establish the convergence of such approximation problems to the original bi-level problem (i.e., Eq.~\eqref{eq:blp}). Such kind of result is commonly used for characterizing the convergence of approximation type optimization method on nonconvex problems, see, for examples, Theorem 17.1 in book~\cite{nocedal2006numerical} for the convergence of the quadratic penalty function method and Theorem 7 in paper~\cite{wright1992interior} for convergence of the interior point method.

\begin{table*}[htb!]
	\setlength{\tabcolsep}{1mm}{
		\small
		\caption{Comparing the convergence results between our method and existing GBMs in different scenarios (i.e., BLO w/ and w/o LLS condition). 
		}\label{tab:result}
		\centering
		\renewcommand\arraystretch{1.2} 
		\begin{threeparttable}
			\begin{tabular}{|p{1.2cm}<{\centering} | c | c | c | c | c |}
				\hline
				\multicolumn{2}{|c|}{\multirow{2}{*}{Alg.} } &   & \multirow{2}{*}{w/ LLS} &  \multicolumn{2}{c|}{w/o LLS}  \\
				\cline{5-6}
				\multicolumn{2}{|c|}{} & & & w/ UL strong convexity &\quad w/o UL strong convexity\quad \\
				\hline
				\multicolumn{2}{|c|}{\multirow{5}{*}{\tabincell{c}{Existing\\GBMs}} }& {UL} &   $F(\x,\cdot)$ is Lipschitz continuous. & \multirow{5}{*}{Not available} & \multirow{5}{*}{Not available}\\
				\cline{3-4}
				\multicolumn{2}{|c|}{} & \multirow{1}{*}{{LL}} & \tabincell{c}{$\{\y_K(\x)\}$ is uniformly bounded on $\mathcal{X}$, \\ $\y_K(\x)\xrightarrow[]{u}\y^*(\x)$.} &  & \\
				\cline{3-4}
				\multicolumn{2}{|c|}{} & \multicolumn{2}{c|}{\tabincell{c}{Main results:  $\x_{K} \xrightarrow[]{s} \x^*$,\\ $\inf_{\x \in \mathcal{X}}\varphi_K(\x) \to \inf_{\x \in \mathcal{X}}\varphi(\x)$. }}& & \\
				\hline
				\multirow{16}{*}{Ours} & & \multirow{2}{*}{{UL}} & \multirow{2}{*}{ $F(\x,\cdot)$ is Lipschitz continuous.} & 
				$F(\x,\cdot)$ is ${L}_F$-smooth, & \multirow{6}{*}{Not available}\\
				&&&& and $\sigma$-strongly convex. & \\
				\cline{3-5}
				&\tabincell{c}{\cite{liu2020generic}}& \multirow{2}{*}{{LL}} & \tabincell{c}{$\{\y_K(\x)\}$ is uniformly bounded on $\mathcal{X}$,\\ $f(\x,\y_K(\x))\xrightarrow[]{u} f^*(\x)$.} & 
				\tabincell{c}{$f(\x,\cdot)$ is ${L}_f$-smooth and convex,\\ $\S(\x)$ is continuous.} & \\
				\cline{4-5}
				& &    & \multicolumn{2}{c|}{$f(\x,\y)$ is level-bounded in $\y$ locally uniformly in $\x\in\X$.}  & \\
				\cline{3-5}
				& & \multicolumn{3}{c|}{Main results:  $\x_{K} \xrightarrow[]{s}\x^*$, $\inf_{\x \in \mathcal{X}}\varphi_K(\x) \to \inf_{\x \in \mathcal{X}}\varphi(\x)$.}   & \\
				\cline{2-6}
				& \multirow{9}{*}{\tabincell{c}{This\\work}}& UL 
				& $F(\x,\cdot)$ is Lipschitz continuous. &
				\multicolumn{2}{c|}{$F(\x,\cdot)$ is $L_F$-smooth, convex and bounded below.}\\
				\cline{3-6}
				&	& LL & \tabincell{c}{$\{\y_K(\x)\}$ is uniformly bounded on $\mathcal{X}$,\\ $f(\x,\y_K(\x))\xrightarrow[]{u} f^*(\x)$.} & 
				\multicolumn{2}{c|}{\tabincell{c}{$f(\x,\cdot)$ is $L_f$-smooth and convex.}} \\
				\cline{3-6}
				& & \multicolumn{4}{c|}{Main results} \\
				\cline{3-6}
				& & \multicolumn{2}{c|}{\tabincell{c}{ Global result:\\ if $\x_K$ is a $\varepsilon_K$-minimum of $\varphi_{K}(\x)$,\\ then $\x_{K} \xrightarrow[]{s}\x^*$ and \\ $\inf_{\x \in \mathcal{X}}\varphi_K(\x) \to \inf_{\x \in \mathcal{X}}\varphi(\x)$.}} & \multicolumn{1}{c|}{\tabincell{c}{ Local result:\\ if $\x_K$ is a local $\varepsilon_K$-minimum of $\varphi_{K}(\x)$\\ then $\varphi(\bar{\x}) \le \varphi(\x), \forall \x \in \mathbb{B}_{\tilde{\delta}} (\bar{\x})\cap \X.$} }
				& \tabincell{c}{ Stationarity:\\ if $\x_K$ be a $\varepsilon_K$-stationary point\\ of $\varphi_{K}(\x)$, then $0 = \nabla \varphi(\bar{\x})$. \\ (require $F,f$ twice conti. diff.,\\ $f(\x,\cdot)$ $\sigma$-strongly convexity) } \\
				\hline
			\end{tabular}
			\begin{tablenotes}
				\footnotesize 
				\item  Here $\xrightarrow[]{s}$ and $\xrightarrow[]{u}$ represent the subsequential and uniform convergence, respectively. The superscript $^*$ denotes that it is the true optimal variables/values. ``conti." and ``diff." denote continuously and differentiable respectively. 
			\end{tablenotes}
		\end{threeparttable}
	}
\end{table*}

\section{Convergence Properties of BDA}\label{subsec:modules-T}

With the above discussions in Section~\ref{sec:alg}, the BLO is reduced to optimize a simple bi-level problem in Eq.~\eqref{eq:simple-blp} w.r.t. the LL variable $\y$, and subsequently solve a single-level problem in Eq.~\eqref{eq:upper_varphiK} w.r.t. the UL variable $\x$. This part analyzes the convergence behavior of the developed iterative algorithm. In other words, this part is devoted to show that our proposed BDA meets two convergence properties stated in Section~\ref{sec:conv_recipes} (i.e., \emph{UL objective convergence property} and \emph{LL objective convergence property}).

Following the above roadmap, convergence behaviors of gradient-based bi-level methods can be systematically investigated. 
The desired convergence results can be successfully achieved once the embedded task-tailored iterative gradient-aggregation modules $\mathcal{T}_k$ meet the \emph{UL objective convergence property} and the \emph{LL objective convergence property}. 
\subsection{UL Convergence Properties}\label{subsec:convex_convergence}
To investigate the convergence behavior of the proposed simple bi-level iterations $\T_k$ in Eq.~\eqref{eq:improved-lower}, with fixed $\x$, we first introduce the following two auxiliary variables 
\begin{equation*}
	\begin{array}{l}
	\z_{k+1}^{u}(\mathbf{x}) = \y_k(\mathbf{x}) - s_u \alpha_k\nabla F(\x,\y_k(\mathbf{x})),\\
	\z_{k+1}^l(\mathbf{x}) = \y_k(\mathbf{x}) -  s_l \beta_k \nabla f(\x,\y_k(\mathbf{x})).
	\end{array}
\end{equation*}
We further denote the optimal value and the optimal solution set of simple bi-level problem (i.e., Eq.~\eqref{eq:simple-blp}) by $\varphi(\x)$ and ${\hat{\S}(\x)}$, respectively.

As the identity of $\x$ is clear from the context, in Section \ref{subsec:convex_convergence} and~\ref{subsec:complexity_simple_bi}, for succinctness we will write $\Psi(\y)$ instead of $F(\x,\y)$, $\Psi^*$ instead of $\varphi(\x)$, $\psi(\y)$ instead of $f(\x,\y)$, $\S$ instead of $\S(\x)$, and $\hat{\S}$ instead of $\hat{\S}(\x)$. Moreover, we will omit the notation $\x$ and use the notations $\y_k$, $\z_{k+1}^{u}$ and $\z_{k+1}^l$ instead of the $\y_k(\mathbf{x})$, $\z_{k+1}^{u}(\mathbf{x})$ and $\z_{k+1}^l(\mathbf{x})$, respectively.

With inner iterative module, this part demonstrate the convergence behavior of simple bi-level. We first provide a descent inequality of function value in the following lemma.

\begin{lemma} \label{simple_bilevel_lem}
	Let $\{\y_k\}$ be the sequence generated by Eq.~\eqref{eq:improved-lower} with $\alpha_k, \beta_k \in (0,1]$, $s_u \in (0,\frac{1}{L_{F}})$, $s_l \in (0,\frac{1}{L_{f}})$ and $\mu \in (0,1)$, then for any $\y \in \Y$, we have
	\begin{equation}\label{alg_p_lem1_eq}
	\begin{array}{l}
	(1-\mu)\beta_kf(\x,\y) + \frac{\mu s_u\alpha_{k}}{s_l}F(\x,\y)\ge (1-\mu)\beta_kf(\x,\z_{k+1}^{l}) \\
	+ \frac{\mu s_u\alpha_{k}}{s_l}F(\x,\z_{k+1}^{u}) + \frac{\mu}{2s_l}(1 - \alpha_{k}s_uL_F)\|\y_k - \z_{k+1}^{u}\|^2\\
	+ \frac{1}{2s_l}\|\y - \y_{k+1}\|^2 + \frac{1}{2s_l} \left\| \left((1-\mu)\z_{k+1}^{l} + \mu \z_{k+1}^{u}\right) - \y_{k+1} \right\|^2\\  
	+ \frac{(1-\mu)}{2s_l}(1- \beta_ks_lL_f)\|\y_k-\z_{k+1}^{l}\|^2-  \frac{1}{2s_l} \|\y - \y_k\|^2.
	\end{array}
	\end{equation}
\end{lemma}
\begin{proof}
	It follows from the definitions of $\z_{k+1}^{u}$ and $\z_{k+1}^{l}$ that
	\begin{equation}\label{alg_p_eqs_opt_con}
	\begin{array}{l}
	0 = \alpha_k\nabla \Psi(\y_k) + \frac{\z_{k+1}^{u} - \y_k}{s_u}\ \text{and}\ 
	0 = \beta_k \nabla \psi(\y_k) + \frac{\z_{k+1}^{l} - \y_k}{s_l}.
	\end{array}
	\end{equation}
	Thus, for any $\y$, we have 
	\setlength{\arraycolsep}{0.0em}
	\begin{eqnarray}
	0 =  \alpha_k\langle \nabla \Psi(\y_{k}) , \y - \z_{k+1}^{u} \rangle + \langle \frac{\z_{k+1}^{u} - \y_{k}}{s_u} , \y - \z_{k+1}^{u}\rangle,\ \label{alg_p_lem1_eq2}\\
	0 = \beta_k \langle \nabla \psi(\y_k) , \y - \z_{k+1}^{l}\rangle + \langle \frac{\z_{k+1}^{l} - \y_k}{s_l} , \y - \z_{k+1}^{l}\rangle.\ \ \label{alg_p_lem1_eq1}
	\end{eqnarray}
	As $\psi$ is convex and $\nabla \psi$ is Lipschitz continuous with constant $L_{f}$, we have
	\begin{equation}\label{alg_p_eqs_lip_con}
	\begin{array}{l}
	\langle \nabla \psi(\y_k), \y - \z_{k+1}^{l}\rangle\\
	=  \langle \nabla \psi(\y_k), \y - \y_k\rangle + \langle \nabla \psi(\y_k), \y_k - \z_{k+1}^{l}\rangle \\
	\le \psi(\y) - \psi(\y_k) + \psi(\y_k) - \psi(\z_{k+1}^{l}) + \frac{L_{f}}{2}\|\y_k- \z_{k+1}^{l}\|^2  \\
	= \psi(\y) - \psi(\z_{k+1}^{l}) + \frac{L_{f}}{2}\|\y_k- \z_{k+1}^{l}\|^2.
	\end{array}
	\end{equation}
	Combining with $\langle \z_{k+1}^{l} - \y_k , \y - \z_{k+1}^{l}\rangle = \frac{1}{2}(\|\y - \y_k\|^2 - \|\y- \z_{k+1}^{l}\|^2 - \|\y_k - \z_{k+1}^{l}\|^2)$ and Eq.~\eqref{alg_p_lem1_eq1} yields
	\begin{equation}\label{alg_p_lem1_eq3}
	\begin{array}{r}
	\beta_k \psi(\y) \ge \beta_k \psi(\z_{k+1}^{l}) - \frac{1}{2s_l}\|\y - \y_k\|^2 + \frac{1}{2s_l} \|\y- \z_{k+1}^{l}\|^2 \\
	+ \frac{1}{2s_l}(1 - \beta_k s_l L_{f})\|\y_k - \z_{k+1}^{l}\|^2.
	\end{array}
	\end{equation}
	As $\Psi$ is convex and $\nabla \Psi$ is Lipschitz continuous with constant $L_{F}$, by similar arguments, we can have
	\begin{equation}\label{alg_p_lem1_eq4}
	\begin{array}{r}
	\alpha_k \Psi(\y) \ge \alpha_k\Psi(\z_{k+1}^{u}) - \frac{1}{2s_u}\|\y - \y_{k}\|^2
	+ \frac{1}{2s_u} \|\y- \z_{k+1}^{u}\|^2\\
	+ \frac{1}{2s_u}(1 - \alpha_k s_u L_{F})\|\y_{k} - \z_{k+1}^{u}\|^2.
	\end{array}
	\end{equation}
	Multiplying Eq.~\eqref{alg_p_lem1_eq3} and Eq.~\eqref{alg_p_lem1_eq4} by $1-\mu$ and $\frac{s_u\mu}{s_l}$, respectively, and then summing them up implies that
	\begin{equation}\label{alg_p_lem1_eq5}
	\begin{array}{l}
	(1-\mu)\beta_k\psi (\y) + \frac{\mu s_u\alpha_{k}}{s_l}\Psi(\y)\\
	\ge (1-\mu)\beta_k\psi(\z_{k+1}^{l}) + \frac{\mu s_u\alpha_{k}}{s_l}\Psi(\z_{k+1}^{u}) -  \frac{1}{2s_l} \|\y - \y_k\|^2 \\
	+ \frac{1}{2s_l}\left( (1-\mu)  \|\y - \z_{k+1}^{l}\|^2 + \mu \|\y- \z_{k+1}^{u}\|^2 \right)\\
	+ \frac{(1-\mu)}{2s_l}(1 - \beta_ks_lL_{f})\|\y_k - \z_{k+1}^{l}\|^2 \\
	+ \frac{\mu}{2s_l}(1 - \alpha_{k}s_uL_{F})\|\y_k - \z_{k+1}^{u}\|^2.
	\end{array}
	\end{equation}
	By the convexity of $\|\cdot\|^2$, we have
	\begin{equation*}
	\begin{array}{l}
	(1-\mu)  \|\y - \z_{k+1}^{l}\|^2 + \mu \|\y- \z_{k+1}^{u}\|^2\\
	\ge \| \y - \left((1-\mu)\z_{k+1}^{l} + \mu \z_{k+1}^{u}\right) \|^2.
	\end{array}
	\end{equation*}
	Next, as $\mathtt{Proj}_{\Y}$ is firmly nonexpansive (see, e.g.,\cite[Proposition 4.8]{Heinz-MonotoneOperator-2011}), for any $\y \in \Y$, we have
	\begin{equation}\label{alg_p_lem1_eq4.5}
	\begin{array}{l}
	\left\| \y - \left((1-\mu)\z_{k+1}^{l} + \mu \z_{k+1}^{u}\right) \right\|^2 \\
	\ge \|\y - \y_{k+1}\|^2 + \left\| \left((1-\mu)\z_{k+1}^{l} + \mu \z_{k+1}^{u}\right) - \y_{k+1} \right\|^2.
	\end{array}
	\end{equation}
	Then, since $\alpha_k, \beta_k \le 1$, we obtain form Eq.~\eqref{alg_p_lem1_eq5} that for any $\y \in \Y$,
	\begin{equation}
	\begin{array}{l}
	(1-\mu)\beta_k\psi(\y) + \frac{\mu s_u\alpha_{k}}{s_l}\Psi(\y) \\
	\ge (1-\mu)\beta_k\psi(\z_{k+1}^{l}) + \frac{\mu s_u\alpha_{k}}{s_l}\Psi(\z_{k+1}^{u}) -  \frac{1}{2s_l} \|\y - \y_k\|^2 \\
	+ \frac{1}{2s_l}\|\y - \y_{k+1}\|^2  + \frac{(1-\mu)}{2s_l}(1 - \beta_ks_lL_{f})\|\y_k - \z_{k+1}^{l}\|^2 \\
	+ \frac{\mu}{2s_l}(1 - \alpha_{k}s_uL_{F})\|\y_k - \z_{k+1}^{u}\|^2  \\
	+ \frac{1}{2s_l} \left\| \left((1-\mu)\z_{k+1}^{l} + \mu \z_{k+1}^{u}\right) - \y_{k+1} \right\|^2.
	\end{array}
	\end{equation} 
	This completes the proof. 
\end{proof}

\begin{lemma}\label{lem2}
	Let $\{a_k\}$ and $\{b_k\}$ be sequences of non-negative real numbers. Assume that there exists $n_0 \in \mathbb{N}$ such that 
	\begin{equation*}
	\begin{array}{c}
	a_{k+1} + b_k - a_k\le 0, \quad \forall k \ge n_0.
	\end{array}
	\end{equation*}
	Then $\lim_{k \rightarrow \infty} a_k$ exists and $\sum_{k=1}^{\infty}b_k < \infty$.
\end{lemma}
\begin{proof}
	Adding the inequality $	a_{k+1} + b_k - a_k\le 0,$ 
	from $k=n_0$ to $k = n-1$, we get
	$a_{n} + \sum_{k=n_0}^{n-1}b_k \le a_{n_0}$. 
	By letting $n \rightarrow \infty$, we get $\sum_{k=n_0}^{\infty}b_k < \infty$. As $\{a_k\}_{k \ge n_0}$ is a non-negative decreasing sequence, $\lim_{k \rightarrow \infty} a_k$ exists.
\end{proof}
The above Lemma~\ref{lem2} aims to analyze sequence inequality that will be applied in the following Theorem. We explore the boundness of inner iterative sequence in the following Lemma~\ref{lem_bounded}. 
\begin{lemma}\label{lem_bounded}
	Let $\{\y_k\}$ be the sequence generated by Eq.~\eqref{eq:improved-lower} with $\alpha_k \in (0,1]$, $\beta_k \in (0, 1]$, $s_u \in (0,\frac{1}{L_{F}})$, $s_l \in (0,\frac{1}{L_{f}})$ and $\mu\in (0,1)$, then for any $\bar{\y} \in \S(\x)$, we have
	\begin{equation}
	\begin{array}{l}
	\|\z_{k+1}^{l} - \bar{\y}\| \le \|\y_{k} - \bar{\y}\|.
	\end{array}
	\end{equation}
	Furthermore, when $\Y$ is compact, sequences $\{\y_{k}\}$, $\{\z_{k}^{l}\}$, $\{\z_{k}^{u}\}$ are all bounded. 
\end{lemma}
\begin{proof}
	According to \cite[Proposition 4.8, Proposition 4.33, Corollary 18.16]{Heinz-MonotoneOperator-2011}, we know that when $0 \le \beta_k s_l \le \frac{1}{L_{f}}$, $0 \le \alpha_k s_u \le \frac{1}{L_{F}}$, operators $I - \beta_k s_l\nabla \psi$ and $I - \alpha_k s_u\nabla \Psi$ are both nonexpansive (i.e., $1$-Lipschitz continuous). Then, since $\z_{k+1}^{l} = \y_k - \beta_k s_l\nabla \psi(\y_k)$ and $\bar{\y} = \bar{\y} - \beta_k s_l\nabla \psi(\bar{\y})$ for any $\bar{\y} \in \mathcal{S}$, we have
	\begin{equation*}
	\begin{aligned}
	\|\z_{k+1}^{l} - \bar{\y}\| &= \|\y_k - \beta_k s_l\nabla \psi(\y_k) - \bar{\y} + \beta_k s_l\nabla \psi(\bar{\y})\| \\
	&\le \|\y_{k} - \bar{\y}\|.
	\end{aligned}
	\end{equation*}
	If $\Y$ is compact, then the desired boundedness of $\{\y_k\}$ follows directly from the iteration scheme in Eq.~\eqref{eq:improved-lower}. And it follows from $\|\z_{k+1}^{l} - \bar{\y}\| \le \|\y_{k} - \bar{\y}\|$ that $\{\z_{k}^{l}\}$ is bounded.
	Next, because
	\begin{equation*}
	\|\z_{k+1}^{u} - (\bar{\y} - \alpha_ks_u\nabla \Psi(\bar{\y}))\| \le  \|\y_{k} - \bar{\y}\|,
	\end{equation*}
	and $\alpha_k \in (0,1]$, we have $\{\z_{k}^{u}\}$ is bounded.
\end{proof}
With the above lemmas, we are now ready to obtain the convergence result of our proposed algorithm in the following theorem. 
\begin{thm}\label{simple_bilevel_convergence}
	Let $\{\y_k(\x)\}$ be the sequence generated by Eq.~\eqref{eq:improved-lower} with $\alpha_k \in (0,1]$, $\alpha_k \searrow 0$, $\sum \alpha_k = + \infty$, $\beta_k \in [\underline{\beta}, 1]$ with some $\underline{\beta} > 0$, $s_u \in (0,\frac{1}{L_{F}})$, $s_l \in (0,\frac{1}{L_{f}})$ and $\mu \in (0,1)$, suppose that $\Y$ is compact, for any given $\x$, if ${\hat{\S}(\x)}$ is nonempty , we have
	\begin{equation*}
		\begin{array}{c}
			\lim\limits_{k \rightarrow \infty}\mathrm{dist}(\y_k(\x), {\hat{\S}(\x)}) = 0,
		\end{array}
	\end{equation*}
	and then
	\begin{equation*}
		\begin{array}{c}
			\lim\limits_{k \rightarrow \infty}F(\x,\y_k(\x)) =  \varphi(\x).
		\end{array}
	\end{equation*}
\end{thm}
\begin{proof}
	Let $\delta > 0$ be a constant satisfying $\delta < \frac{1}{2s_l}\min\{(1-\mu)(1- s_lL_{f}),\mu(1 - s_uL_{F})\}$. We consider a sequence of $\{\tau_n\}$ defined by
	\begin{equation*}
	\begin{array}{l}
	\!\tau_n := \max\Big\{k\! \in\! \mathbb{N}\ | k\le n\ \text{and}\ \delta\|\y_{k-1} \!-\!\z_{k}^{l}\|^2 + \delta\|\y_{k-1} \!-\! \z_{k}^{u}\|^2  \\
	\!+ \frac{1}{4s_l} \left\| \left((1-\mu)\z_{k}^{l} \!+\! \mu \z_{k}^{u}\right) \!-\! \y_{k} \right\|^2 
	+ \frac{\mu s_u\alpha_{k-1}}{s_l}\left(\Psi(\z_{k}^{u}) \!-\! \Psi^*\right) < 0 \Big\}.
	\end{array}
	\end{equation*}
	Inspired by \cite{cabot2005proximal}, we consider the following two cases: 
	(a) $\{\tau_n\}$ is finite, i.e., there exists $k_0 \in \mathbb{N}$ such that
	\begin{equation*}
	\begin{array}{l}
	\delta\|\y_{k-1} -\z_{k}^{l}\|^2 + \frac{1}{4s_l} \left\| \left((1-\mu)\z_{k}^{l} + \mu \z_{k}^{u}\right) - \y_{k} \right\|^2\\ 
	+ \delta\|\y_{k-1} - \z_{k}^{u}\|^2
	+ \frac{\mu s_u\alpha_{k-1}}{s_l}\left(\Psi(\z_{k}^{u}) - \Psi^*\right) \ge 0,
	\end{array}
	\end{equation*}
	for all $k \ge k_0$; 
	(b) $\{\tau_n\}$ is not finite, i.e., for all $ k_0 \in \mathbb{N}$, there exists $k \ge k_0$ such that 
	\begin{equation*}
	\begin{array}{r}
	\delta\|\y_{k-1} - \z_{k}^{u}\|^2 + \frac{1}{4s_l} \left\| \left((1-\mu)\z_{k}^{l} + \mu \z_{k}^{u}\right) - \y_{k} \right\|^2 +\\
	\delta\|\y_{k-1} -\z_{k}^{l}\|^2 + \frac{\mu s_u\alpha_{k-1}}{s_l}\left(\Psi(\z_{k}^{u}) - \Psi^*\right) < 0.
	\end{array}
	\end{equation*}
	
	\noindent\textbf{Case (a):} We assume that $\{\tau_n\}$ is finite and there exists $k_0 \in \mathbb{N}$ such that 
	\begin{equation}\label{thm1_eq0}
	\begin{array}{l}
	\delta\|\y_{k-1} -\z_{k}^{l}\|^2 + \frac{1}{4s_l} \left\| \left((1-\mu)\z_{k}^{l} + \mu \z_{k}^{u}\right) - \y_{k} \right\|^2 \\
	+ \delta\|\y_{k-1} - \z_{k}^{u}\|^2  + \frac{\mu s_u\alpha_{k-1}}{s_l}\left(\Psi(\z_{k}^{u}) - \Psi^*\right) \ge 0,
	\end{array}
	\end{equation}
	for all $k \ge k_0$. Let $\bar{\y}$ be any point in $\hat{\mathcal{S}}$, setting $\y = \bar{\y}$ in Eq.~\eqref{alg_p_lem1_eq}, as $\psi(\bar{\y}) = \min_{\y \in \mathbb{R}^n} \psi(\y) \le \psi(\z_{k+1}^{l})$, $\mu \in (0,1)$ and $\alpha_k, \beta_k \le 1$, we have
	\begin{equation}\label{thm1_eq1}
	\begin{array}{l}
	\frac{1}{2s_l}\|\bar{\y} - \y_k\|^2\\
	\ge \frac{1}{2s_l}\|\bar{\y}- \y_{k+1}\|^2 + \left( \frac{(1-\mu)(1 - s_lL_{f})}{2s_l} - \delta  \right)\|\y_k - \z_{k+1}^{l}\|^2 \\
	+ \left( \frac{\mu(1 - s_uL_{F})}{2s_l} - \delta \right)\|\y_k - \z_{k+1}^{u}\|^2  \\
	+ \frac{1}{4s_l} \left\| \left((1-\mu)\z_{k+1}^{l} + \mu \z_{k+1}^{u}\right) - \y_{k+1} \right\|^2  \\
	+ \delta\|\y_{k} -\z_{k+1}^{l}\|^2 + \ \delta\|\y_{k} - \z_{k+1}^{u}\|^2 \\
	+ \frac{1}{4s_l} \left\| \left((1-\mu)\z_{k+1}^{l} + \mu \z_{k+1}^{u}\right) - \y_{k+1} \right\|^2 \\
	+ \frac{\mu s_u\alpha_{k}}{s_l}\left(\Psi(\z_{k+1}^{u}) - \Psi^*\right).
	\end{array}
	\end{equation}
	For all $k \ge k_0$, $0 < \delta < \frac{1}{2s_l}\min\{(1-\mu)(1- s_lL_{f}),\mu(1 - s_uL_{F})\}$ yields $( (1-\mu)(1 - s_lL_{f})/{2s_l} - \delta  )\|\y_k - \z_{k+1}^{l}\|^2 \ge 0$ and $( \mu(1 - s_uL_{F})/{2s_l} - \delta )\|\y_k - \z_{k+1}^{u}\|^2 \ge 0$. Then applying Lemma \ref{lem2} on Eq.~\eqref{thm1_eq1} with Eq.~\eqref{thm1_eq0} implies that
	\begin{equation*}
	\begin{array}{l}
	\sum\limits_{k=0}^{\infty}\|\y_k - \z_{k+1}^{l}\|^2 < \infty, \ 
	\sum\limits_{k=0}^{\infty}\|\y_k - \z_{k+1}^{u}\|^2 < \infty,\\
	\sum\limits_{k=0}^{\infty}\left\| \left((1-\mu)\z_{k+1}^{l} + \mu \z_{k+1}^{u}\right) - \y_{k+1} \right\|^2  < \infty,\\
	\sum\limits_{k=0}^{\infty} \alpha_k \left(\Psi(\z_{k+1}^{u}) - \Psi^* \right) < \infty,\\
	\end{array}
	\end{equation*}
	and $\lim_{k \rightarrow \infty}\|\bar{\y}- \y_k\|^2 $ exists.

	We now show that there exists subsequence $\{\y_{\ell}\} \subseteq \{\y_k\}$ such that $\lim_{\ell \rightarrow \infty} \Psi(\y_{\ell} ) \le \Psi^* $. This is obviously true if for any $\hat{k} > 0$, there exists $k > \hat{k}$ such that $\Psi(\y_{k} ) \le \Psi^*$. Thus, we just need to consider the case where there exists $\hat{k} > 0$ such that $ \Psi(\y_{k} ) > \Psi^*$ for all $k \ge \hat{k}$. If there does not exist subsequence $\{\y_{\ell}\} \subseteq \{\y_k\}$ such that $\lim_{\ell \rightarrow \infty} \Psi(\y_{\ell} ) \le \Psi^* $, there must exist $\epsilon > 0$ and $k_1 \ge \max\{\hat{k}, k_0\}$ such that $\Psi(\y_{k}) - \Psi^* \ge 2\epsilon$ for all $k \ge k_1$. As $\Y$ is compact, it follows from Lemma \ref{lem_bounded} that sequences $\{\y_k\}$ and $\{\z_{k}^{u}\}$ are both bounded. Since $\Psi$ is continuous and $\lim_{k \rightarrow \infty}\|\y_k - \z_{k+1}^{u}\| = 0$, there exists $k_2 \ge k_1$ such that $|\Psi(\y^{k}) - \Psi(\z_{k+1}^{u})| < \epsilon$ for all $k \ge k_2$ and thus $\Psi(\z_{k+1}^{u}) - \Psi^* \ge \epsilon$ for all $k \ge k_2$.
	Then we have
	\begin{equation*}
	\begin{array}{l}
		\epsilon \sum\limits_{k = k_2}^{\infty} \alpha_k \le \sum\limits_{k = k_2}^{\infty}  \alpha_k \left(\Psi(\z_{k+1}^{u}) - \Psi^*\right) < \infty,
	\end{array}
	\end{equation*}
	where the last inequality follows from $ \sum_{k=0}^{\infty} \alpha_k \left(\Psi(\z_{k+1}^{u}) - \Psi^*\right) < \infty$.
	This result contradicts to the assumption $\sum_{k=0}^{\infty} \alpha_k = +\infty$. As $\{\y_{\ell}\}$ is bounded, we can assume without loss of generality that $\lim_{\ell \rightarrow \infty}\y_{\ell} = \tilde{\y} $ by taking a subsequence. By the continuity of $\Psi$, we have $\Psi(\tilde{\y}) = \lim_{\ell \rightarrow \infty} \Psi(\y_{\ell} ) \le \Psi^*$.
	Next, let $k = \ell$ and $\ell \rightarrow \infty$ in Eq.~\eqref{alg_p_eqs_opt_con} , by the continuity of $\nabla \psi$, $\beta_k \ge \underline{\beta} > 0$, and $\lim_{k \rightarrow \infty}\|\y_k - \z_{k+1}^{l}\| = 0$, we have
	\begin{equation*}
	\begin{array}{c}
	0 \in \nabla \psi(\tilde{\y}),
	\end{array}
	\end{equation*}
	and thus $\tilde{\y} \in \mathcal{S}$. Combining with $\Psi(\tilde{\y}) \le \Psi^*$, we show that $\tilde{\y} \in \hat{\mathcal{S}}$. Then by taking $\bar{\y} = \tilde{\y}$ and since $\lim_{k \rightarrow \infty}\|\bar{\y}- \y_k\|^2 $ exists, we have $\lim_{k \rightarrow \infty}\|\bar{\y}- \y_k\|^2 = 0$ and thus $\lim_{k \rightarrow \infty}\mathrm{dist}(\y_k, \hat{\mathcal{S}}) = 0$.
	
	\noindent\textbf{Case (b):} We assume that $\{\tau_n\}$ is not finite and for any $ k_0 \in \mathbb{N}$, there exists $k \ge k_0$ such that $\delta\|\y_{k-1} -\z_{k}^{l}\|^2 + \ \delta\|\y_{k-1} - \z_{k}^{u}\|^2  + \frac{1}{4s_l} \left\| \left((1-\mu)\z_{k}^{l} + \mu \z_{k}^{u}\right) - \y^{k} \right\|^2 + \frac{\mu s_u\alpha_{k-1}}{s_l}\left(\Psi(\z_{k}^{u}) - \Psi^*\right) < 0$ .
	It follows from the assumption that $\tau_n$ is well defined for $n$ large enough and $\lim_{n \rightarrow \infty} \tau_n = + \infty$. We assume without loss of generality that $\tau_n$ is well defined for all $n$.
	
	By setting $\y = \mathtt{Proj}_{\hat{\mathcal{S}}}(\y_k)$ in Eq.~\eqref{alg_p_lem1_eq}, we have
	\begin{equation}\label{thm1_eq4}
	\begin{array}{l}
	\frac{1}{2s_l} \mathrm{dist}^2(\y_k,\hat{\mathcal{S}})\\
	\ge\frac{1}{2s_l}\mathrm{dist}^2(\y_{k+1},\hat{\mathcal{S}}) + \left( \frac{(1-\mu)(1 - s_lL_f)}{2s_l} - \delta  \right)\|\y_k - \z_{k+1}^{l}\|^2 \\
	+ \left( \frac{\mu(1 - s_uL_{F})}{2s_l} - \delta \right)\|\y_k - \z_{k+1}^{u}\|^2
	+ \delta\|\y_{k} -\z_{k+1}^{l}\|^2\\
	+ \ \delta\|\y_{k} - \z_{k+1}^{u}\|^2  + \frac{1}{4s_l} \left\| \left((1-\mu)\z_{k+1}^{l} + \mu \z_{k+1}^{u}\right) - \y_{k+1} \right\|^2  \\
	+ \frac{1}{4s_l} \left\| \left((1-\mu)\z_{k+1}^{l} + \mu \z_{k+1}^{u}\right) - \y_{k+1} \right\|^2 \\
	+ \frac{\mu s_u\alpha_{k}}{s_l}\left(\Psi(\z_{k+1}^{u}) - \Psi^*\right) + \beta_k\left(\psi(\z_{k+1}^{l}) - \min \psi \right).
	\end{array}
	\end{equation}
	Suppose $\tau_n \le n-1$, and by the definition of $\tau_n$, we have 
	\begin{equation*} 
	\begin{array}{r}
	\delta\|\y_{k} -\z_{k+1}^{l}\|^2 + \ \delta\|\y_{k} - \z_{k+1}^{u}\|^2 + \frac{\mu s_u\alpha_{k}}{s_l}\left(\Psi(\z_{k+1}^{u}) - \Psi^* \right) \\
	+ \frac{1}{4s_l} \left\| \left((1-\mu)\z_{k+1}^{l} + \mu \z_{k+1}^{u}\right) - \y_{k+1} \right\|^2  \ge 0,
	\end{array}
	\end{equation*}
	for all $\tau_n \le k \le n-1 $. Then 
	\begin{equation}\label{thm1_eq5}
	h_{k+1}-h_{k} \le 0, \quad \tau_n \le k \le n-1,
	\end{equation}
	where $h_k := \frac{1}{2s_l} \mathrm{dist}^2(\y_k,\hat{\mathcal{S}})$. 
	Adding these $n-\tau_n$ inequalities, we have
	\begin{equation}\label{thm1_eq6}
	h_{n} \le h_{\tau_n}.
	\end{equation}
	Eq.~\eqref{thm1_eq6} is also true when $\tau_n = n$ because $h_{\tau_n} = h_n$. Once we are able to show that $\lim_{n \rightarrow \infty}h_{\tau_n} = 0$, we can obtain from Eq.~\eqref{thm1_eq6} that $\lim_{n \rightarrow \infty}h_{n} = 0$. 
	
	By the definition of $\{\tau_n\}$, $\Psi^* > \Psi(\z_{k}^{u})$ for all $k \in \{\tau_n\}$. Since $\Y$ is compact, according to Lemma \ref{lem_bounded}, both $\{\y_{\tau_n}\}$ and $\{\z^{u}_{\tau_n}\}$ are bounded, and hence $\{h_{\tau_n}\}$ is bounded. As $\Psi$ is assumed to be continuous,  there exists $M_0$ such that
	\begin{equation*}
	0 \le \Psi^* - \Psi(\z^u_{k}) \le \Psi^* - M_0.
	\end{equation*}
	According to the definition of $\tau_n$, we have for all $k \in \{\tau_n\}$,
	\begin{equation*}
	\begin{array}{l}
	\delta(\|\y_{k-1} -\z_{k}^{l}\|^2 + \|\y_{k-1} - \z_{k}^{u}\|^2) \\
	+ \frac{1}{4s_l} \left\| \left((1-\mu)\z_{k}^{l} + \mu \z_{k}^{u}\right) - \y_{k} \right\|^2 \\
	< \frac{\mu s_u\alpha_{k-1}}{s_l}\left(\Psi^*  - \Psi(\z_{k}^{u}) \right)
	\le \frac{\mu s_u\alpha_{k-1}}{s_l}\left(\Psi^* - M_0\right).
	\end{array}
	\end{equation*}
	As  $\lim_{n \rightarrow \infty} \tau_n = + \infty$, $\alpha_k \rightarrow 0$, we have
	\begin{equation*}
	\begin{array}{c}
	\lim_{n \rightarrow \infty}\|\y_{\tau_n-1} -\z^{l}_{\tau_n}\| = 0, \\
	\lim_{n \rightarrow \infty}\|\y_{\tau_n-1} - \z^{u}_{\tau_n}\| = 0, \\
	\lim_{n \rightarrow \infty}\| \left((1-\mu)\z_{\tau_n}^{l} + \mu \z_{\tau_n}^{u}\right) - \y_{\tau_n} \| = 0.
	\end{array}
	\end{equation*}
	Let $\tilde{\y}$ be any limit point of $\{\y_{\tau_n}\}$, and $\{\y_{\ell}\}$ be the subsequence of $\{\y_{\tau_n}\}$ such that 
	\begin{equation*}
	\begin{array}{c}
	\lim_{\ell \rightarrow \infty}\y_{\ell} = \tilde{\y},
	\end{array}
	\end{equation*}
	as $\lim_{n \rightarrow \infty}\|\y_{\tau_n-1} - \y_{\tau_n}\| \le \lim_{n \rightarrow \infty}(\|\y_{\tau_n-1} - \left((1-\mu)\z_{\tau_n}^{l} + \mu \z_{\tau_n}^{u}\right)\|+\| \left((1-\mu)\z_{\tau_n}^{l} + \mu \z_{\tau_n}^{u}\right) - \y_{\tau_n} \| ) = 0$. We have $\lim_{\ell \rightarrow \infty}\y_{\ell-1} = \tilde{\y} $. Let $k = \ell-1$ and $\ell \rightarrow \infty$ in Eq.~\eqref{alg_p_eqs_opt_con}, by the continuity of $\nabla \psi$, $\beta_k \ge \underline{\beta} > 0$ and $\lim_{\ell \rightarrow \infty}\|\y_{\ell-1} - \z^{l}_{\ell}\|  = 0$. Then, we have
	\begin{equation*}
	0 \in \nabla \psi(\tilde{\y}),
	\end{equation*}
	and thus $\tilde{\y} \in \mathcal{S}$. As $\Psi^* > \Psi(\z_{k}^{u})$ for all $k \in \{\tau_n\}$ and hence $\Psi^* > \Psi(\z^{u}_\ell)$ for all $\ell$. Then it follows from the continuity of $\Psi$ and $\lim_{n \rightarrow \infty}\|\z^u_{\tau_n} - \y_{\tau_n}\| = 0$ that $\Psi^* \ge \Psi(\tilde{\y})$, which implies $\tilde{\y} \in \hat{\mathcal{S}}$ and $\lim_{\ell \rightarrow 0} h_{\ell} = 0$. Now, as we have shown above that $\tilde{\y} \in \hat{\mathcal{S}}$ for any limit point $\tilde{\y}$ of $\{\y_{\tau_n}\}$, we can obtain from the boundness of $\{\y_{\tau_n}\}$ and $\{h_{\tau_n}\}$ that $\lim_{n \rightarrow \infty} h_{\tau_n} = 0$. Thus $\lim_{n \rightarrow \infty}h_{n} = 0$, and $\lim_{k \rightarrow \infty}\mathrm{dist}(\y_k, \mathcal{S}) = 0$.
\end{proof}

\subsection{LL Convergence Properties}\label{subsec:complexity_simple_bi}

	Specially, when we take $\alpha_k = 1/(k+1)$, we have the following uniformly complexity estimation. We first denote $D = \sup\limits_{\y,\y'\in \Y}\|\y-\y'\|$, $M_{F} := \sup\limits_{\x \in \X, \y\in \Y} \|\nabla_{\y} F(\x, \y) \| $ and $M_{f} : = \sup\limits_{\x \in \X, \y \in \Y} \|\nabla_{\y} f(\x, \y) \|$. And it should be notice that $D$, $M_{F}$ and $M_{f}$ are all finite when $\mathcal{X}$ and $\mathcal{Y}$ are compact.

\begin{lemma}\label{lem_monontone}
	Let $\{\y_k\}$ be the sequence generated by Eq.~\eqref{eq:improved-lower} with $\alpha_k = \frac{1}{k+1}$, $\beta_k \in [\underline{\beta}, 1]$ with some $\underline{\beta} > 0$, $|\beta_k - \beta_{k-1}| \le \frac{c_\beta}{(k+1)^2}$ with some $c_\beta > 0$, $s_u \in (0,\frac{1}{L_{F}})$, $s_l \in (0,\frac{1}{L_{f}})$ and $\mu\in (0,1)$, then for any $\bar{\y} \in \S(\x)$, we have
	\begin{equation*}
	\begin{array}{l}
	\|\y_{k+1} - \y_{k} \|^2 \le\|\y_{k} - \y_{k-1}\|^2  + \frac{\mu}{(k+1)^2}\|\y_{k-1} - \z_{k}^{u} \|^2  \\
	+ \frac{2(1-\mu) s_l c_\beta DM_{f}}{(k+1)^2} + \frac{2\mu s_uDM_{F}}{k(k+1)} 	+ \frac{(1-\mu)c_\beta^2}{\underline{\beta}^2(k+1)^4}\|\y_{k-1} - \z_{k}^{l} \|^2.
	\end{array}
	\end{equation*}
\end{lemma}
\begin{proof}
	According to \cite[Proposition 4.8, Proposition 4.33, Corollary 18.16]{Heinz-MonotoneOperator-2011}, we know that when $0 \le \beta_k s_l \le \frac{1}{L_f}$, $0 \le \alpha_k s_u \le \frac{1}{L_{F}}$, operators $I - \beta_k s_l\nabla \psi$, $I - \alpha_k s_u\nabla \Psi$ and $\mathtt{Proj}_{\Y}$ are all nonexpansive (i.e., $1$-Lipschitz continuous). Next, as 
	\begin{equation*}
	\begin{array}{l}
	\y_{k+1} = \mathtt{Proj}_{\Y}\left( \mu\z_{k+1}^{u} + (1-\mu)\z_{k+1}^{l} \right)\\
	\quad\quad\ \  = \mathtt{Proj}_{\Y}\left( \y_k \!-\! \left( \mu \alpha_ks_u\nabla \Psi(\y_k) \! +\! (1\!-\!\mu)\beta_ks_l\nabla \psi(\y_k) \right) \right),
	\end{array}
	\end{equation*}
	by denoting $\Delta^k_{\alpha}:=\alpha_{k} - \alpha_{k-1}$ and $\Delta^k_{\beta}:=\beta_{k} - \beta_{k-1}$, we have the following inequality
	\begin{equation*}
	\begin{array}{l}
	\|\y_{k+1} - \y_{k} \|^2 \\
	\le\mu\|\z_{k+1}^{u} - \z_{k}^{u}\|^2 + (1-\mu)\|\z_{k+1}^{l} - \z_{k}^{l}\|^2, \\
	\le\mu \|(I \!-\! \alpha_k s_u\nabla \Psi)(\y_{k} \!-\! \y_{k-1})\|^2 \!+\! \mu s_u^2|\Delta^k_{\alpha}|^2\|\nabla \Psi(\y_{k-1}) \|^2 \\
	+ 2\mu s_u|\delta^k_{\alpha}| \|(I\!-\! \alpha_k s_u\nabla \Psi)(\y_{k}\! -\! \y_{k-1})\|\|\nabla \Psi(\y_{k-1}) \| \\
	+ (1-\mu)\|(I - \beta_k s_l\nabla \psi)(\y_{k} - \y_{k-1})\|^2  \\
	+ 2(1-\mu)s_l|\Delta^k_{\beta}|\|(I - \beta_k s_l\nabla \psi)(\y_{k} - \y_{k-1})\|\|\nabla \psi(\y_{k-1}) \| \\
	+  (1-\mu)s_l^2|\Delta^k_{\beta}|^2\|\nabla \psi(\y_{k-1}) \|^2 \\
	\le\|\y_{k} - \y_{k-1}\|^2 + 2\mu s_u|\Delta^k_{\alpha}| \|\y_{k} - \y_{k-1}\|\|\nabla \Psi(\y_{k-1}) \| \\
	+ \frac{\mu|\Delta^k_{\alpha}|^2}{\alpha_{k-1}^2}\|\y_{k-1} - \z_{k}^{u} \|^2+ \frac{(1-\mu)|\Delta^k_{\beta}|^2}{\beta_{k-1}^2}\|\y_{k-1} - \z_{k}^{l} \|^2\\
	+ 2(1-\mu)s_l|\Delta^k_{\beta}|\|\y_{k} - \y_{k-1}\|\|\nabla \psi(\y_{k-1}) \|,
	\end{array}
	\end{equation*}
	where the first inequality follows from the nonexpansiveness of $\mathtt{Proj}_{\Y}$ and the convexity of $\|\cdot\|^2$, the second inequality comes from the definitions of $\z_{k}^{u}, \z_{k}^{l}$ and the last inequality follows from the nonexpansiveness of $I - \beta_k s_l\nabla \psi$ and $I - \alpha_k s_u\nabla \Psi$ and the definitions of $\z_{k}^{u}, \z_{k}^{l}$.
	Then, since $\alpha_k = \frac{1}{k+1}$, $\beta_k \ge \underline{\beta} > 0$, $|\beta_k - \beta_{k-1}| \le \frac{c_\beta}{(k+1)^2}$, $D = \sup_{\y,\y'\in \Y}\|\y-\y'\|$, $\sup_{\y\in \Y} \|\nabla \Psi(\y) \| \le M_{F}$ and $\sup_{\y \in \Y} \|\nabla \psi(\y) \| \le M_{f}$, we have the following result
	\begin{equation*}
	\begin{array}{l}
	\|\y_{k+1} - \y_{k} \|^2 \le\;  \|\y_{k} - \y_{k-1}\|^2 + \frac{\mu}{(k+1)^2}\|\y_{k-1} - \z_{k}^{u} \|^2 \\
	+ \frac{2(1-\mu) s_l c_\beta DM_{f}}{(k+1)^2}  + \frac{2\mu s_uDM_{F}}{k(k+1)} +  \frac{(1-\mu)c_\beta^2}{\underline{\beta}^2(k+1)^4}\|\y_{k-1} - \z_{k}^{l} \|^2.
	\end{array}
	\end{equation*}
\end{proof}

\begin{thm}\label{simple_bilevel_complexity}
	Let $\{\y_k(\x)\}$ be the sequence generated by Eq.~\eqref{eq:improved-lower} with $\alpha_k = \frac{1}{k+1}$, $\beta_k \in [\underline{\beta}, 1]$ with some $\underline{\beta} > 0$, $|\beta_k - \beta_{k-1}| \le \frac{c_\beta}{(k+1)^2}$ with some $c_\beta > 0$, $s_u \in (0,\frac{1}{L_{F}})$, $s_l \in (0,\frac{1}{L_{f}})$ and $\mu \in (0,1)$. Suppose ${\hat{\S}(\x)}$ is nonempty, $\Y$ is compact, $F(\x,\cdot)$ is bounded below by $M_0$, we have
	for $k \ge 2$,
	\begin{equation*}
		\begin{array}{l}
			\|\y_k(\x) - \z_{k+1}^{l}(\x)\|^2 \le \frac{(2C_2 + C_3)}{\underline{\beta}^2} \frac{1+\ln k}{k^{\frac{1}{4}}}, \\
			f(\z_{k+1}^{l}(\x)) - \min f \le  \frac{D}{\underline{\beta}^2 s_l} \sqrt{(2C_2 + C_3)} \sqrt{\frac{1+\ln k}{k^{\frac{1}{4}}}},
		\end{array}
	\end{equation*}
	where $C_3:= \frac{D^2 + 2s_u\left( \varphi(\x) -  M_0 \right)}{(1-\mu)(1 - s_lL_f)}$, $C_2 := (s_l^2L_f^2D + \frac{4DL_{f}}{\underline{\beta}})\sqrt{C_1}$, $C_1:= \frac{C_0 ( D^2 + 2s_u( \varphi(\x) - M_0 ) ) + 2\mu s_uDM_{F} + 2(1-\mu) s_lc_\beta DM_{f}}{\min \{ (1 - s_lL_f),(1 - s_uL_F), 1 \}}$ and $C_0=\max\{2+ c_\beta^2/\underline{\beta}^2,3\}$.
\end{thm}
\begin{proof}
	Let $\bar{\y}$ be any point in $\mathcal{S}$, and set $\y = \bar{\y}$ in Eq.~\eqref{alg_p_lem1_eq}, since $\psi(\bar{\y}) = \min_{\y \in \mathbb{R}^n} \psi(\y) \le \psi(\z^l_{k+1})$, we have
	\begin{equation}\label{thm2_eq1}
	\begin{array}{l}
	\frac{1}{2} \|\bar{\y} - \y_k\|^2  + \frac{\mu s_u}{k+1}\left( \Psi^* -  \Psi(\z_u^{k+1}) \right) \\ \ge\,
	\frac{1}{2} \|\bar{\y}- \y_{k+1}\|^2 + \frac{1}{2}(1-\mu)(1 - \beta_ks_lL_f)\|\y_k - \z_{k+1}^{l}\|^2 \\
	+ \frac{1}{2}\mu(1 - \alpha_{k}s_uL_{F})\|\y_k - \z_{k+1}^{u}\|^2  \\
	+ \frac{1}{2} \left\| \left((1-\mu)\z_{k+1}^{l} + \mu \z_{k+1}^{u}\right) - \y_{k+1} \right\|^2
	\end{array}
	\end{equation}
	Adding the Eq.~\eqref{thm2_eq1} from $k = 0$ to $k = n-1$, and since $\alpha_k, \beta_k \in (0,1]$, we have
	\begin{equation}\label{thm2_eq2}
	\begin{array}{l}
	\frac{1}{2} \|\bar{\y}- \y_{n}\|^2 + \frac{1}{2}(1-\mu)(1 - s_lL_f)\sum\limits_{k=0}^{n-1}\|\y_k - \z_{k+1}^{l}\|^2 \\
	+ \frac{1}{2}\mu(1 -  s_uL_{F})\sum\limits_{k=0}^{n-1}\|\y_k - \z_{k+1}^{u}\|^2 \\
	+ \frac{1}{2}\sum\limits_{k=0}^{n-1}\left\| \left((1-\mu)\z_{k+1}^{l} + \mu \z_{k+1}^{u}\right) - \y_{k+1} \right\|^2 \\
	\le\frac{1}{2} \|\bar{\y} - \y_0\|^2  + \sum\limits_{k=0}^{n-1}\frac{s_u}{k+1}\left( \Psi^* -  \Psi(\z_{k+1}^{u}) \right) \\
	\le \frac{1}{2} \|\bar{\y} - \y_0\|^2 + s_u(1+\ln n) \left( \Psi^* -  M_0 \right),
	\end{array}
	\end{equation}
	where the last inequality follows from the assumption that $\inf \Psi \ge M_0$. By Lemma \ref{lem_monontone}, we have
	\begin{equation}\label{thm2_eq4}
	\begin{array}{l}
	\|\y_{k+1} - \y_{k} \|^2 \le\;  \|\y_{k} - \y_{k-1}\|^2 + \frac{\mu}{(k+1)^2}\|\y_{k-1} - \z_{k}^{u} \|^2 \\
	+ \frac{2(1-\mu) s_l c_\beta DM_{f}}{(k+1)^2}  + \frac{2\mu s_uDM_{F}}{k(k+1)} +  \frac{(1-\mu)c_\beta^2}{\underline{\beta}^2(k+1)^4}\|\y_{k-1} - \z_{k}^{l} \|^2.
	\end{array}
	\end{equation}
	and thus
	\begin{equation}\label{thm2_eq4.5}
	\begin{array}{l}
	n\|\y_{n} - \y_{n-1} \|^2 \le\sum\limits_{k=0}^{n-1}\|\y_{k+1} - \y_{k}\|^2 + \mu\sum\limits_{k=0}^{n-1}\|\y_{k} - \z_{k+1}^{u} \|^2 \\
	\!+ \frac{(1-\mu)c_\beta^2}{\underline{\beta}^2}\sum\limits_{k=0}^{n-1}\|\y_{k}\! -\! \z_{k+1}^{l} \|^2 \!+\! 2\mu s_uDM_{F} \!+\! 2(1\!-\!\mu) s_lc_\beta DM_{f}.
	\end{array}
	\end{equation}
	Then it follows from Eq. \eqref{thm2_eq2} and Eq.~\eqref{thm2_eq4.5} that
	\begin{equation*}
	\begin{array}{l}
	\min \left\{ (1 - s_lL_f),(1 - s_uL_{F}), 1 \right\}n\|\y_n - \y_{n-1}\|^2  \\
	\le \min \left\{ (1 - s_lL_f), (1 - s_uL_{F}), 1 \right\} \sum\limits_{k=0}^{n-1}\|\y_{k+1} - \y_{k}\|^2 \\
	+ \frac{c_\beta^2}{\underline{\beta}^2} (1-\mu)(1 - s_lL_{f}) \sum\limits_{k=0}^{n-1}\|\y_{k} - \z_{k+1}^{l} \|^2 + 2\mu s_uDM_{F}\\
	+ \mu(1 - s_uL_{F})\sum\limits_{k=0}^{n-1}\|\y_{k} - \z_{k+1}^{u} \|^2 + 2(1-\mu) s_lc_\beta DM_{f}\\
	\le \max\{2+ \frac{c_\beta^2}{\underline{\beta}^2},3\} \left(\|\bar{\y} - \y_0\|^2 + 2s_u(1+\ln n)\left( \Psi^* - M_0 \right) \right) \\
	+ 2\mu s_uDM_{F} + 2(1-\mu) s_lc_\beta DM_{f},
	\end{array}
	\end{equation*}
	where the second inequality comes from $\y_{k} - \y_{k+1} = (1-\mu)(\y_k - \z_{k+1}^{l}) + \mu(\y_k - \z_{k+1}^{u}) + (1-\mu)\z_{k+1}^{l} + \mu \z_{k+1}^{u} - \y_{k+1}$ and the convexity of $\|\cdot\|^2$. Combining with $\|\bar{\y} - \y_0\| \le D$, we have
	\begin{equation}\label{thm2_eq5}
	\|\y_n - \y_{n-1}\|^2 \le \frac{C_1(1+\ln n)}{n},
	\end{equation}
	where $C_1:= (\max\{2+ c_\beta^2/\underline{\beta}^2,3\} ( D^2 + 2s_u( \Psi^* - M_0 ) ) + 2\mu s_uDM_{F} + 2(1-\mu) s_lc_\beta DM_{f})/\min \{ (1 - s_lL_f),(1 - s_uL_{F}), 1 \}$.
	Next, by Lemma \ref{lem_bounded}, we have for all $k$,
	\begin{equation*}
	\|\z^l_{k+1} - \y_k\| \le \| \z^l_{k+1} - \bar{\y}\| + \| \y_{k}- \bar{\y}\| \le 2\| \y_{k}- \bar{\y}\| \le 2D.
	\end{equation*}
	Then, we have
	\begin{equation}\label{thm2_eq3}
	\begin{array}{l}
	\frac{1}{\beta_k^2}\|\z^l_{k+1} - \y_k\|^2 \\
	\le \frac{2}{\beta_{k-1}}\| \z^l_{k}- \y_{k-1}\| \| \frac{\z^l_{k+1} - \y_k}{\beta_k} - \frac{\z^l_{k}  - \y_{k-1}}{\beta_{k-1}}\|\\
	+\frac{1}{\beta_{k-1}^2}\| \z^l_{k}- \y_{k-1}\|^2 + \| \frac{\z^l_{k+1} - \y_k}{\beta_k} - \frac{\z^l_{k}  - \y_{k-1}}{\beta_{k-1}}\|^2  \\
	\le \frac{1}{\beta_{k-1}^2}\| \z^l_{k} - \y_{k-1}\|^2 + s_l^2\| \nabla \psi(\y_k) - \nabla \psi(\y_{k-1})\|^2 \\
	+ \frac{4D}{\beta_{k-1}}\|\nabla \psi(\y_k) - \nabla \psi(\y_{k-1})\| \\
	\le\! \frac{1}{\beta_{k-1}^2} \|\z^l_{k} - \y_{k-1}\|^2 + (s_l^2L_f^2D + \frac{4DL_{f}}{\underline{\beta}})\| \y_k - \y_{k-1}\|,
	\end{array}
	\end{equation}
	where the second inequality follows from the definition of $\z^l_{k}$ and the last inequality comes from $\| \y_k - \y_{k-1}\| \le D$ and $\beta_k \ge \underline{\beta}$. This implies that for any $n > n_0 >0$,
	\begin{equation*}
	\begin{array}{r}
	\frac{1}{\beta_n^2}\|\z^l_{n+1} - \y_n\|^2 
	\le (s_l^2L_f^2D + \frac{4DL_{f}}{\underline{\beta}})\sum\limits_{k= n_0+1}^{n}\| \y_{k} - \y_{k-1}\|\\
	+\frac{1}{\beta_{n_0}^2} \|\z^l_{n_0+1} - \y_{n_0}\|^2.
	\end{array}
	\end{equation*}
	Thus, since $\beta_k \in [\underline{\beta},1]$, for any $m \ge 2$ and $n_0 = n-m+1$, the following holds
	\begin{equation}\label{thm2_eq3.5}
	\begin{array}{l}
	m\underline{\beta}^2\|\z^l_{n+1} - \y_n\|^2 \\
	\!\le\! (s_l^2L_f^2D \!+\! \frac{4DL_{f}}{\underline{\beta}})\!\sum\limits_{k= n_0\!+\!1}^{n}\!(k\!-\!n_0)\!\| \y_{k} \!-\! \y_{k\!-\!1}\!\| 
	\!+\! \sum\limits_{k=\! n_0}^{n}\!\|\z_{k\!+\!1}^{l}\! -\! \y_k \|^2 \\
	\!\le\! \sum\limits_{k=n_0}^{n}\|\z_{k\!+\!1}^{l}\! -\! \y_k \|^2
	\!+\! (s_l^2L_f^2D \!+\! \frac{4DL_{f}}{\underline{\beta}})\sqrt{C_1}\frac{m(m-1)}{2} \frac{\sqrt{(1+\ln n_0)}}{\sqrt{n_0}},
	\end{array}
	\end{equation}
	where the last inequality follows from Eq.~\eqref{thm2_eq5} that $\| \y_{k} - \y_{k-1}\|^2 \le \frac{C_1(1+\ln n_0)}{n_0}$ for all $k \ge n_0$, and it can be easily verified that the above inequality holds when $m = 1$.
	By Eq.~\eqref{thm2_eq2}, we have
	\begin{equation*}
	\begin{array}{l}
	\frac{1}{2}(1-\mu)(1 - s_lL_f)\sum\limits_{k=0}^{n-1}\|\y_k - \z_{k+1}^{l}\|^2 \\
	\le \frac{1}{2} \|\bar{\y} - \y_0\|^2 + s_u(1+\ln n)\left( \Psi^* -  M_0 \right).
	\end{array}
	\end{equation*}
	Then, for any $n$, let $m$ be the smallest integer such that $m \ge n^{\frac{1}{4}}$ and let $n_0 = n-m+1$, combining the above inequality with Eq.~\eqref{thm2_eq3.5}, we have
	\begin{equation*}
	\begin{array}{l}
	\frac{\|\bar{\y} - \y_0\|^2 + 2s_u(1+\ln n)\left( \Psi^* -  M_0 \right)}{(1-\mu)(1 - s_lL_f)} 
	\ge \sum\limits_{k= n_0}^{n}\|\y_k - \z_{k+1}^{l}\|^2 \\
	\ge  m\underline{\beta}^2\|\y_n - \z_{n+1}^{l}\|^2 - C_2\frac{m(m-1)}{2} \frac{\sqrt{(1+\ln n_0)}}{\sqrt{n_0}} ,
	\end{array}
	\end{equation*}
	where $C_2 := (s_l^2L_\psi^2D + \frac{4DL_{f}}{\underline{\beta}})\sqrt{C_1}$.
	
	Next, as $n^{\frac{1}{4}}+1 \ge m \ge n^{\frac{1}{4}}$, and hence $n_0 \ge (m-1)^4 - m +1$. Then $16n_0 - m^2(m-1)^2 \ge (m-1)[(m-1)(3m-4)(5m-4)-1] > 0$ when $m \ge 2$. Thus, when $n \ge 2$, we have $m \ge 2$ and $\frac{m(m-1)}{2} \frac{\sqrt{(1+\ln n_0)}}{\sqrt{n_0}} \le 2\sqrt{(1+\ln n_0)}$. Then, let $C_3:= \frac{D^2 + 2s_u\left( \Psi^* -  M_0 \right)}{(1-\mu)(1 - s_lL_f)}$, we have for any $n \ge 2$,
	\begin{equation*}
	\begin{array}{l}
	\|\y_n - \z_{n+1}^{l}\|^2 \le \frac{1}{m\underline{\beta}^2}\left( C_3(1+\ln n)  + 2C_2\sqrt{(1+\ln n_0)}\right) \\
	\le \frac{(2C_2 + C_3)}{\underline{\beta}^2}\frac{1+\ln n}{n^{\frac{1}{4}}},
	\end{array}
	\end{equation*}
	where the last inequality follows from $\sqrt{1+\ln n_0} \le 1+\ln n$ and $m \ge n^{\frac{1}{4}}$. 	
	By the convexity of $\psi$, and $\y_n - \z_{n+1}^{l} = \beta_n s_l\nabla \psi(\y_n)$, we have
	\begin{equation*}
	\begin{array}{l}
	\psi(\y_n) \le \psi(\bar{\y}) + \langle \nabla \psi(\y_n) , \y_n - \bar{\y} \rangle \\
	= \min\psi + \frac{1}{\beta_n s_l}\langle \y_n - \z_{n+1}^{l} , \y_n - \bar{\y} \rangle \\
	 \le \min\psi + \frac{D}{\underline{\beta}^2 s_l} \sqrt{(2C_2 + C_3)\frac{1+\ln n}{n^{\frac{1}{4}}}}.
	\end{array}
	\end{equation*}
	This complete the proof. 
\end{proof}

\subsection{Approximation Quality and Convergence of BDA}
This part is devoted to the justification of the approximation quality and hence the convergence of our bi-level updating scheme (stated in Eqs.~\eqref{eq:upper_varphiK}-\eqref{eq:update-t}, with embedded $\T_k$ in Eq.~\eqref{eq:improved-lower}). Following the general proof recipe, we only need to verify that the convergence of $\T_k$ in Eq.~\eqref{eq:improved-lower} meets the \emph{UL objective convergence property} and the \emph{LL objective convergence property}.

\begin{thm}\label{thm_convergence}
	Suppose Assumptions~\ref{assum:F} is satisfied, $\X$ and $\Y$ are compact, and $\hat{\S}(\x)$ is nonempty for all $\x \in \X$. Let $\{\y_k(\x)\}$ be the output generated by \eqref{eq:improved-lower} with
	$s_l \in (0,1/L_f)$, $s_u \in (0,1/L_F)$, $\mu \in (0,1)$, $\alpha_k = \frac{1}{k+1}$, $\beta_k \in [\underline{\beta}, 1]$ with some $\underline{\beta} > 0$, $|\beta_k - \beta_{k-1}| \le \frac{c_\beta}{(k+1)^2}$ with some $c_\beta > 0$,
	then we have that both the LL and UL objective convergence properties hold. 
\end{thm}

\begin{proof}
	Since $\X$ and $\Y$ are both compact, and $F(\x,\y)$ is continuous on $\X \times \Y$, we have that $F(\x,\y)$ is uniformly bounded above on $\X \times \Y$ and thus $\min_{\y \in \Y \cap \mathcal{S}(\x)} F(\x,\y)$ is uniformly bounded above on $\X$. And combining with the assumption that $F(\x,\y)$ is uniformly bounded below with respect to $\y$ by $M_0$ for any $\x \in \X$, $\Y$ is compact, we can obtain from the Theorem \ref{simple_bilevel_complexity} that there exists $C > 0$ such that for any $\x \in \X$, we have 
	\begin{equation*}
	\begin{array}{l}
	 f(\x,\y_K(\x)) - f^*(\x) \le C\sqrt{\frac{1+\ln K}{K^{\frac{1}{4}}}}.
	\end{array}
	\end{equation*}
	As $\sqrt{\frac{1+\ln K}{K^{\frac{1}{4}}}} \rightarrow 0$ as $K \rightarrow \infty$, $\{\y_K(\x)\} \subset \Y$, and $\Y$ is compact, \emph{LL objective convergence property} holds. Next, it follows from Theorem \ref{simple_bilevel_convergence} that $\varphi_K(\x) \rightarrow \varphi(\x)$ as $K \rightarrow \infty$ for any $\x \in \X$ and thus \emph{UL objective convergence property} holds.
\end{proof}

Further, in the following, we will show that when $f(\x,\y)$ is level-bounded in $\y$ uniformly in $\x\in\X$, compactness assumption on $\Y$ in Theorem \ref{thm_convergence} can be safely removed and $\Y$ can be taken as $\mathbb{R}^m$. 

\begin{thm}\label{thm6}
	Suppose Assumptions~\ref{assum:F} is satisfied, $f(\x,\y)$ is level-bounded in $\y$ uniformly in $\x\in\X$, $\X$ is compact, and $\hat{\S}(\x)$ is nonempty for all $\x \in \X$. Let $\{\y_k(\x)\}$ be the output generated by \eqref{eq:improved-lower} with
	$s_l = s_u = s \in (0,1/\max(L_F,l_f))$, $\mu \in (0,1)$, $\alpha_k = \frac{1}{k+1}$, $\beta_k \in [\underline{\beta}, 1]$ with some $\underline{\beta} > 0$, $\beta_k \le \beta_{k-1}$, $|\beta_k - \beta_{k-1}| \le \frac{c_\beta}{(k+1)^2}$ with some $c_\beta > 0$,
	then we have that both the LL and UL objective convergence properties hold. 
\end{thm}
\begin{proof}
According to the update scheme of $\y_{k+1}$ given in Eq. \eqref{eq:improved-lower}, $\y_{k+1}$ can be equivalently regarded as
\begin{equation*}
\begin{array}{l}
\y_{k+1} = \arg\min\limits_{\y \in \Y}\langle \nabla_{\y} \phi_k(\x, \mathbf{y}_{k}), \y - \y_k\rangle + \frac{1}{2s}\|  \y - \y_k\|^2,
\end{array}
\end{equation*}
where $\phi_k(\x,\y) =  \alpha_k \mu F(\mathbf{x},\mathbf{y} ) + \beta_k(1-\mu)f(\mathbf{x},\mathbf{y} )$. Since $s \in (0,1/\max(L_F,l_f))$, $\alpha_k, \beta_k, \mu \in (0,1)$, we have $s \le 1/L_{\phi_k}$, where $L_{\phi_k}$ denotes the Lipschitz continuity constant of $\nabla_{\y} \phi_k(\x,\cdot)$. Then, \cite[Lemma 10.4]{beck2017first} yields that 
\begin{equation*}
\begin{array}{l}
\phi_k(\x,\y_{k+1}) \le \phi_k(\x,\mathbf{y}_{k}).
\end{array}
\end{equation*}
Since $f(\x,\y)$ is assumed to be level-bounded in $\y$ uniformly in $\x\in\X$, there exists $m_0$ such that $f(\x,\y)$ is bounded below by $m_0$ on $\mathbb{R}^n \times \Y$. By Assumption \ref{assum:F}, $F$ is bounded below by $M_0$. And as $\alpha_k$ and $\beta_k$ are positive and nonincreasing, it follows from the above inequality that
\begin{equation*}
\begin{array}{l}
\alpha_{k+1} \mu (F(\mathbf{x},\mathbf{y}_{k+1}) - M_0) + \beta_{k+1}(1-\mu)(f(\mathbf{x},\mathbf{y}_{k+1}) - m_0)\\
\le \alpha_{k} \mu (F(\mathbf{x},\mathbf{y}_{k}) - M_0) + \beta_{k}(1-\mu)(f(\mathbf{x},\mathbf{y}_{k}) - m_0).
\end{array}
\end{equation*}
Thus $\beta_k  \in [\underline{\beta}, 1]$ implies that $\forall k$ the following holds
\begin{equation*}
\begin{array}{l}
\underline{\beta}(1-\mu)(f(\mathbf{x},\mathbf{y}_{k}) - m_0)\\
\le \alpha_{0} \mu (F(\mathbf{x},\mathbf{y}_{0}) - M_0) + \beta_{0}(1-\mu)(f(\mathbf{x},\mathbf{y}_{0}) - m_0).
\end{array}
\end{equation*}
Since both $F$ and $f$ are continuous and $\X$ is compact, $\alpha_{0} \mu (F(\mathbf{x},\mathbf{y}_{0}) - M_0) + \beta_{0}(1-\mu)(f(\mathbf{x},\mathbf{y}_{0}) - m_0)$ is bounded on $\X$, and thus $f(\mathbf{x},\mathbf{y}_{k})$ is uniformly bounded on $\X$ for any $k$.
Then, as $f(\x,\y)$ is assumed to be level-bounded in $\y$ uniformly in $\x\in\X$, there exists $C > 0$ such that
\begin{equation*}
\| \y_{k}(\x)\| \le C, \ \forall k, \ \x \in \X.
\end{equation*}
Then by the continuity of $\nabla_y F(\x,\y)$and $\nabla_y f(\x,\y)$, there exists a compact set $\mathcal{C} \subset \mathbb{R}^m$ such that
\begin{equation*}
\hat{\y}_{k+1}(\mathbf{x}) \in \mathcal{C}, \  \forall k, \ \x \in \X,
\end{equation*}
where $\hat{\y}_{k+1}(\mathbf{x})$ is defined in Eq.~\eqref{eq:improved-lower}.
This implies that the sequence $\{\y_{k}\}$ coincides with the one generated by the update scheme in Eq. \eqref{eq:improved-lower} with $\Y = \mathcal{C}$. Then since $\mathcal{C}$ is compact, the conclusion follows from  Theorem \ref{thm_convergence} immediately.
\end{proof}

\begin{remark}
	Following the analysis recipe, the entire Section 5 is devoted to show that the constructed algorithm (i.e., BDA) meet two convergence properties. In particular, the UL and LL convergence verifications are presented in Theorem~\ref{simple_bilevel_convergence} and Theorem~\ref{simple_bilevel_complexity}, respectively. The proof of Theorem~\ref{simple_bilevel_convergence} mainly relies on the sufficiently decreasing inequality given in Lemma~\ref{simple_bilevel_lem}. Theorem~\ref{thm6} discussed the convergence behavior of BDA without the compactness assumption on $\Y$. 
\end{remark}

	\section{Stationarity Analysis of BDA}\label{subsec:stationary}
	This part provides the convergence behavior of the problem that $\min_{\x}\varphi_K(\x)$ is solved to (approximate) stationarity. 
	We consider the special case where $\Y = \mathbb{R}^m$ and the LL objective function $f(\x,\cdot) : \mathbb{R}^m \rightarrow \mathbb{R}$ is $\sigma$-strongly convex for any $\x \in \X$. In this case, the solution set of LL problem $\S(\x)$ is a singleton, and we denote its unique solution by $\y^*(\x)$. In the following, we are going to show the convergence of BDA with respect to stationary points in this special case. Our analysis is partly inspired by \cite{grazzi2020iteration}. We first make the following assumptions.
	
	\begin{assum}\label{assum2}
		$F$ and $f$ are both twice continuously differentiable.	For any $\x \in \X$,  $f(\x,\cdot) : \mathbb{R}^m \rightarrow \mathbb{R}$ is $\sigma$-strongly convex.
	\end{assum}
	Before providing the convergence results, we begin with a lemma.
	\begin{lemma}\label{lem4}
		Let $\{a_k\}$ and $\{b_k\}$ be sequences of non-negative real numbers. Assume that $b_k \rightarrow 0$ and there exist $\rho \in (0,1)$, $n_0 \in \mathbb{N}$ such that 
		$a_{k+1} \le \rho a_k + b_k, \ \forall k \ge n_0.$ 
		Then $\lim_{k \rightarrow \infty} a_k = 0$.
	\end{lemma}
	\begin{proof}
		As $b_k \rightarrow 0$, there exists $B > 0$ such that $b_k \le B$ for all $k$. And we have for any $k \ge n_0$,
		\begin{equation*}
		\begin{array}{l}
		a_{k+1} \le \rho a_k + b_k \le \rho a_k + B 
		\le \rho^{k-n_0}a_{n_0} + \frac{B}{1-\rho},
		\end{array}
		\end{equation*}
		which implies the boundedness of the sequence $\{a_k\}$ and thus there exists $A$ such that $a_k \le A$ for all $k$.
		
		For any $\epsilon > 0$, since $b_k \rightarrow 0$, there exists $k_1 > n_0$ such that $b_k \le \frac{(1-\rho)\epsilon}{2}$ for all $k \ge k_1$. And for any $k \ge k_1$,
		\begin{equation*}
		\begin{array}{l}
		a_{k+1} \le \rho a_k + b_k \le \rho a_k + \frac{(1-\rho)\epsilon}{2} 
		\le \rho^{k-k_1}A + \frac{\epsilon}{2}.
		\end{array}
		\end{equation*}
		Since $\rho \in (0,1)$, there exists $k_2 \ge k_1$ such that for any $k \ge k_2$, $\rho^{k-k_1}A \le \frac{\epsilon}{2}$ and hence $a_{k+1} \le \epsilon$. As $\epsilon$ is arbitrarily chosen, we obtain that $\lim_{k \rightarrow \infty} a_k = 0$.
	\end{proof}
	
	By applying implicit function theorem on the optimality condition of the LL problem, we obtain that $\y^*(\x)$ is differentiable on $\X$ and its derivative is given by
	\begin{equation}\label{thm7_eq2}
	\begin{array}{l}
	\frac{\partial\y^{\ast}(\x)}{\partial\x} =-\left(\nabla_{\y\y}f(\x,\y^{\ast}(\x))\right)^{-1} \nabla_{\y\x} f(\x,\y^{\ast}(\x)).
	\end{array}
	\end{equation}
	Hence, the function $\varphi(\x)=F(\x,\y^{\ast}(\x))$ is also differentiable and its derivative is given by Eq.~\eqref{eq:grad_x}. With the above Lemma~\ref{lem4}, we have the following proposition. 
	
	\begin{prop}\label{prop1}
		Suppose Assumptions~\ref{assum:F} and \ref{assum2} are satisfied, $f(\x,\y)$ is level-bounded in $\y$ uniformly in $\x\in\X$, $\X$ is compact, $\Y = \mathbb{R}^m$, and $\hat{\S}(\x)$ is nonempty for all $\x \in \X$. Let $\{\y_k(\x)\}$ be the output generated by \eqref{eq:improved-lower} with
		$s_l = s_u = s \in (0,1/\max(L_F,l_f))$, $\mu \in (0,1)$, $\alpha_k > 0$, $\alpha_k \le \alpha_{k-1}$, $\lim_k \alpha_k = 0$,
		$\beta_k \in [\underline{\beta}, 1]$ with some $\underline{\beta} > 0$, $\beta_k \le \beta_{k-1}$, $|\beta_k - \beta_{k-1}| \le \frac{c_\beta}{(k+1)^2}$ with some $c_\beta > 0$,
		then we have 
		\begin{equation*}
		\begin{array}{l}
		\sup\limits_{\x \in \X} \| \nabla \varphi_k(\x) - \nabla \varphi(\x) \| \rightarrow 0,\ \text{as}\ k \rightarrow \infty.
		\end{array}
		\end{equation*}
	\end{prop}
	\begin{proof}
		According to the update scheme of $\y_{k+1}$ given in Eq. \eqref{eq:improved-lower}, and since $\Y = \mathbb{R}^m$, we have
		\begin{equation}\label{thm7_eq1}
		\y_{k+1} =\y_k - s \nabla_{\y} \phi_k(\x, \mathbf{y}_{k}),
		\end{equation}
		where $\phi_k(\x,\y) = \alpha_k \mu F(\mathbf{x},\mathbf{y} ) + \beta_k(1-\mu)f(\mathbf{x},\mathbf{y} )$. And we have
		\begin{equation*}
		\y_{k+1} - \y^* \!=\! \y_k - s\beta_k(1-\mu) \nabla_{\y} f(\mathbf{x},\y_k) - \y^* - s\alpha_k \mu \nabla_{\y}F(\mathbf{x},\mathbf{y}_k ).
		\end{equation*}
		As $\mu \in (0,1)$, $ \beta_k \in [\underline{\beta}, 1]$, $s \in (0,1/\max(L_F,l_f))$, $ \nabla_{\y} f(\mathbf{x},\mathbf{y}^*) = 0$, and $f(\x,\cdot) : \mathbb{R}^m \rightarrow \mathbb{R}$ is assumed to be $\sigma$-strongly convex, \cite[Theorem 10.29]{beck2017first} implies
		\begin{equation*}
		\|\y_k - s\beta_k(1-\mu) \nabla_{\y} f(\mathbf{x},\y_k) - \y^*\| \le (1 - s\underline{\beta}(1-\mu) \sigma )\| \y ^k - \y^*\|.
		\end{equation*}
		Let $\rho:= 1 - s\underline{\beta}(1-\mu) \sigma $, then $\rho \in (0, 1)$ and
		\begin{equation*}
		\begin{aligned}
		\|\y_{k+1} - \y^*\| \le \rho\|\y_k - \y^*\| + s\alpha_k \mu \| \nabla_{\y}F(\mathbf{x},\mathbf{y}_k )\|.
		\end{aligned}
		\end{equation*}
		As shown in Theorem \ref{thm6}, there exists a compact set $\mathcal{C}$ such that $\mathbf{y}_k(\x) \in \mathcal{C}$ for any $k$ and $\x \in \X$. Then by the continuity of $\nabla_y F(\x,\y)$ on $\X\times\mathbb{R}^m$, there exists $C >0$ such that $\| \nabla_{\y}F(\mathbf{x},\mathbf{y}_k(\x) )\| \le C$ for any $k$ and $\x \in \X$. Then we have
		\begin{equation}
		\|\y_{k+1}(\x) - \y^*(\x)\| \le \rho\|\y_k(\x) - \y^*(\x)\| + s \mu C \alpha_k, \ \forall k, \x \in \X.
		\end{equation}
		Thus, $\forall k\in\mathbb{N}$, the following holds
		\begin{equation*}
		\sup_{\x \in \X}\|\y_{k+1}(\x) - \y^*(\x)\| \le \rho\sup_{\x \in \X}\|\y_k(\x) - \y^*(\x)\| + s \mu C \alpha_k.
		\end{equation*}
		Since $\alpha_k \rightarrow 0$, we obtain from Lemma \ref{lem4} that 
		\begin{equation*}
		\sup_{\x \in \X} \|\y_{k}(\x) - \y^*(\x)\| \rightarrow 0,\ \text{as}\ k \rightarrow \infty.
		\end{equation*}
		By taking derivative with respect to $\x$ on both sides of Eq.~\eqref{thm7_eq1}, we get 
		\begin{equation*}
		\begin{array}{l}
		\frac{\partial \y_{k+1}(\x)}{\partial\x}=\left(I-s\nabla_{\y\y}\phi_k(\x,\y_k)\right)\frac{\partial \y_{k}(\x)}{\partial\x} -s\nabla_{\y\x}\phi_k(\x,\y_k).
		\end{array}
		\end{equation*}
		Combining with Eq. \eqref{thm7_eq2}, we have
		\begin{equation*}
		\begin{array}{l}
		\frac{\partial \y_{k+1}(\x)}{\partial\x} - \frac{\partial\y^{\ast}(\x)}{\partial\x}= \left(I\!-\!s\nabla_{\y\y}\phi_k(\x,\y_k)\right)\left( \frac{\partial \y_{k}(\x)}{\partial\x} - 	\frac{\partial\y^{\ast}(\x)}{\partial\x} \right)\\ 
		- s\beta_k(1-\mu)\left(\mathtt{dis}_{\y\x}^f(\x,\y_k) - \mathtt{dis}_{\y\y}^f(\x,\y_k)\frac{\partial\y^{\ast}(\x)}{\partial\x}\right) \\
		- s \alpha_k \mu \nabla_{\y\x} F(\mathbf{x}, \y_{k}) - s \alpha_k \mu\nabla_{\y\y} F(\mathbf{x}, \y_{k})\frac{\partial\y^{\ast}(\x)}{\partial\x}.
		\end{array}
		\end{equation*}
		where $\mathtt{dis}_{\y\x}^f(\x,\y_k):=\nabla_{\y\x}f(\mathbf{x},\y_{k}) - \nabla_{\y\x}f(\mathbf{x},\y^*)$ and $\mathtt{dis}_{\y\y}^f(\x,\y_k):=\nabla_{\y\y}f(\mathbf{x},\y_{k}) - \nabla_{\y\y}f(\mathbf{x},\y^*)$. 
		Since $F(\x,\cdot)$ is convex and $f(\x,\cdot)$ is assumed to be $\sigma$-strongly convex, $\nabla_{\y\y} F(\mathbf{x}, \y_{k}) \succeq 0$ and $\nabla_{\y\y}f(\mathbf{x},\y_{k}) \succeq \sigma I$ for any $\x \in \X$ and $k$. Combining with $\mu \in (0,1)$, $ \beta_k \in [\underline{\beta}, 1]$ and $s \in (0,1/\max(L_F,l_f))$, we have
		\begin{equation*}
		\|I-s\nabla_{\y\y}\phi_k(\x,\y_k) \| \le \rho,
		\end{equation*}
		where $\rho = 1 - s\underline{\beta}(1-\mu) \sigma \in (0,1)$. Then, it follows from the above inequality that
		\begin{equation*}
		\begin{array}{l}
		\sup\limits_{\x \in \X}\left\| \frac{\partial \y_{k+1}(\x)}{\partial\x} - \frac{\partial\y^{\ast}(\x)}{\partial\x} \right\|
		\le\rho \sup\limits_{\x \in \X}\left\| \frac{\partial \y_{k}(\x)}{\partial\x} -\frac{\partial\y^{\ast}(\x)}{\partial\x} \right\| \\ 
		+ s\beta_k(1-\mu)\sup\limits_{\x \in \X}\left\| \mathtt{dis}_{\y\x}^f(\x,\y_k) \right\| +s \alpha_k \mu \sup\limits_{\x\in\X}\Gamma(\x)\\
		+s\beta_k(1-\mu)\sup\limits_{\x \in \X}\left\|\mathtt{dis}_{\y\y}^f(\x,\y_k)\right\| \left\| \frac{\partial\y^{\ast}(\x)}{\partial\x} \right\|.
		\end{array}
		\end{equation*}
		where $\Gamma_k(\x)\!=\!\left\|\! \nabla_{\y\x}\! F(\mathbf{x}, \y_{k}(\x))\right\| \!+\! \left\|\! \nabla_{\y\y}\! F(\x, \y_{k}(\x))\right\|\! \left\|\! \frac{\partial\y^{\ast}(\x)}{\partial\x}\! \right\|\!$. 
		Next, we are going to show that the last three terms in the right hand side of the above inequality converge to $0$ as $k \rightarrow \infty$.
		
		First, as discussed above that there exists a compact set $\mathcal{C}$ such that $\mathbf{y}_k(\x) \in \mathcal{C}$ for any $k$ and $\x \in \X$. Since $\nabla_{\y\x}f$ and $ \nabla_{\y\y}f$ are both continuous on $\X\times\mathbb{R}^m$ and $\X, \mathcal{C}$ are compact, $\nabla_{\y\x}f$ and $ \nabla_{\y\y}f$ are both uniformly continuous on $\X\times\mathcal{C}$, then, the fact that $\sup_{\x \in \X} \|\y_{k}(\x) - \y^*(\x)\| \rightarrow 0$ as $k \rightarrow \infty$ implies that $\sup_{\x \in \X}\left\|\mathtt{dis}_{\y\x}^f(\x,\y_k) \right\| \rightarrow 0$ and $\sup_{\x \in \X}\left\|\mathtt{dis}_{\y\y}^f(\x,\y_k)\right\| \rightarrow 0$ as $k \rightarrow \infty$. $\sigma$-strong convexity of $f(\x,\cdot)$ yields the continuity of $\frac{\partial\y^{\ast}(\x)}{\partial\x}$ on $\X$ and thus $\sup_{\x \in \X} \left\| \frac{\partial\y^{\ast}(\x)}{\partial\x} \right\| < + \infty$. Then, we have $\sup_{\x \in \X}\left\| \mathtt{dis}_{\y\x}^f(\x,\y_k)\right\| \left\| \frac{\partial\y^{\ast}(\x)}{\partial\x} \right\| \rightarrow 0$ as $k \rightarrow \infty$. 
		Next, as $\nabla_{\y\x}F$ and $ \nabla_{\y\y}F$ are both continuous on $\X\times\mathbb{R}^m$, $\X, \mathcal{C}$ are compact, and $\mathbf{y}_k(\x) \in \mathcal{C}$ for any $k$ and $\x \in \X$, then 
		\begin{equation*}
		\begin{array}{l}
		\sup\limits_{\x \in \X}\left( \left\| \nabla_{\x\y} F(\mathbf{x}, \y_{k}(\x))\right\| + \left\| \nabla_{\y\y} F(\mathbf{x}, \y_{k}(\x))\right\|  \right) < + \infty.
		\end{array}
		\end{equation*}
		Combining with the fact that $\sup_{\x \in \X} \left\| \frac{\partial\y^{\ast}(\x)}{\partial\x} \right\| < + \infty$, we have $\sup\limits_{\x \in \X}\Gamma_k(\x)< +\infty$. 
		Because $\alpha_k \rightarrow 0$ as $k \rightarrow \infty$, we have $\alpha_k \mu \sup_{\x \in \X}\Gamma_k(\x)\rightarrow 0$ as $k \rightarrow \infty$. 
		According to Lemma \ref{lem4}, we have 
		\begin{equation}\label{thm7_eq3}
		\sup_{\x \in \X}\left\| \frac{\partial \y_{k}(\x)}{\partial\x} - 	\frac{\partial\y^{\ast}(\x)}{\partial\x} \right\|  \rightarrow 0, \ \text{as}\ k \rightarrow \infty.
		\end{equation}
		
		Recalling Eq.~\eqref{eq:grad_x} and
		\begin{equation*}
		\nabla \varphi_k(\x) = \nabla_{\x} F(\x,\y_k(\x)) + \left(\frac{\partial\y_k(\x)}{\partial\x}\right)^{\top} \nabla_{\y}F(\x,\y_k(\x)),
		\end{equation*}
		we have the following estimate
		\begin{equation*}
		\begin{array}{l}
		\sup\limits_{\x \in \X} \| \nabla \varphi_k(\x) \!-\! \nabla \varphi(\x) \| \le \sup\limits_{\x \in \X} \| \mathtt{dis}_{\y}^{F}(\x,\y_k)\|  \left\| \frac{\partial\y^{\ast}(\x)}{\partial\x} \right\|+\\
	 \sup\limits_{\x \in \X}\! \left\| \frac{\partial \y_{k}(\x)}{\partial\x}\! -\! 	\frac{\partial\y^{\ast}(\x)}{\partial\x}\! \right\| \| \nabla_{\y}F(\x,\y_k(\x)) \| \!+\!\sup\limits_{\x \in \X}\! \| \mathtt{dis}_{\x}^{F}(\x,\y_k) \|. 
		\end{array}
		\end{equation*}
		where $\mathtt{dis}_{\y}^{F}(\x,\y_k)=\nabla_{\y}F(\x,\y_k(\x)) - \nabla_{\y}F(\x,\y^{\ast}(\x))$ and $\mathtt{dis}_{\x}^{F}(\x,\y_k)=\nabla_{\x} F(\x,\y_k(\x)) \!-\! \nabla_{\x} F(\x,\y^{\ast}(\x))$. 
		Since $\nabla_{\x} F$ and $\nabla_{\y} F$ are continuous on $\X\times\mathbb{R}^m$ and thus uniformly continuous on compact set $\X\times\mathcal{C}$. Then the fact that $\sup_{\x \in \X} \|\y_{k}(\x) - \y^*(\x)\| \rightarrow 0$ as $k \rightarrow \infty$ implies that $\sup_{\x \in \X} \| \mathtt{dis}_{\x}^{F}(\x,\y_k) \| \rightarrow 0$ and $\sup_{\x \in \X} \| \mathtt{dis}_{\y}^{F}(\x,\y_k)\| \rightarrow 0$ as $k \rightarrow \infty$.
		The continuity of $\frac{\partial\y^{\ast}(\x)}{\partial\x}$ on $\X$ yields that $\sup_{\x \in \X} \left\| \frac{\partial\y^{\ast}(\x)}{\partial\x} \right\| < + \infty$ and thus $ \sup_{\x \in \X} \| \mathtt{dis}_{\y}^{F}(\x,\y_k)\|  \left\| \frac{\partial\y^{\ast}(\x)}{\partial\x} \right\| \rightarrow 0$ as $k \rightarrow \infty$.
		Next, as $\nabla_{\y}F$ is continuous on $\X\times\mathbb{R}^m$ and $\y_k(\x)$ belongs to a compact set for any $k$ and $\x \in \X$, it holds that $\sup_{\x \in \X}\| \nabla_{\y}F(\x,\y_k(\x)) \| < + \infty$ for any $k$. Then Eq. \eqref{thm7_eq3} implies that $ \sup_{\x \in \X} \left\| \frac{\partial \y_{k}(\x)}{\partial\x} - 	\frac{\partial\y^{\ast}(\x)}{\partial\x} \right\| \cdot \sup_{\x \in \X}\| \nabla_{\y}F(\x,\y_k(\x)) \| \rightarrow 0$ as $k \rightarrow \infty$. Then we get the conclusion directly from Lemma \ref{lem4}.
	\end{proof}
	
	\begin{thm}
		Suppose Assumptions~\ref{assum:F} and \ref{assum2} are satisfied, $f(\x,\y)$ is level-bounded in $\y$ uniformly in $\x\in\X$, $\X$ is compact, $\Y = \mathbb{R}^m$, and $\hat{\S}(\x)$ is nonempty for all $\x \in \X$. Let $\{\y_k(\x)\}$ be the output generated by \eqref{eq:improved-lower} with
		$s_l = s_u = s \in (0,1/\max(L_F,l_f))$, $\mu \in (0,1)$, $\alpha_k \le \alpha_{k-1}$, $\lim_k \alpha_k = 0$,
		$\beta_k \in [\underline{\beta}, 1]$ with some $\underline{\beta} > 0$, $\beta_k \le \beta_{k-1}$, $|\beta_k - \beta_{k-1}| \le \frac{c_\beta}{(k+1)^2}$ with some $c_\beta > 0$, and let $\x_K$ be a $\varepsilon_K$-stationary point of $\varphi_{K}(\x)$, i.e.,
		\begin{equation*}
		\varepsilon_K = \nabla \varphi_K(\x_K).
		\end{equation*}
		Then if $\varepsilon_K \rightarrow 0$, we have that any limit point $\bar{\x}$ of the sequence $\{\x_K\}$ is a stationary point of $\varphi$, i.e., 
		\begin{equation*}
		0 = \nabla \varphi(\bar{\x}).
		\end{equation*}
	\end{thm}
	\begin{proof}
		For any limit point $\bar{\x}$ of the sequence $\{\x_K\}$, let $\{\x_{l}\}$ be a subsequence of $\{\x_K\}$ such that $\x_{l} \rightarrow \bar{\x} \in \X$. For any $\epsilon > 0$, as shown in Proposition \ref{prop1}, there exists $k_1$ such that 
		\begin{equation*}
		\sup_{\x \in \X} \| \nabla \varphi_k(\x) - \nabla \varphi(\x) \| \le \epsilon/2, \quad \forall k \ge k_1.
		\end{equation*}
		Since $\varepsilon_k \rightarrow 0$, there exists $k_2 > 0$ such that $\varepsilon_k \le  \epsilon/2$ for any $k \ge k_2$. Then, for any $l \ge \max(k_1,k_2)$, we have
		\begin{equation*}
		\|\nabla \varphi(\x_l) \| \le \|\nabla \varphi(\x_l) - \nabla \varphi_l(\x_l) \| + \| \nabla \varphi_l(\x_l)  \| \le  \epsilon.
		\end{equation*}
		Taking $l \rightarrow \infty$ in the above inequality, and by the continuity of $\nabla \varphi$, we get $\|\nabla \varphi(\bar{\x}) \| \le \epsilon$. 
		Since $\epsilon$ is arbitrarily chosen, we get $0 = \nabla \varphi(\bar{\x})$.
	\end{proof}

\section{Discussions}\label{sec:revisit}

This section first provides a comparison with the existing LLS scheme by a high dimension counter-example in Section~\ref{subsec:comp_LLS}. Then, Section~\ref{subsec:comp_conference} lists the improvements of this work. Finally, we develops an one-stage extension scheme in Section~\ref{sec:extension}.

\subsection{Comparison with Existing LLS Theories}\label{subsec:comp_LLS}

As aforementioned, a number of gradient-based methods have been proposed to solve BLO in Eq.~\eqref{eq:blp}. However, these existing methods all rely on the uniqueness of $\S(\x)$ (i.e., LLS assumption). That is, rather than considering the original BLO in Eq.~\eqref{eq:blp}, they actually solve the simplification in Eq.~\eqref{eq:blp_lls}. By considering $\y$ as a function of $\x$, the idea behind these approaches is to take a gradient-based first-order scheme (e.g, gradient descent, stochastic gradient descent, or their variations) on the LL subproblem. Therefore, with the initialization point $\y_0$, a sequence $\{\mathbf{y}_{k}(\x)\}_{k=0}^K$ parameterized by $\x$ can be generated, e.g.,
\begin{equation}
\begin{array}{l}
\mathbf{y}_{k+1}(\x)=\mathbf{y}_{k}(\x)-s_l\nabla_{\mathbf{y}} f(\mathbf{x},\mathbf{y}_{k}(\x)),\ k=0,\cdots,K-1,\label{eq:gradient_f}
\end{array}
\end{equation}
where $s_l>0$ is an appropriately chosen step size. Then by considering $\y_K(\x)$ (i.e., the output of Eq.~\eqref{eq:gradient_f} for a given $\x$) as an approximated optimal solution to the LL subproblem, we can incorporate $\y_K(\x)$ into the UL objective and obtain a single-level approximation model, i.e., $\min_{\mathbf{x}\in\mathcal{X}} F(\mathbf{x},\y_K(\mathbf{x}))$. Next, by unrolling the iterative update scheme in Eq.~\eqref{eq:gradient_f}, we can calculate the derivative of $F(\mathbf{x},\y_K(\mathbf{x}))$ (w.r.t. $\x$) to optimize Eq.~\eqref{eq:blp_lls} by automatic differentiation techniques~\cite{franceschi2017forward,baydin2017automatic}. 

\begin{figure}[t]
	\centering 
	\begin{tabular}{c@{\extracolsep{0.2em}}c}
		\includegraphics[width=0.228\textwidth]{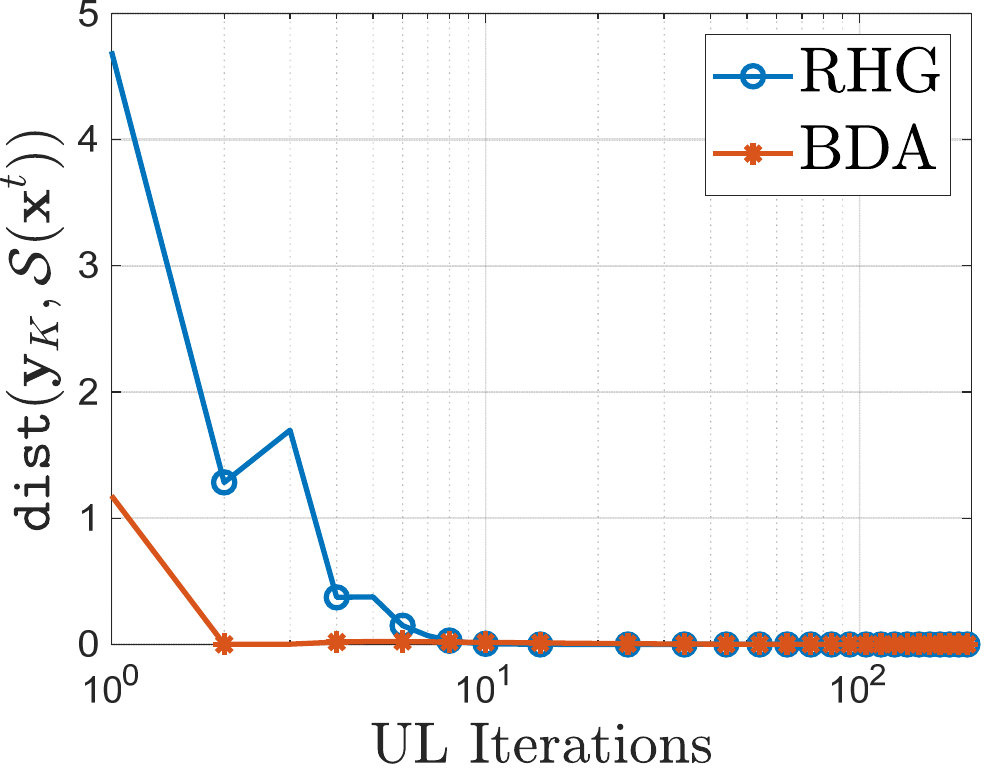}
		&\includegraphics[width=0.228\textwidth]{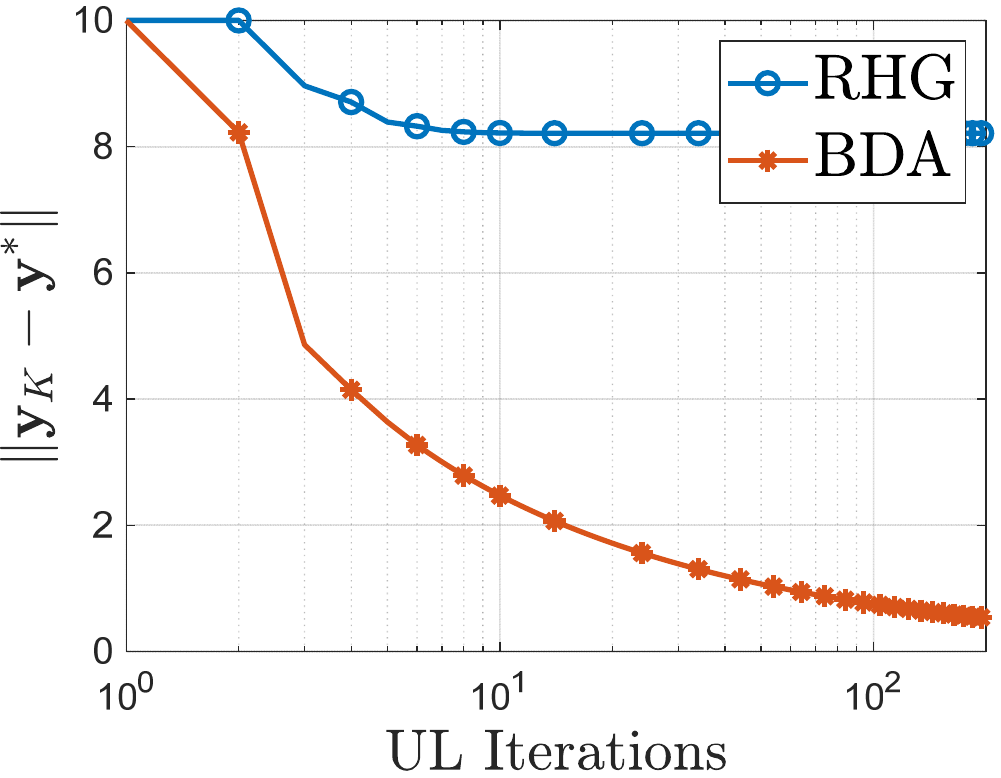}\\
	\end{tabular}
	\caption{An evaluation of the convergence behavior about the LL variable $\mathbf{y}$. We compare our BDA with gradient-based BLO algorithm (i.e., RHG). We set the initial points ($\x$, $\y$) = (0, 0), $n=50$ and $K=20$. $\x^t$ denotes the UL variable at the $t$-th UL iterations.} \label{fig:conver_y}
\end{figure}

\begin{figure*}[t]
	\centering 
	\begin{tabular}{c@{\extracolsep{0.2em}}c@{\extracolsep{0.2em}}c@{\extracolsep{0.2em}}c}
		\includegraphics[width=0.238\textwidth]{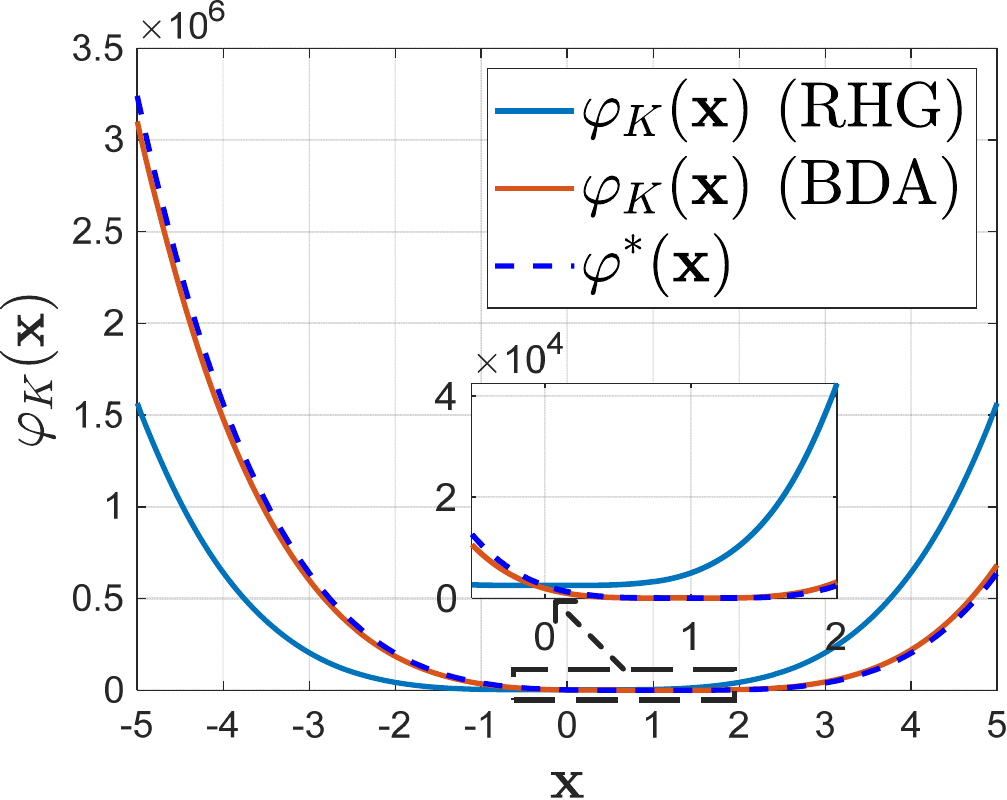}
		&\includegraphics[width=0.244\textwidth]{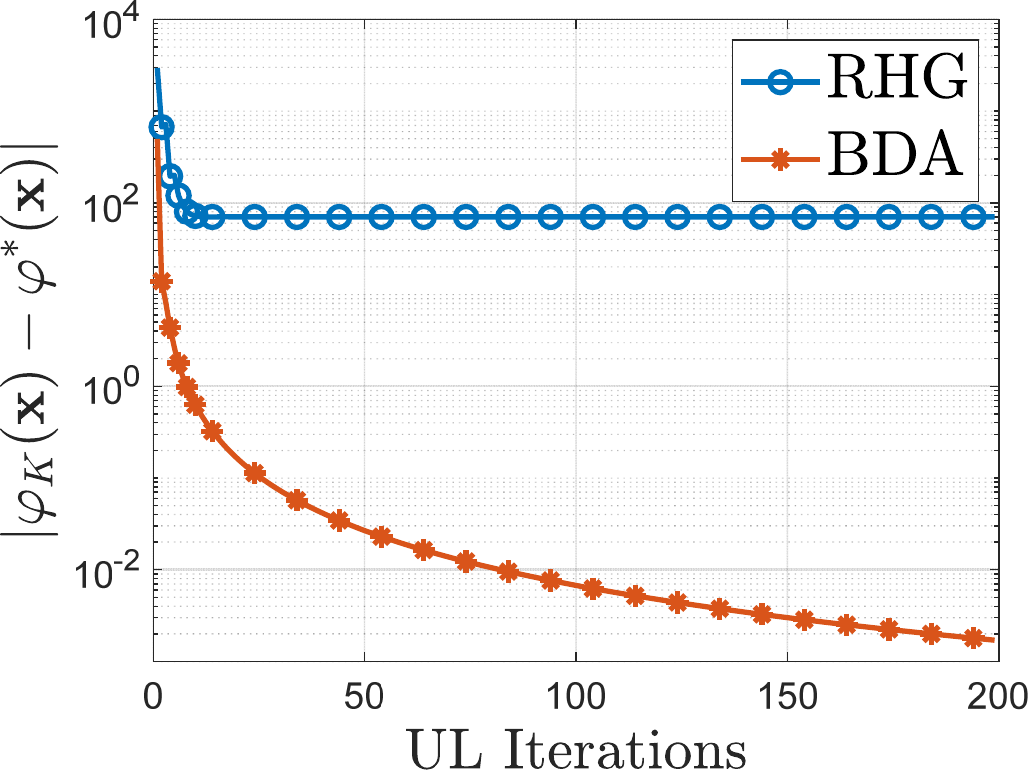}
		&\includegraphics[width=0.236\textwidth]{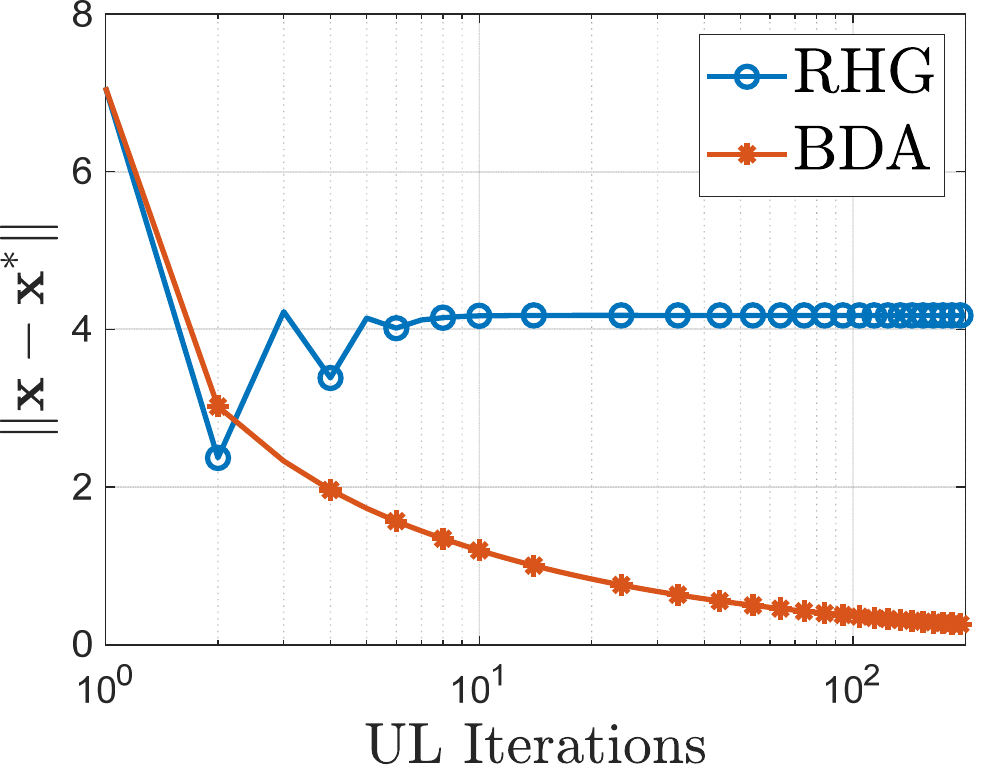}
		&\includegraphics[width=0.24\textwidth]{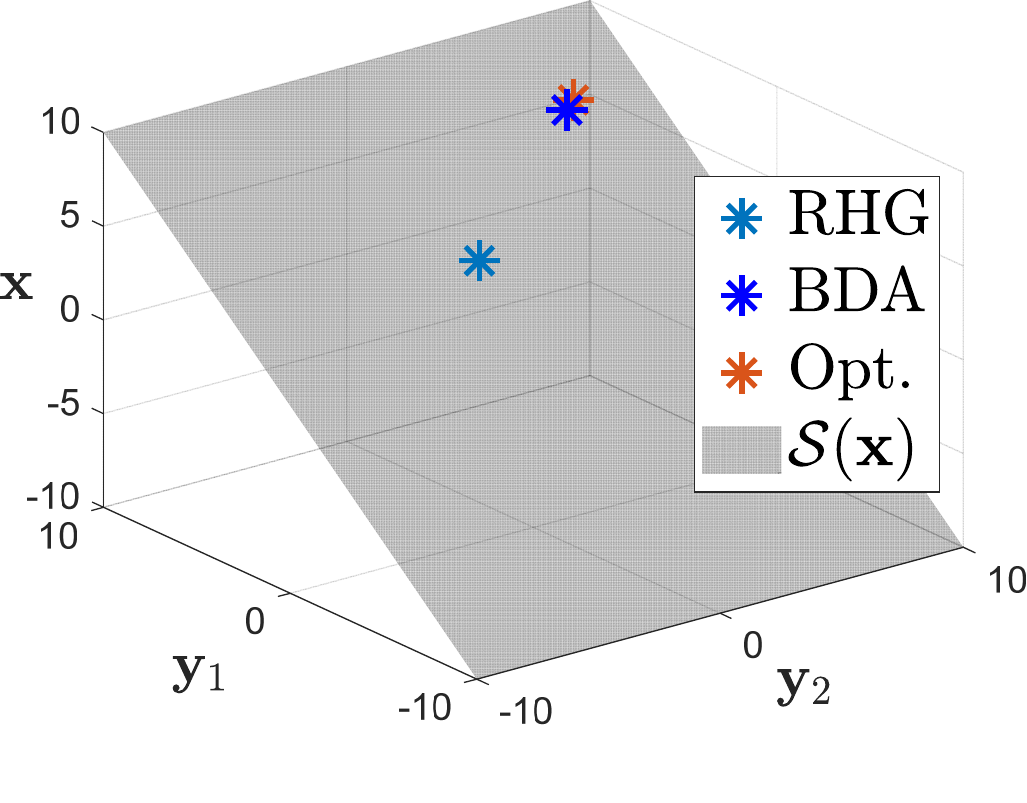}\\
	\end{tabular}
	\caption{Illustrating the convergence behavior of gradient-based BLO algorithms about the UL variable $\x$. We set the initial points ($\x$, $\y$) = (0, 0), $n=50$ and $K=20$. In the first subfigure, $\varphi_K(\x)$ and $\varphi^*(\x)$ denote the UL objective with LL computational solution $\y_K$ and the optimal LL solution $\y^*$ respectively. The second and third subfigures respectively show the errors of UL objective (i.e., $|\varphi_K(\x)-\varphi^*(\x)|$) and UL variable (i.e., $\|\x-\x^*\|$). The last subfigure illustrates the relationship among Optimal solution (short for ``Opt.'', the red star) and the iteration solutions of RHG and BDA. } \label{fig:conver_x}
\end{figure*}

The UL objective $F$ is indeed a function of both the UL variable $\x$ and the LL variable $\y$. Conventional gradient-based bi-level methods (Eq.~\eqref{eq:gradient_f}) only use the gradient information of the LL subproblem to update $\mathbf{y}$. Thanks to the LLS assumption, for fixed UL variable $\x$, the LL solution $\y$ can be uniquely determined. Thus the sequence $\{\mathbf{y}_{k}\}_{k=0}^K$ could converge to the true optimal solution, that minimizes both the LL and UL objectives. However, when LLS is absent, $\{\mathbf{y}_{k}\}_{k=0}^K$ may easily fail to converge to the true solution. Therefore, $\x_K^*$ may tend to be incorrect limiting points.

\begin{example}(Counter-Example)\label{CExam}
	With $\x\in\mathbb{R}^{n}$, $\y\in\mathbb{R}^{n}$ and $\z\in\mathbb{R}^{n}$, we consider the following BLO problem:
	\begin{equation}
	\begin{array}{l}
	\min\limits_{\x\in \X}\| \x - \z \|^4+\|\y-\mathbf{e}\|^4,\\
	s.t. \ (\y,\z) \in\arg\min\limits_{\y\in\mathbb{R}^{n},\z\in\mathbb{R}^{n}}\frac{1}{2}\|\y\|^2 - \x^{\top}\y,
	\end{array}\label{eq:ce-2}
	\end{equation}
	where $\X = [-100,100]\times\cdots [-100,100] \subset \mathbb{R}^n$, $\mathbf{e}$ denotes the vector whose elements are all equal to $1$. By simple calculation, we know that the unique optimal solution of Eq.~\eqref{eq:ce-2} is $\x^* = \y^* =\z^* =\mathbf{e}$. However, if adopting the existing gradient-based scheme in Eq.~\eqref{eq:gradient_f} with initialization $(\y_0,\z_0)=(0,0)$ and varying step size $s_l^k \in (0,1)$, we have that $\y_K = (1-\prod_{k=0}^{K-1}(1-s_l^k))\x$ and $\z_K=0$. Then the approximated problem of Eq.~\eqref{eq:ce-2} amounts to \begin{equation*}
	\begin{array}{l}
		\min\limits_{\mathbf{x}\in \X} F(\x,\y_K,\z_K)= \|\x\|^4 + \| (1-\prod_{k=0}^{K-1}(1-s_l^k)) \x- \mathbf{e}\|^4.
	\end{array}
	\end{equation*}
	Consider sequence  
	$\x_K^{*}=\arg\min_{\x\in \X} F(\x,\y_K,\z_K),$  
	it follows from the first-order optimality condition that, 
	\begin{equation}\label{eq_counterex}
	0 = 4\|\x_K^{*}\|^2\x_K^{*} +  4a_K\| a_K \x_K^{*}- \mathbf{e} \|^2(a_K\x_K^{*}- \mathbf{e}),
	\end{equation} 
	where $a_K = (1-\prod_{k=0}^{K-1}(1-s_l^k))$. 
	Then, if sequence $\{\x_K^*\}$ converge to a limit point $\mathbf{e}$, and since $\{a_K\}$ is bounded, there exist subsequences $\{\x_{K_\ell}^*\} \subset \{\x_{K}^*\} $ and $ \{ a_{K_\ell} \} \subset \{ a_K \} $ such that $\{\x_{K_\ell}^*\} \rightarrow \mathbf{e}$ and $\{ a_{K_\ell} \} \rightarrow \bar{a}$. By considering subsequences $\{\x_{K_\ell}^*\}$ and $\{ a_{K_\ell} \}$ in Eq.~\eqref{eq_counterex} and taking $K_\ell \rightarrow \infty$, we should have
	\begin{equation*}
		\begin{array}{l}
		0 = \|\mathbf{e}\|^2\mathbf{e} +  \bar{a}\| \bar{a} \mathbf{e}- \mathbf{e} \|^2(\bar{a}\mathbf{e}- \mathbf{e})\\
		\quad=  [1+(\bar{a}-1)^3\bar{a}] \|\mathbf{e}\|^2\mathbf{e},
		\end{array}
	\end{equation*}
	and thus $ 0 = 1+(\bar{a}-1)^3\bar{a}$. However, since $a_K = (1-\prod_{k=0}^{K-1}(1-s_l^k)) \in [0,1]$, then $\bar{a} \in [0,1]$ and 
	\begin{equation*}
	1+(\bar{a}-1)^3\bar{a} \ge 1 - |(\bar{a}-1)\bar{a}| \ge \frac{3}{4} > 0,
	\end{equation*} 
	which is a contradiction to $0 = 1+(\bar{a}-1)^3\bar{a}$.
	Therefore, any subseuqnce of $\{\x_K^*\}$ cannot converge to the true solution (i.e., $\x^* = \mathbf{e}$). 
\end{example}

\begin{remark}\label{remark1}
	Actually, even with strongly convex UL objective w.r.t. LL variable $\y$, the existing bi-level based methods still may fail to reach an optimal solution. For example, with $\x\in[-100,100]$ and $\y\in\mathbb{R}^2$, we consider the following BLO problem:
		\begin{equation*}
		\begin{array}{c}
		\min\limits_{\x\in[-100,100]}\frac{1}{2}\|\x-\y_2\|^2+\frac{1}{2}\|\y_{1}-1\|^2,\\
		s.t. \ \y\in\arg\min\limits_{\y\in\mathbb{R}^2}\frac{1}{2}\|\y_1\|^2 - \x^{\top}\y_1.
		\end{array}\label{eq:ce}
		\end{equation*}
		By simple calculation, we know that the unique optimal solution of Eq.~\eqref{eq:ce} is $\x^*=1, \y^* = (1,1)$. However, if adopting the existing gradient-based scheme in Eq.~\eqref{eq:gradient_f} with initialization $\y_0=(0,0)$ and varying step size $s_l^k  \in (0,1)$, we have that $[\y_K]_1 = (1-\prod_{k=0}^{K-1}(1-s_l^k))\x$ and $[\y_K]_2=0$. By defining $\varphi_K(\x) = F(\x,\y_K)$, we have
		$
		\x_K^{*}=\frac{(1- \prod_{k=0}^{K-1}(1-s_l^k) )}{1+(1- \prod_{k=0}^{K-1}(1-s_l^k) )^2}.
		$ 
		It is easy to check that 
		$\x_K^{*}\le \frac{1}{2}$.  
		So $\x_K^*$  cannot converge to the true solution (i.e., $\x^* = 1$).
\end{remark}

\begin{remark}
	In applications, to achieve the LLS, people sometimes add a strongly convex regularization term to the LL subproblem. We must clarify that this strategy is only heuristic, which usually causes unpredictable large deviation from the true solution.

	 Indeed, even the strongly convex regularization is set to be vanishing, such an approximation procedure cannot guarantee any convergence to the true solution. We will take the counter-example in Remark \ref{remark1} again for illustration. Specifically, we introduce a quadratic term $1/2 \varepsilon \|\y_2\|^2$ to the LL subproblem
	\begin{equation*}
	\min\limits_{\y \in\mathbb{R}^{2}}\frac{1}{2}\|\y_1\|^2 + \frac{1}{2} \varepsilon \|\y_2\|^2- \x^{\top}\y_1.
	\end{equation*}
	Apparently, the LL objective becomes strongly convex. But it can be checked that the optimal solution to such bilevel problem with regularized LL
	\begin{equation*}
	\begin{array}{l}
	\min\limits_{\x\in [-100,100]}\frac{1}{2}\| \x - \y_2 \|^2+\frac{1}{2}\|\y_1-1\|^2,\\
	s.t. \ \y \in\arg\min\limits_{\y\in\mathbb{R}^{2}}\frac{1}{2}\|\y_1\|^2 + \frac{1}{2} \varepsilon \|\y_2\|^2 - \x^{\top}\y_1,
	\end{array}
	\end{equation*}
	becomes $\x^*(\varepsilon)=\frac{1}{2}$,  $\y_1^*(\varepsilon)=\frac{1}{2}$, $\y_2^*(\varepsilon)=0$ which is obviously no longer the true solution to the original bilevel model. Moreover, even with $\varepsilon$ tending $0$, unfortunately, $\x^*(\varepsilon)$, $\y_1^*(\epsilon)$ and $\y_2^*(\epsilon)$ still fail to converge to the true solution $(1, 1, 1)$. 	
\end{remark}

To demonstrate the convergence behavior of our BDA and the most popular bi-level method (i.e., RHG~\cite{franceschi2017forward,franceschi2018bilevel}), we first illustrate the optimization procedure of LL variable (i.e., $\y_{K}$) in Figure~\ref{fig:conver_y}. As can be observed that the LL variable $\y_{K}$ can converge to the LL solution set $\S(\x^t)$ for both RHG and our BDA in the left subfigure. But, the LL variable of our method can find the optimal point, i.e., $\y^*$, while RHG cannot. Note that we set the dimension $n=50$.

\begin{figure}[t]
	\centering \begin{tabular}{c@{\extracolsep{0.2em}}c@{\extracolsep{0.2em}}c}
		\footnotesize $\x_0$ & \footnotesize $\x_{5}$ & \footnotesize $\x_{20}$\\
		\includegraphics[width=0.16\textwidth]{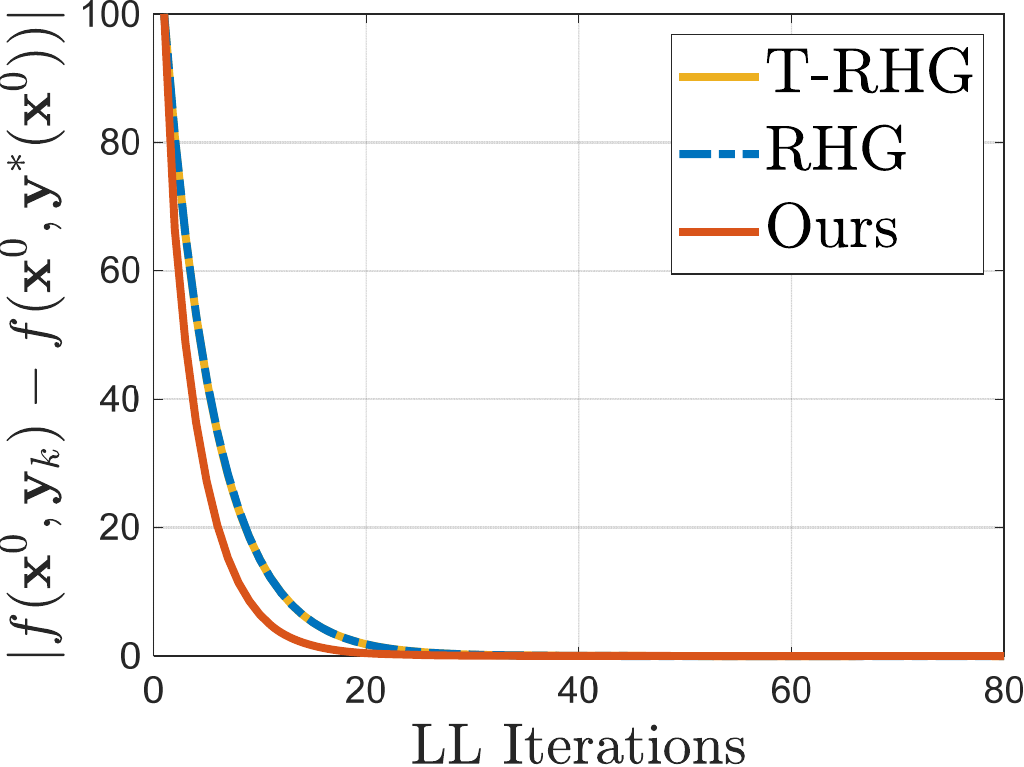}
		&\includegraphics[width=0.16\textwidth]{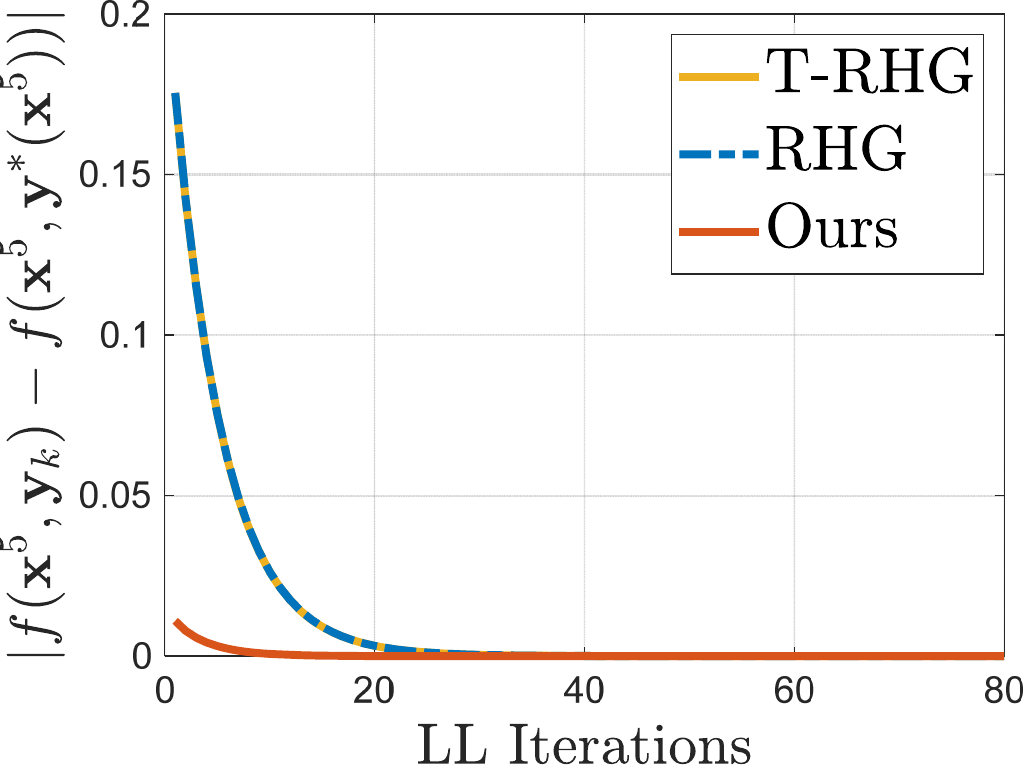}
		&\includegraphics[width=0.16\textwidth]{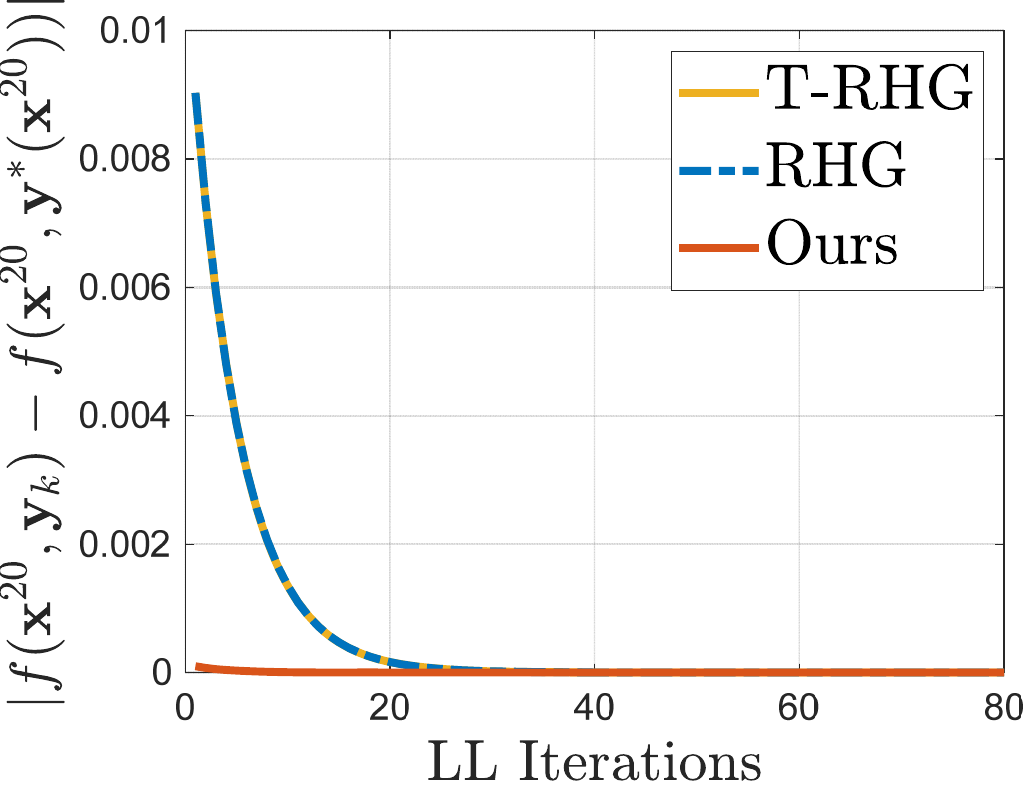}\\
		\includegraphics[width=0.16\textwidth]{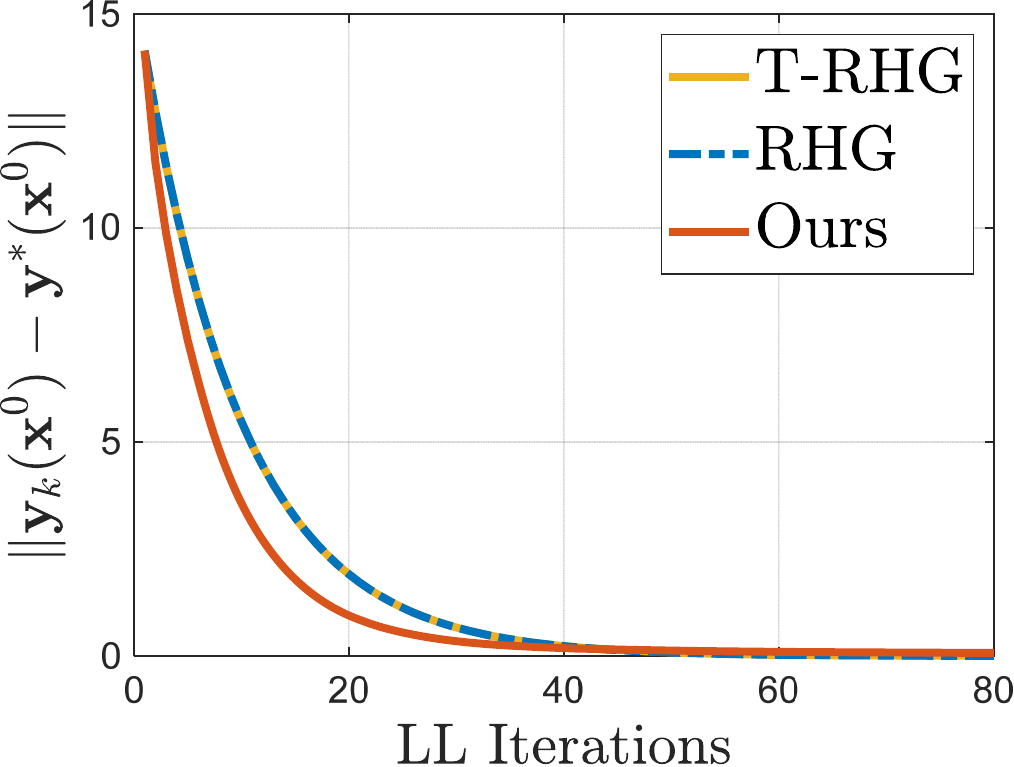}
		&\includegraphics[width=0.16\textwidth]{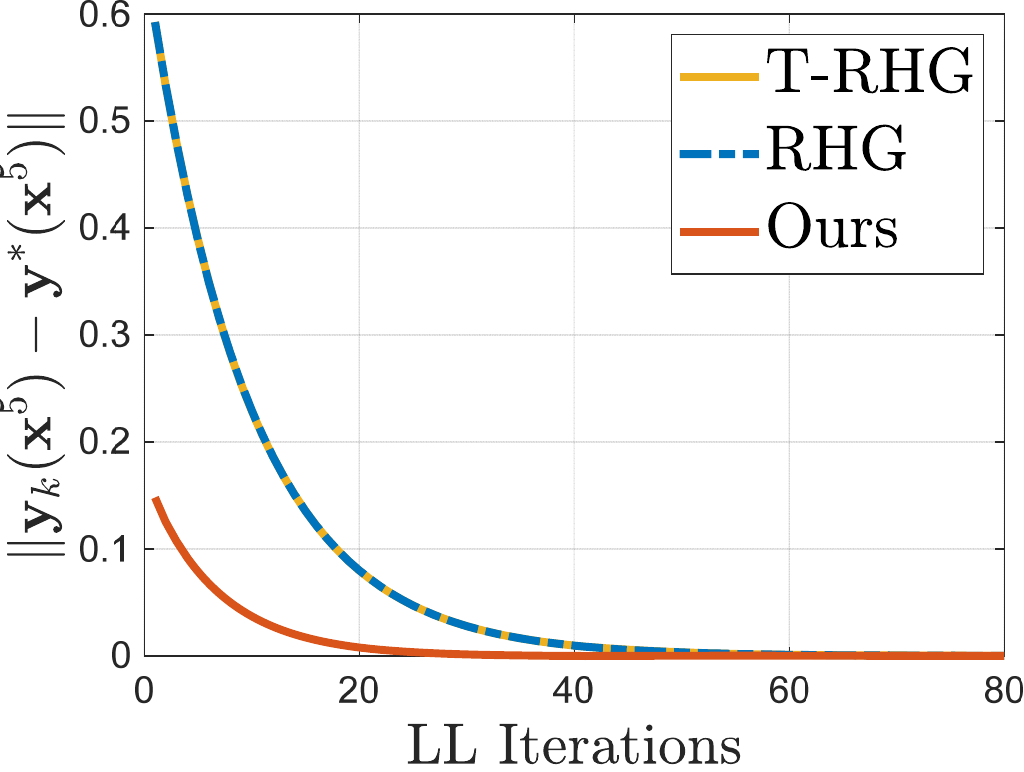}
		&\includegraphics[width=0.16\textwidth]{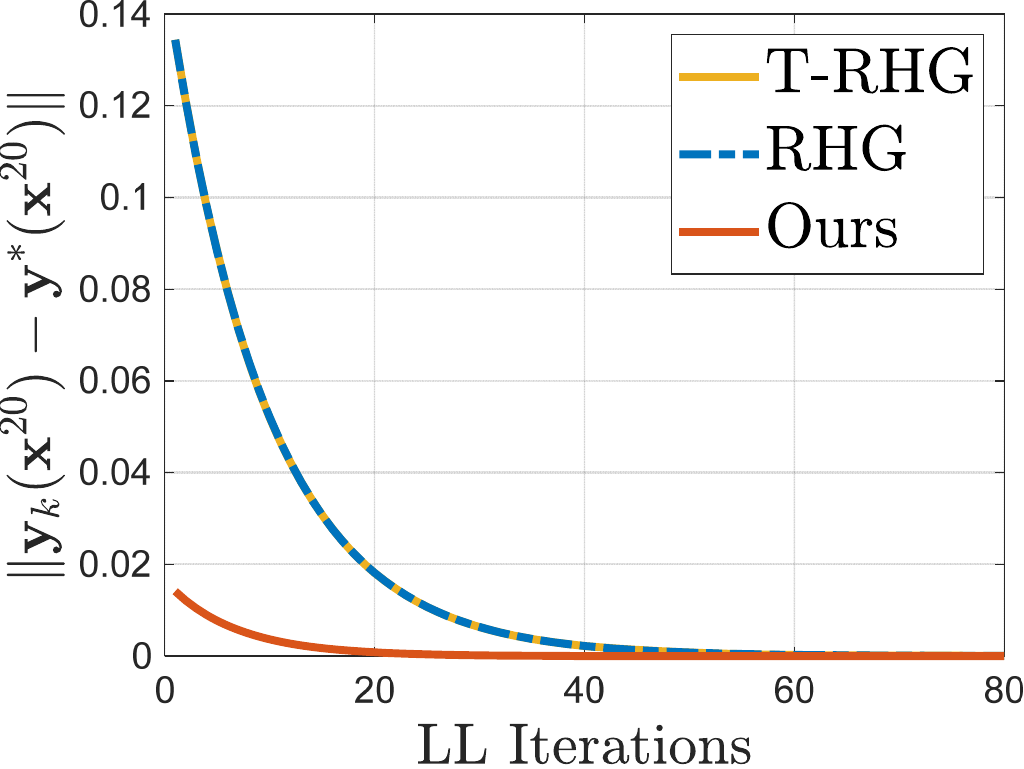}\\
		\includegraphics[width=0.16\textwidth]{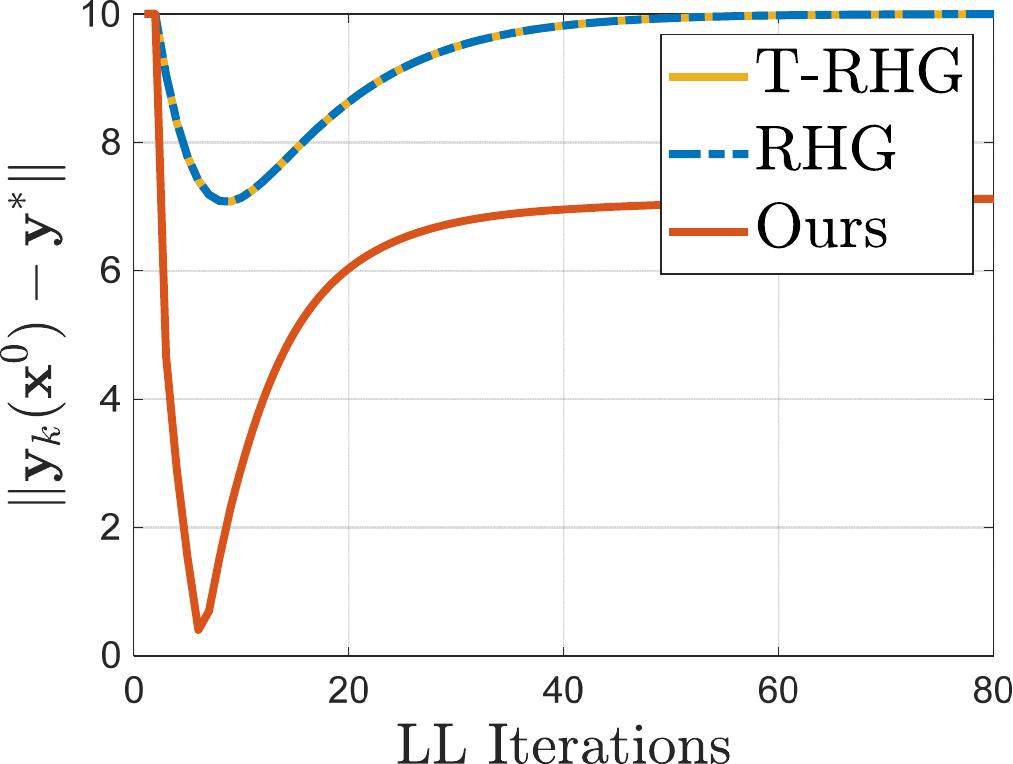}
		&\includegraphics[width=0.16\textwidth]{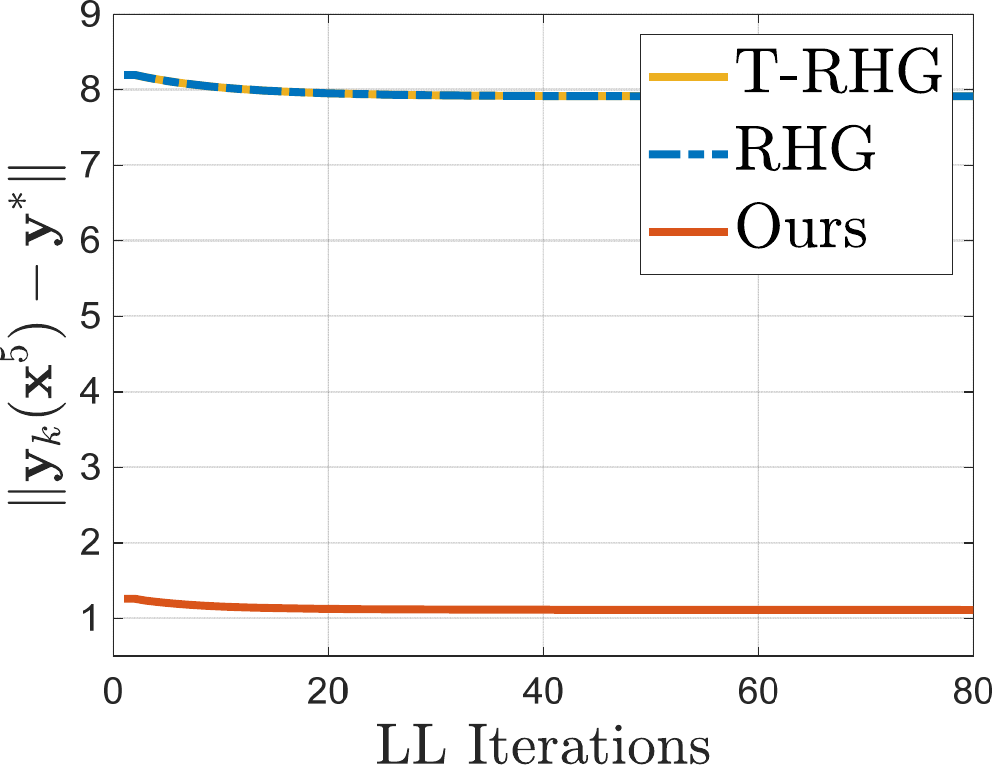}
		&\includegraphics[width=0.16\textwidth]{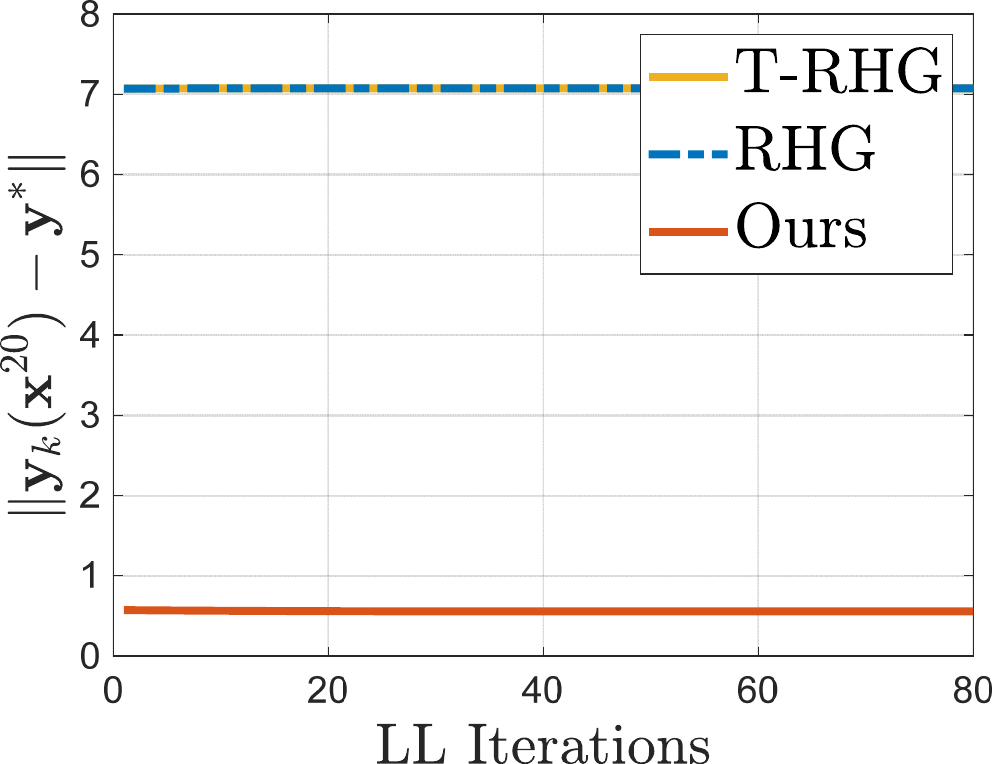}\\
	\end{tabular}
	\caption{LL iteration curves of gradient-based BLO algorithms (T-RHG, RHG and Ours) under three fixed $\x$ (i.e., $\x^0$, $\x^5$, $\x^{20}$). 	$\y^*(\x)$ and $\y^*$ denote the optimal solution with and without relationship about $\x$. } \label{fig:toy_diff_K}
\end{figure}

In Figure~\ref{fig:conver_x}, comparing with RHG, we then demonstrate the optimization procedure of UL variable (i.e., $\x$). In the first subfigure, under fixed LL iterative solution $\y_K$, the UL objective $\varphi_{K}(\x)$ illustrates that our BDA can efficiently fit the optimal objective function (i.e., $\varphi^*(\x)$) for any UL variable. To further demonstrate the convergence behavior, we plotted the errors of the UL objective (i.e., $|\varphi_K(\x)-\varphi(\x)|$) and variable (i.e., $\|\x-\x^*\|$) in the second and third subfigures. With the above illustration, we summarize the relationship of Optimal solution (short for ``Opt.'', the red one in the last subfigure) with the iterative solutions of RHG and BDA in the last subfigure. Thus, we conclude that our BDA can find the optimal point, while RHG converge to a non-optimal point in $\S(\x)$. 

\subsection{Comparison with the Work in~\cite{liu2020generic}}\label{subsec:comp_conference}	
	First of all, this work significantly improves the assumptions required by our convergence analysis. That is, we successfully remove the strong convexity property on the UL objective, the level-bounded in $\y$ and locally uniformly in $\x\in\X$ property on the LL objective. Furthermore, we replace the essential condition ``LL solution set property'' by ``LL objective convergence property'', which is much weaker and easily verifiable. In this way, we actually obtain a more general and feasible proof recipe for challenging real-world applications.
	
	We also extend our convergence results to other optimization scenarios, such as local and stationary results. Specifically, we obtain convergence results for the case that there are only local solutions to the UL approximation (i.e., ``$\min_{\x}\varphi_K(\x)$''). Moreover, we provide new methodology to analyze the convergence behaviors of our BDA in the scenario that we can only obtain the stationary points for the UL approximation. Therefore, this work has comprehensively analyzed the convergence behaviors for our BDA in various (i.e., global, local and stationary) optimization scenarios. 
			
	Algorithmically, this work establishes a more general framework, in which we introduce the projection-based operations to handle set constrains in BLOs and design more flexible strategy to set the aggregation parameters during iterations. We also established a one-stage fast approximation to BDA for solving large-scale real-world problems. 
	
	We design a high-dimensional counter-example (see Eq.~(40)) and conduct various experiments to verify our new theoretical findings, i.e., the efficiency of BDA for BLOs without the LLS condition and both the UL and LL objectives are convex but not strongly convex. We further do new experiments to more clearly analyze the components of BDA and report more results on real applications (e.g., with new evaluation metrics, on more challenging benchmarks and compared with more state-of-the-art approaches).

\subsection{One-stage BDA: A Fast Implementation}\label{sec:extension}

Multi-step of the LL iteration modules $\T_k$ will cause a lot of memory consumption that may be an obstacle in modern massive-scale deep learning applications. Thus it would be useful to simplify iteration steps. This part provides an extension scheme leveraging a one-stage simplification to reduce complicated gradient-based calculation steps~\cite{liu2018darts}. By setting $K = 1$ in Eq.~\eqref{eq:improved-lower}, the algorithm reads as 
\begin{equation}\label{eq:one-stage}
\begin{aligned}
\y_1(\x) = \T_{1}\left(\x,\y_0\right)&=\mathtt{Proj}_{\Y}\left( \y_0 -s\partial_{\y}\phi(\x,\y_0) \right),
\end{aligned}
\end{equation}
where $\phi(\x,\y_0)= \alpha F(\mathbf{x},\mathbf{y}_{0})+\beta f(\mathbf{x},\mathbf{y}_{0})$ and $\alpha,\beta\in(0,1]$ denote the aggregation parameters. Indeed, if $(\y_0-s\partial_{\y}\phi(\x,\y_0))\in\Y$, with this one-stage simplification, we can simplify the back-propagation calculation with the following finite difference approximation
\begin{equation*}
\begin{array}{l}
\frac{d\varphi_{1}(\x)}{d\x} 
= \frac{\partial F(\x,\y_1)}{\partial \x} + \frac{\partial F(\x,\y_1)}{\partial \y_1}\frac{d\y_1}{d \x} \\
\quad\quad\quad \approx \frac{\partial F(\x,\y_1)}{\partial \x} -s \frac{\partial_{\x}\phi(\x,\mathbf{h}_0^{+}) -\partial_{\x}\phi(\x,\mathbf{h}_0^{-}) }{2\epsilon},
\end{array}
\end{equation*}
where $\mathbf{h}_0^{\pm}=\y_0\pm\epsilon\partial F(\x,\y_1)/\partial \y_1$ and $\partial_{\x}\phi(\x,\y)= \alpha\partial_{\x}F(\x,\y)+\beta\partial_{\x} f(\x,\y)$. Since $\Y$ can be a big interval, this case (i.e., $(\y_0-s\partial_{\y}\phi(\x,\y_0))\in\Y$) is often satisfied in general. If $(\y_0-s\partial_{\y}\phi(\x,\y_0))\notin\Y$, the above back-propagation can be calculated by the following form
\begin{equation*}
\begin{array}{l}
\frac{d\varphi_{1}(\x)}{d\x}
\approx\frac{\partial F(\x,\y_1)}{\partial \x}\\ \quad\quad\quad -\frac{\partial_{\x}\phi(\x,\mathbf{h}_0^{++})-\partial_{\x}\phi(\x,\mathbf{h}_0^{-+})-\left(\partial_{\x}\phi(\x,\mathbf{h}_0^{+-})-\partial_{\x}\phi(\x,\mathbf{h}_0^{--})\right)}{4\epsilon^{1+\frac{1}{2}}},
\end{array}
\end{equation*}
where $\mathbf{h}_{0}^{\pm+}=\y_0\pm \epsilon\mathtt{Proj}_{\Y}\left( \z_0+\epsilon^{1/2}\partial F(\x,\y_1)/\partial\y \right) $ and $\mathbf{h}_0^{\pm-}=\y_0\pm \epsilon\mathtt{Proj}_{\Y}\left( \z_0-\epsilon^{1/2}\partial F(\x,\y_1)/\partial\y \right)$ with $\z_0=\y_0-s\partial_{\y}\phi(\x,\y_0)$.
                                   
\begin{figure}[t]
	\centering 
	\begin{tabular}{c@{\extracolsep{0.6em}}c}
		\includegraphics[width=0.23\textwidth]{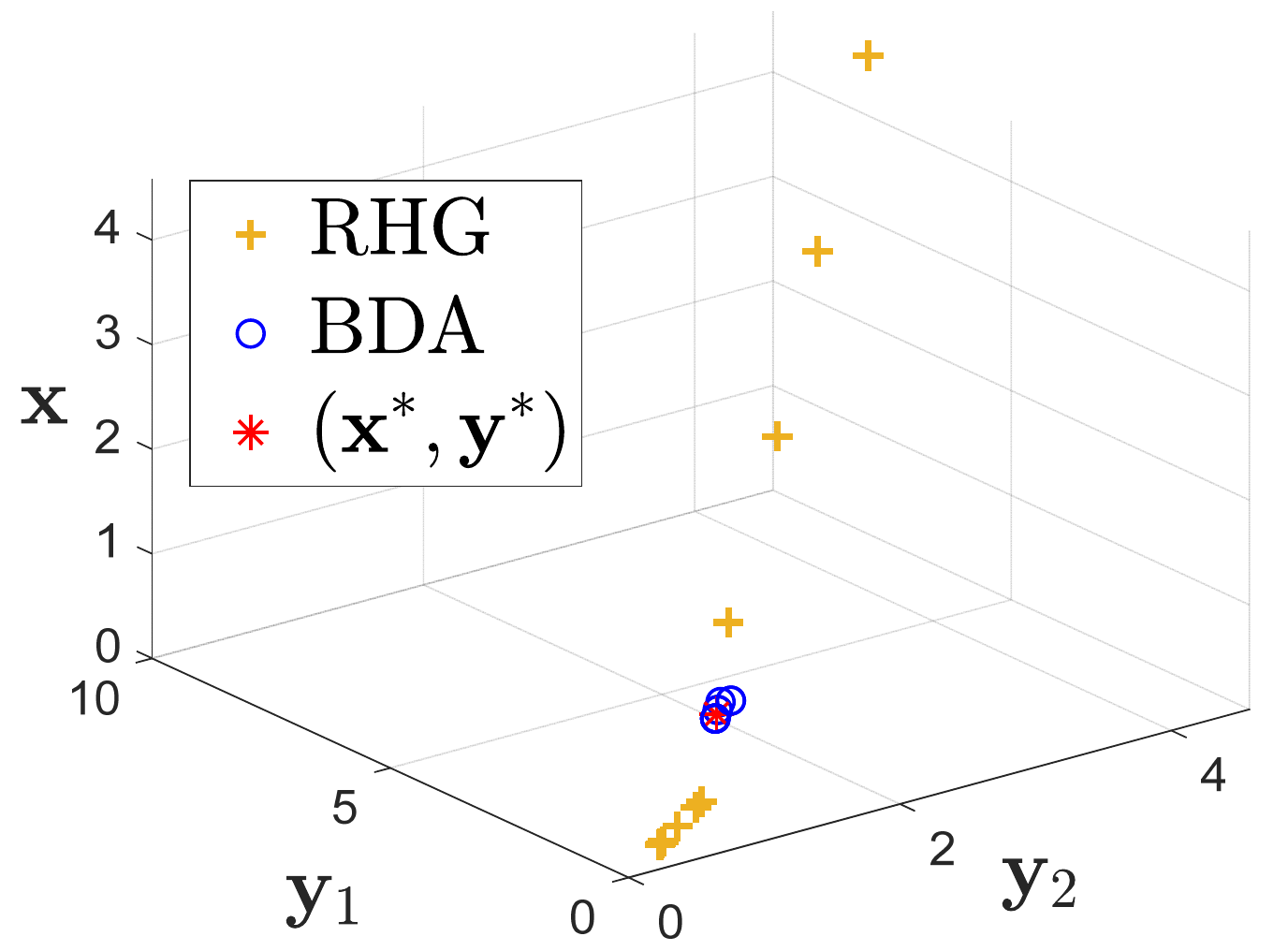}
		&\includegraphics[width=0.23\textwidth]{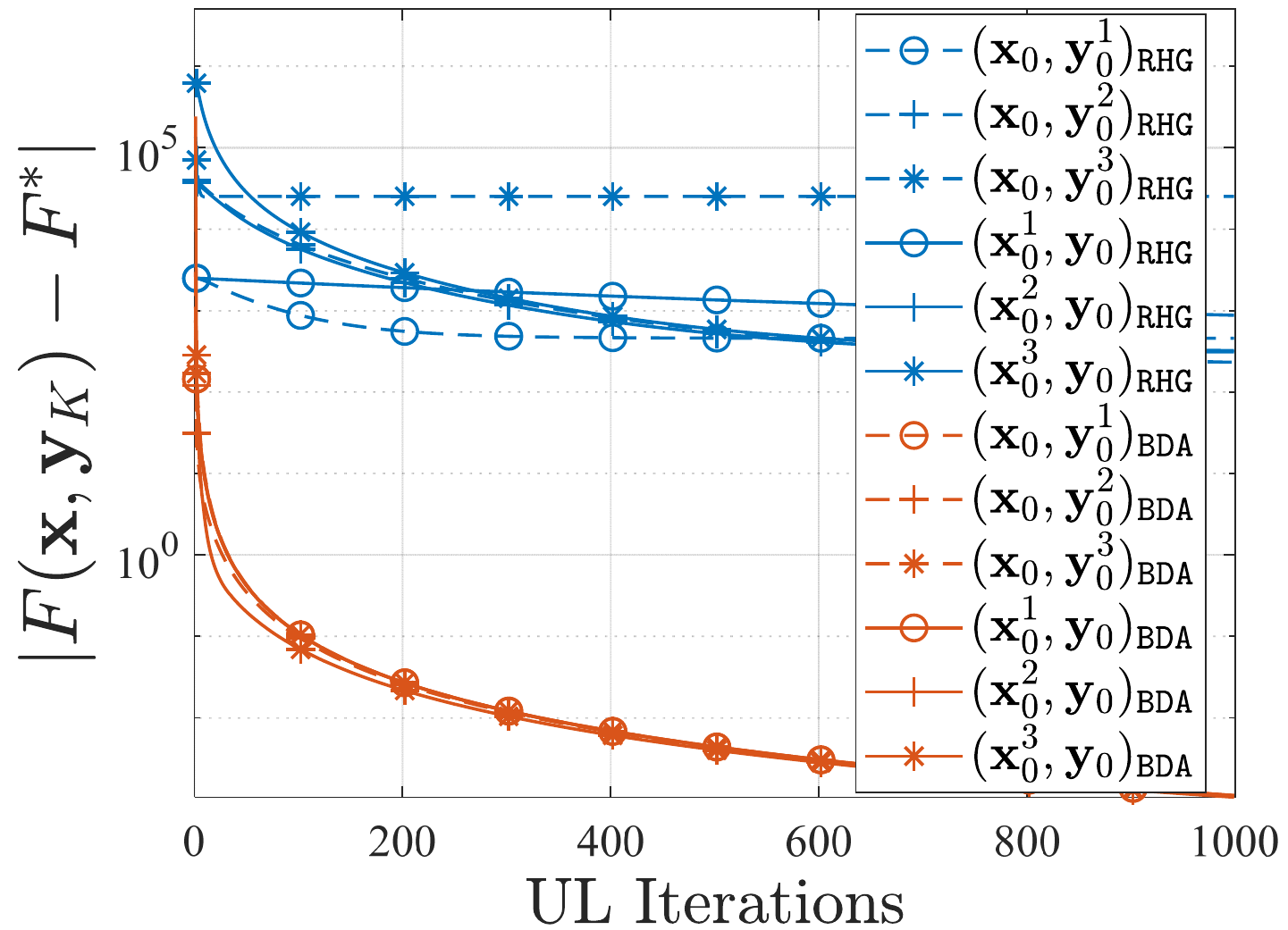}\\
	\end{tabular}
	\caption{Comparisons of BDA with RHG on ten different initial points. We set the dimensional $n=50$ and $K=20$. The left subfigure show the iteration solution of different initial points. We select five different initial points and show the UL objective behavior on the right subfigure. } \label{fig:toy_diff_initial}
\end{figure}

\begin{figure}[t]
	\centering 
	\begin{tabular}{c@{\extracolsep{0.2em}}c}
		\includegraphics[width=0.226\textwidth]{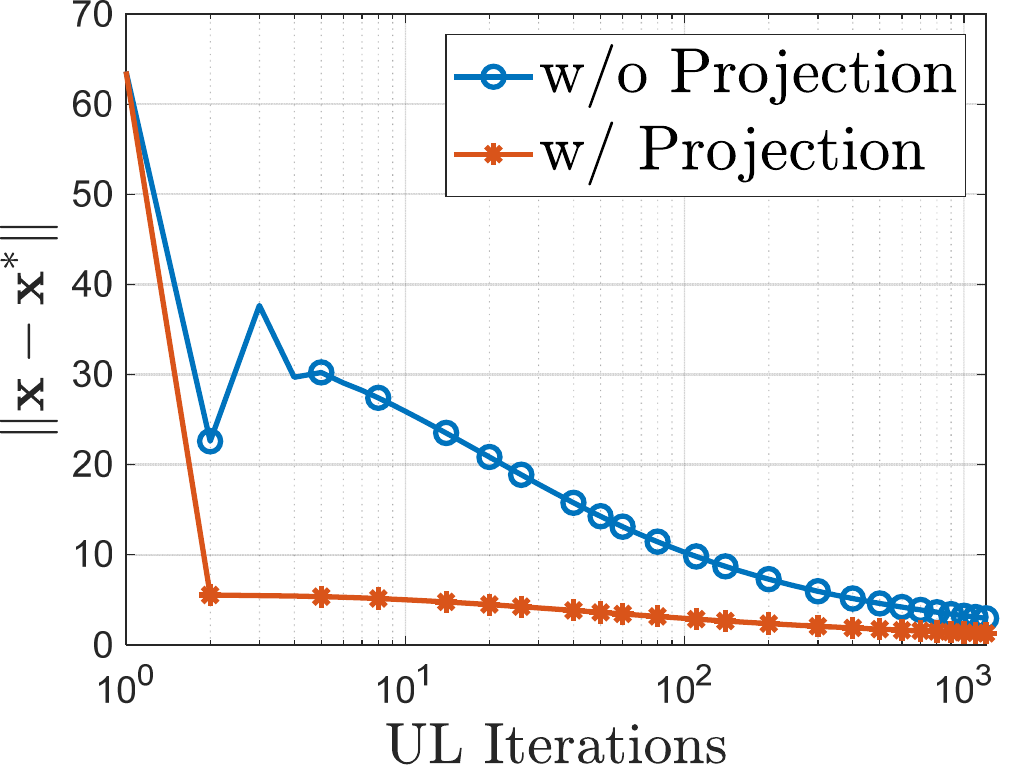}
		&\includegraphics[width=0.23\textwidth]{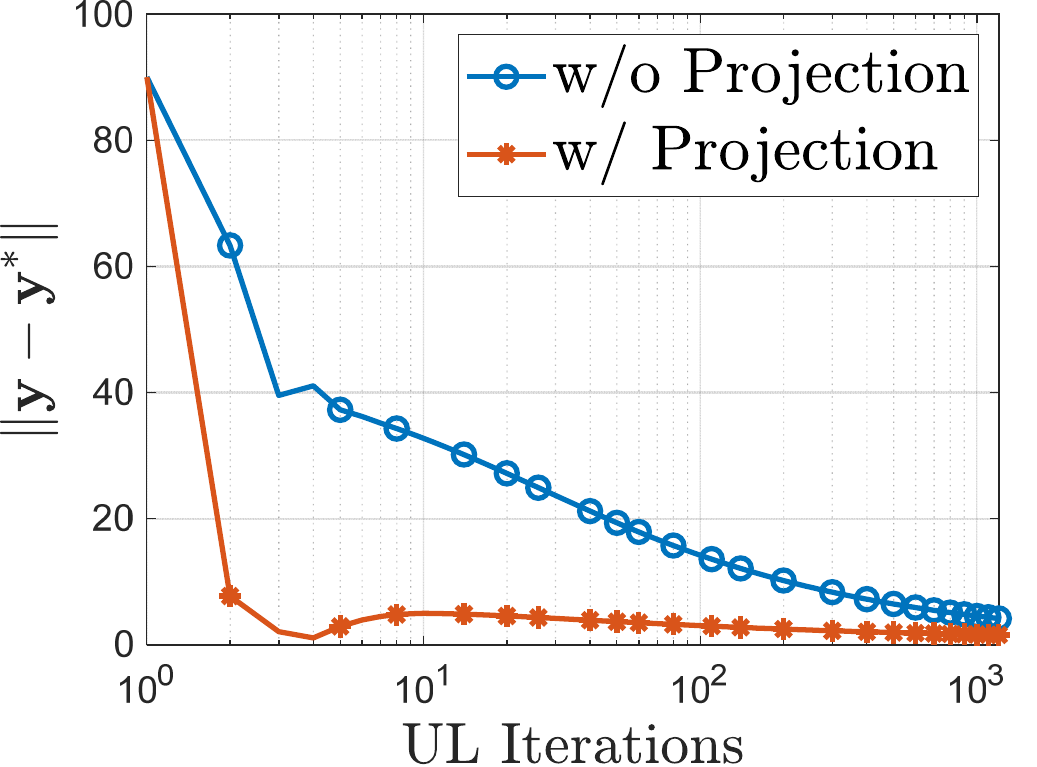}\\
	\end{tabular}
	\caption{Comparing of our BDA under different settings, i.e., with and without projection operator (namely w/ Projection and w/o Projection). We set $n=50$ and $K=20$.} \label{fig:proj}
\end{figure}

\section{Experimental Results}\label{sec:experiments}

This section first verify the numerical results and then evaluate the performance of our proposed method on different problems. 

\subsection{Numerical Evaluations}\label{subsec:numerical_results}

Our numerical results are investigated based on the synthetic BLO described in Section~\ref{subsec:comp_LLS}, i.e., Counter Example in Eq.~\eqref{eq:ce-2}. As stated in Section~\ref{subsec:comp_LLS}, this deterministic bi-level formulation satisfies all the assumptions required in Section~\ref{sec:conv_recipes}, but it cannot meet the LLS condition considered in~\cite{finn2017model,franceschi2017forward,franceschi2018bilevel,shaban2018truncated,rajeswaran2019meta}.

To show the influence of the LL iterations (i.e., $K$) on different methods, we first plotted the convergence behaviors (i.e., $|f(\x^t,\y_k)-f(\x^t,\y^*(\x^t))|$, $\|\y_k(\x^t)-\y^*(\x^t)\|$ and $\|\y_k(\x^t)-\y^*\|$ with $t = 0,5,20$) under different given $\x$ (i.e., $\x^0$, $\x^5$, $\x^{20}$) in Figure~\ref{fig:toy_diff_K}. This figure compare our BDA with the most popular bi-level based methods (i.e., T-RHG and RHG). Note that $t = 0,5,20$ are the UL iteration steps during the operation process. From the first and second row of Figure~\ref{fig:toy_diff_K}, we observed that with fixed UL variable $\x$, the results of RHG and BDA converge to the optimal solution with corresponding given $\x^t$. The third row of Figure~\ref{fig:toy_diff_K} plotted the distance between the current iteration step and the optimal solution $\y^*$. As can be seen, after a few UL iteration steps (i.e., $t\geq 5$), BDA is close to the optimal solution $\y^*$ while RHG and T-RHG cannot. In the above figures, we set $\alpha_k=0.5/k$, $k= 1,\cdots, K$, $s_u=s_l=0.1$, $\beta_k=1$, $\mu=0.1$.

\begin{figure}[t]
	\centering 
	\begin{tabular}{c@{\extracolsep{0.2em}}c}
		\includegraphics[width=0.226\textwidth]{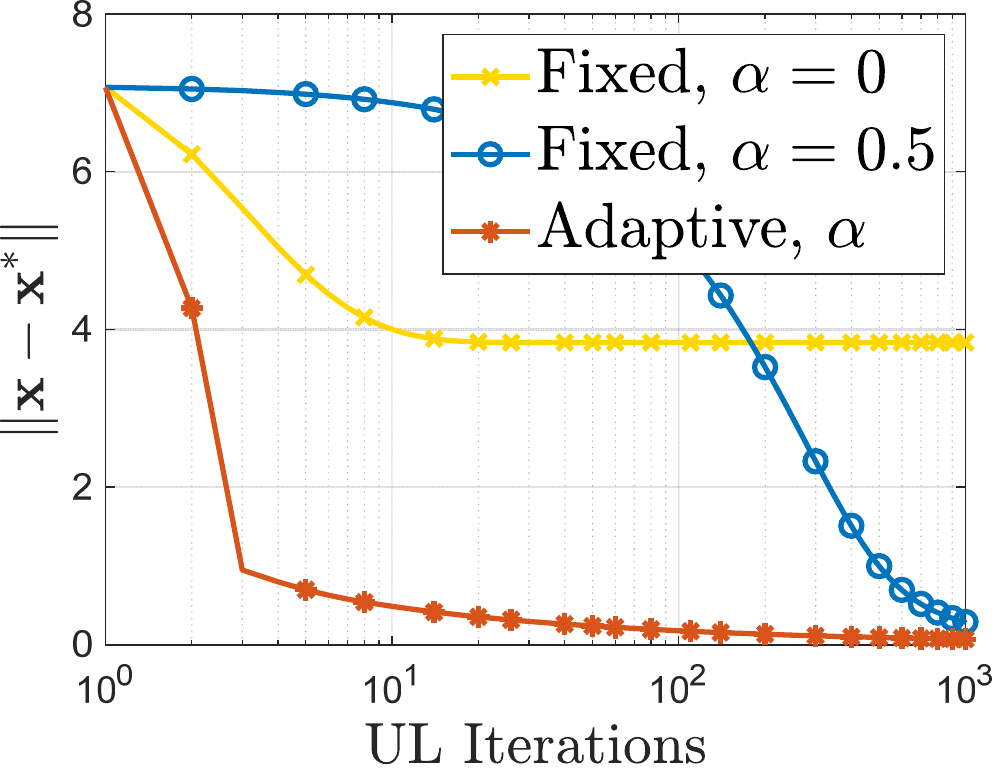}
		&\includegraphics[width=0.23\textwidth]{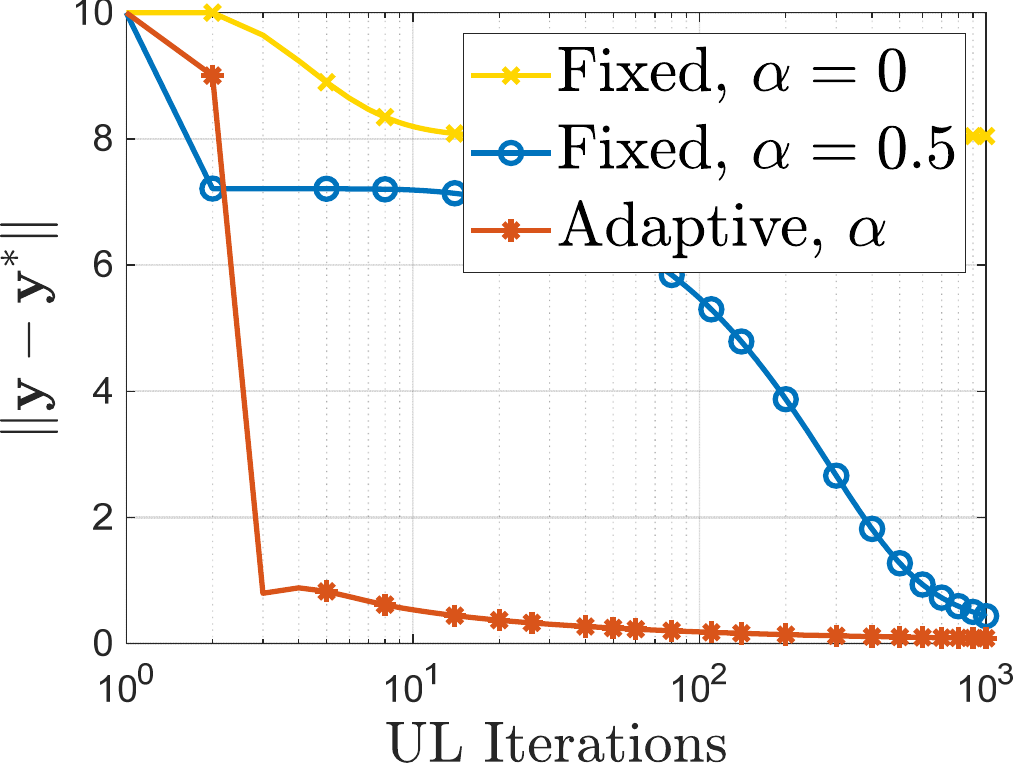}\\
	\end{tabular}
	\caption{The iteration curves of the developed BDA with different $\alpha$ settings (i.e., with Fixed $\alpha=0$, $\alpha=0.5$ and Adaptive $\alpha=0.5/k$). We set $n=50$ and $K=20$.} \label{fig:diff_alpha}
\end{figure}

\begin{table*}[t]
	\caption{Data hyper-cleaning accuracy of the compared methods on two different datasets, i.e., MNIST~\cite{lecun1998gradient} and Fashion MNIST~\cite{xiao2017fashion}.} 
	\label{tab:HyperClearning_2}
	\centering
	\renewcommand\arraystretch{1.1}
	\setlength{\tabcolsep}{3mm}{
		\begin{tabular}{ | c | c | c | c | c | c | c | c | c |}
			\hline
			\multirow{2}{*}{Methods} & \multicolumn{4}{c|}{MNIST} & \multicolumn{4}{c|}{Fashion MNIST} \\
			\cline{2-9}
			& Val. Acc. & Test Acc. & F1-Score & Time(s) & Val. Acc. & Test Acc. & F1-Score & Time(s)\\
			\hline\hline
			IHG & 86.98 &  87.69 & 87.62 & 12.14 $\pm\ 0.73$ & 82.66  &  83.82	& 83.63 & 10.92 $\pm$ 0.75 \\
			\hline
			RHG  & 88.08 &  88.30 & 88.20  & 3.72 $\pm\ $0.01 & 85.12  & 86.14 & 86.04  & 1.09 $\pm$ 0.01  \\
			\hline
			T-RHG  & 88.30 &  86.16 & 88.10  & 2.49 $\pm\ $0.07 & 85.12 &  86.06 & 86.07  & 0.63 $\pm$ 0.01 \\
			\hline		
			O-BDA & \textbf{ 88.84} &  88.45 & 88.37  & 2.89 $\pm\ $0.01 & \textbf{86.34} & 86.16 & 86.05 & 0.66 $\pm$ 0.01\\
			\hline
			BDA &88.26 &\textbf{ 88.47}	& \textbf{88.42}  & 7.82 $\pm\ $0.12 & 85.28 & \textbf{86.26} & \textbf{86.17} & 1.91 $\pm$ 0.01 \\
			\hline
		\end{tabular}
	}
\end{table*}

Figure~\ref{fig:toy_diff_initial} plotted numerical results of the proposed BDA and RHG~\cite{franceschi2017forward,franceschi2018bilevel} with ten different initialization points. We considered different numerical metrics, such as the relationship of $(\x,\y)$ with optimal solution $(\x^*,\y^*)$ and the distance between $F(\x,\y_K)$ and $F^*$ (i.e., $|F(\x,\y_K)-F^*|$), for evaluations. It needs to be noted that we select five different initial points to show the performance of $|F(\x,\y_K)-F^*|$. As can be observed that RHG is always hard to obtain the correct solution, even start from different initialization points. It is mainly because that the solution set of the LL subproblem in Eq.~\eqref{eq:ce-2} is not a singleton, which does not satisfy the fundamental assumption of RHG. In contrast, the proposed method can obtain a truly optimal solution in all these scenarios. The initialization only slightly affects the convergence speed of our iterative sequences.

To explore the performance under projection operator denoted in Eq.~\eqref{eq:improved-lower}, we report in Figure \ref{fig:proj} the results (i.e., $\|\x-\x^*\|$ and $\|\y-\y^*\|$) of comparing the performance with and without projection (i.e., w/ Projection, w/o Projection). In this experiment, we set the initial value far away from the optimal point with relatively close projection interval $\Y$. As can be seen, with the projection operator, the iteration sequences reach convergence with fewer steps.

Figure~\ref{fig:diff_alpha} evaluated the convergence behaviors of BDA with different choices of $\alpha_k$. We set $\beta_k=1$, $\mu=0.5$, $s_u=0.1$ and $s_l=0.1$. By setting $\alpha_k=0$, we were unable to use the UL information guiding the LL updating. Thus it is hard to obtain proper feasible solutions for the UL approximation subproblem. When choosing a fixed $\alpha_k$ in $(0,1)$ (e.g., $\alpha_k=0.5$), the numerical performance can be improved but the convergence speed was still slow. Fortunately, we followed our theoretical findings and introduced an adaptive strategy to incorporate UL information into LL iterations, leading to nice convergence behaviors for both UL and LL variables.

\begin{table*}[t]
	\caption{Averaged accuracy scores $\pm$ standard deviation of various methods (model-based methods and gradient-based bi-level methods) on few-shot classification classification problems (1-shot and 5-shot, i.e., $M=1,5$, $N=5,20,30,40$) on Omniglot. }
	\label{tab:omniglot}
	\centering 	
	\renewcommand\arraystretch{1.2}
	\setlength{\tabcolsep}{1mm}{
		\begin{tabular}{|c|c|c|c|c|c|c|c|c|}
			\hline
			 \multirow{2}{*}{Method}& \multicolumn{2}{c|}{ $5$-way}&\multicolumn{2}{c|}{$20$-way}&\multicolumn{2}{c|}{$30$-way} &\multicolumn{2}{c|}{$40$-way}\\
			\cline{2-9}
			  &$1$-shot & $5$-shot &  $1$-shot &  $5$-shot & $1$-shot &  $5$-shot & $1$-shot &  $5$-shot\\
			\hline\hline
			MAML &98.70 $\pm$ $0.40\%$&\textbf{99.91} $\pm$ $0.10\%$ &95.80 $\pm$ $0.30\%$ &98.90 $\pm$ $0.20\%$ & 86.86 $\pm\ 0.49\%$ & 96.86 $\pm\ 0.19\%$ & 85.98 $\pm\ 0.45\%$& 94.46 $\pm\ 0.13\%$  \\
			\hline
			Meta-SGD  & 97.97 $\pm\ 0.70\%$ & 98.96 $\pm\ 0.20\%$ & 93.98 $\pm\ 0.43\%$& 98.42 $\pm\ 0.11\%$ & 89.91 $\pm\ 0.04\%$ & 96.21 $\pm 0.15\%$ & 87.39 $\pm\ 0.43\%$ & 95.10 $\pm\ 0.15\%$\\
			\hline
			Reptile & 97.68 $\pm\ 0.04\%$ & 99.48 $\pm\ 0.06\%$ & 89.43 $\pm\ 0.14\%$ & 97.12 $\pm\ 0.32\%$ & 85.40 $\pm\ 0.30\%$ & 95.28 $\pm\ 0.30\%$ & 82.50 $\pm\ 0.30\%$ & 92.79 $\pm\ 0.33\%$\\
			\hline
			iMAML,GD  & \textbf{99.16} $\pm$ $ 0.35\%$ & 99.67 $\pm$  0.12$\% $ & 94.46 $\pm$ 0.42$\%$ & 98.69 $\pm \ 0.10\%$ & 89.52 $\pm\ 0.20\%$ & 96.51 $\pm\ 0.08\%$ & 87.28 $\pm\ 0.21\%$ & 95.27 $\pm\ 0.08\%$ \\
			\hline\hline
			RHG & 98.64 $\pm$ $ 0.21\%$ & 99.58 $\pm$ $ 0.12\%$ & 96.13 $\pm$ $ 0.20\%$ & 99.09 $\pm$ $ 0.08\%$  & 93.92 $\pm\ 0.18\%$ & 98.43 $\pm\ 0.08\%$  & 90.78 $\pm\ 0.20\%$ & 96.79 $\pm\ 0.10\%$\\
			\hline
			T-RHG  & 98.74 $\pm$ $ 0.21\%$& 99.71 $\pm$ $ 0.07\%$ & 95.82 $\pm$ $ 0.20\%$ & 98.95 $\pm$ $ 0.07\%$& 94.02 $\pm\ 0.18\%$ & 98.39 $\pm\ 0.07\%$ & 90.73 $\pm\ 0.20\%$  & 96.79 $\pm\ 0.10\%$ \\
			\hline
			BDA & 99.04 $\pm$ $ 0.18\%$ & 99.74 $\pm$ $ 0.05\%$ & \textbf{96.50} $\pm$ $ 0.16\%$ & \textbf{99.19} $\pm$ $ 0.07\%$ & \textbf{94.37} $\pm\ 0.18\%$ & \textbf{98.53} $\pm\ 0.07\%$ & \textbf{92.49} $\pm\ 0.18\%$ & \textbf{97.12} $\pm$ 0.09 $\%$\\
			\hline
		\end{tabular}
	}
\end{table*}

\subsection{Hyper-parameter Optimization}\label{subsec:hyper_parameter}

For the hyper-parameter optimization problem, the key idea is to choose a set of optimal hyper-parameters for a given machine learning task. In this experiment, we consider a specific hyper-parameter optimization example (i.e., data hyper-cleaning~\cite{franceschi2017forward,shaban2018truncated}) to evaluate the developed bi-level algorithm. This task aims to train a linear classifier on a given image set, but part of the training labels are corrupted. Here we consider soft-max regression (with parameters $\y$) as our classifier and introduce hyper-parameters $\x$ to weight samples for training. We define the LL objective as the following weighted training loss:  
\begin{equation*}
f(\x,\y)=\sum_{(\mathbf{u}_i,\mathbf{v}_i)\in\mathcal{D}_{\mathtt{tr}}}[\sigma(\x)]_i\ell(\y;\mathbf{u}_i,\mathbf{v}_i),\end{equation*}
where $\x$ is the hyper-parameter vector to penalize the objective for different training samples, $\ell(\y;\mathbf{u}_i,\mathbf{v}_i)$ means the cross-entropy function with the classification parameter $\y$, and data pairs $(\mathbf{u}_i,\mathbf{v}_i)$ and denote $\mathcal{D}_{\mathtt{tr}}$ and $\mathcal{D}_{\mathtt{val}}$ as the training and validation sets, respectively. Here $\sigma(\x)$ denotes the element-wise sigmoid function on $\x$ and is used to constrain the weights in the range $[0,1]$. For the UL subproblem, we define the objective as the cross-entropy loss with $\ell_2$ regularization on the validation set, i.e., 
\begin{equation*}
F(\x,\y)=\sum_{(\mathbf{u}_i,\mathbf{v}_i)\in\mathcal{D}_{\mathtt{val}}}\ell(\y(\x);\mathbf{u}_i,\mathbf{v}_i).
\end{equation*}
In particular, the UL and LL objective $F$ and $f$ w.r.t. $\y$ is required to be convex. To satisfy this requirement, we design the classifier with a fully connected layer.

We applied our BDA and One-stage BDA (O-BDA) together with the bi-level based methods, i.e., Implicit HG (IHG) \cite{pedregosa2016hyperparameter}, RHG and Truncated RHG (T-RHG) \cite{shaban2018truncated}. We first conduct the experiment on two datasets (MNIST dataset \cite{lecun1998gradient} and Fashion MNIST dataset \cite{xiao2017fashion}) that each with 5000 training examples (i.e., $\mathcal{D}_{\mathtt{tr}}$), 5000 validation examples (i.e., $\mathcal{D}_{\mathtt{val}}$) and a test set with the remaining 60000 samples. We randomly chose 2500 training samples from $\mathcal{D}_{\mathtt{tr}}$ and pollute the labels. 

We use validation accuracy (i.e., Val. Acc.), test accuracy (i.e., Test Acc.), F1-score and running times as the metrics of our developed algorithm. As shown in Table \ref{tab:HyperClearning_2}, the developed method perform the best both on MNIST and Fashion MNIST dataset. The LL iterations are $K=200$ and $K=50$ on MNIST and Fashion MNIST, respectively. For T-RHG, we chose 100-step and 25-step truncated back-propagation respectively from $K=200$ and $K=50$ to guarantee its convergence. Besides, the developed O-BDA still perform better when comparing with the existing bi-level based methods.

\begin{table}[htb]
	\caption{The few-shot classification performances on MiniImageNet ($N=5$ and $M=1$). The second column reported the averaged accuracy after converged. The rightmost two columns compared the UL Iterations (denoted as ``UL Iter.''), when achieving almost the same accuracy ($\approx 44\%$). Here  ``Ave. $\pm$ Var. (Acc.)'' denotes the averaged accuracy and the corresponding variance.}
	\label{tab:mini}
	\renewcommand\arraystretch{1.1}
	\centering 
	\setlength{\tabcolsep}{4.5mm}{
		\begin{tabular}{|c | c ||c| c |}
			\hline
			Method & Acc. & Ave. $\pm$ Var. (Acc.)  & UL Iter. \\
			\hline\hline
			RHG & $48.89$  & 44.46 $\pm$ 0.78$\%$& 3300\\
			\hline
			T-RHG & $47.67$ & 44.21 $\pm$ 0.78$\%$ & 3700\\
			\hline
			PBDA & \textbf{49.08} & 44.24 $\pm$ 0.79$\%$ & \textbf{2500} \\
			\hline
		\end{tabular}
	}
\end{table} 

\subsection{Meta-Learning}\label{subsec:meta_learning}

Meta-learning aims to leverage a large number of similar few-shot tasks to learn an algorithm that should work well on novel tasks in which only a few labeled samples are available. In particular, we consider the few-shot learning problem \cite{vinyals2016matching,qiao2018few}, where each task is to discriminate $N$ separate classes and it is to learn the hyper-parameter $\x$ such that each task can be solved only with $M$ training samples (i.e., $N$-way $M$-shot). Following the experimental protocol used in recent works that the network architecture is with four-layer CNNs followed by fully connected layer, we separate the network architecture into two parts: the cross-task intermediate representation layers (parameterized by $\x$) outputs the meta features and the multinomial logistic regression layer (parameterized by $\y^j$) as our ground classifier for the $j$-th task. We also collect a meta training data set $\mathcal{D}=\{\mathcal{D}^j\}$, where $\mathcal{D}^j=\mathcal{D}_{\mathtt{tr}}^j\cup\mathcal{D}_{\mathtt{val}}^j$ is linked to the $j$-th task. Then for the $j$-th task, we consider the cross-entropy function $\ell(\x,\y^j;\mathcal{D}_{\mathtt{tr}}^j)$ as the task-specific loss and thus the LL objective can be defined as 
\begin{equation*}
\begin{array}{c}
f(\x,\{\y^j\})=\sum\limits_{j}\ell(\x,\y^j;\mathcal{D}_{\mathtt{tr}}^j).
\end{array}
\end{equation*}
As for the UL objective, we also utilize cross-entropy function but define it based on $\{\mathcal{D}_{\mathtt{val}}^j\}$ as 
\begin{equation*}
\begin{array}{c}
F(\x,\{\y^j\})=\sum\limits_{j}\ell(\x,\y^j;\mathcal{D}_{\mathtt{val}}^j).
\end{array}
\end{equation*}

\begin{figure}[t]
	\centering 
	\begin{tabular}{c@{\extracolsep{0.1em}}c}
		\footnotesize Omniglot dataset & \footnotesize MiniImageNet dataset\\
		\includegraphics[width=0.235\textwidth]{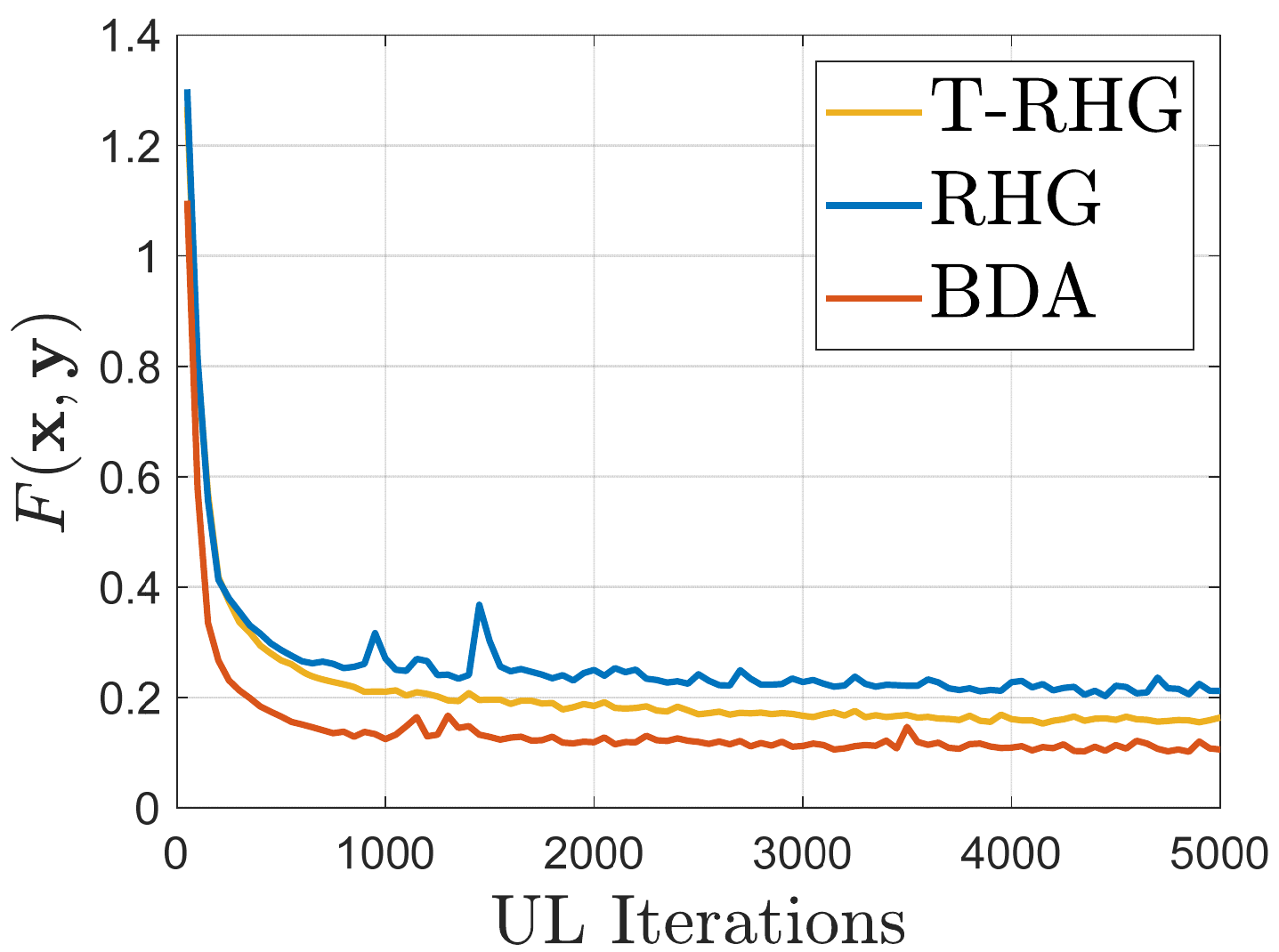}
		&\includegraphics[width=0.235\textwidth]{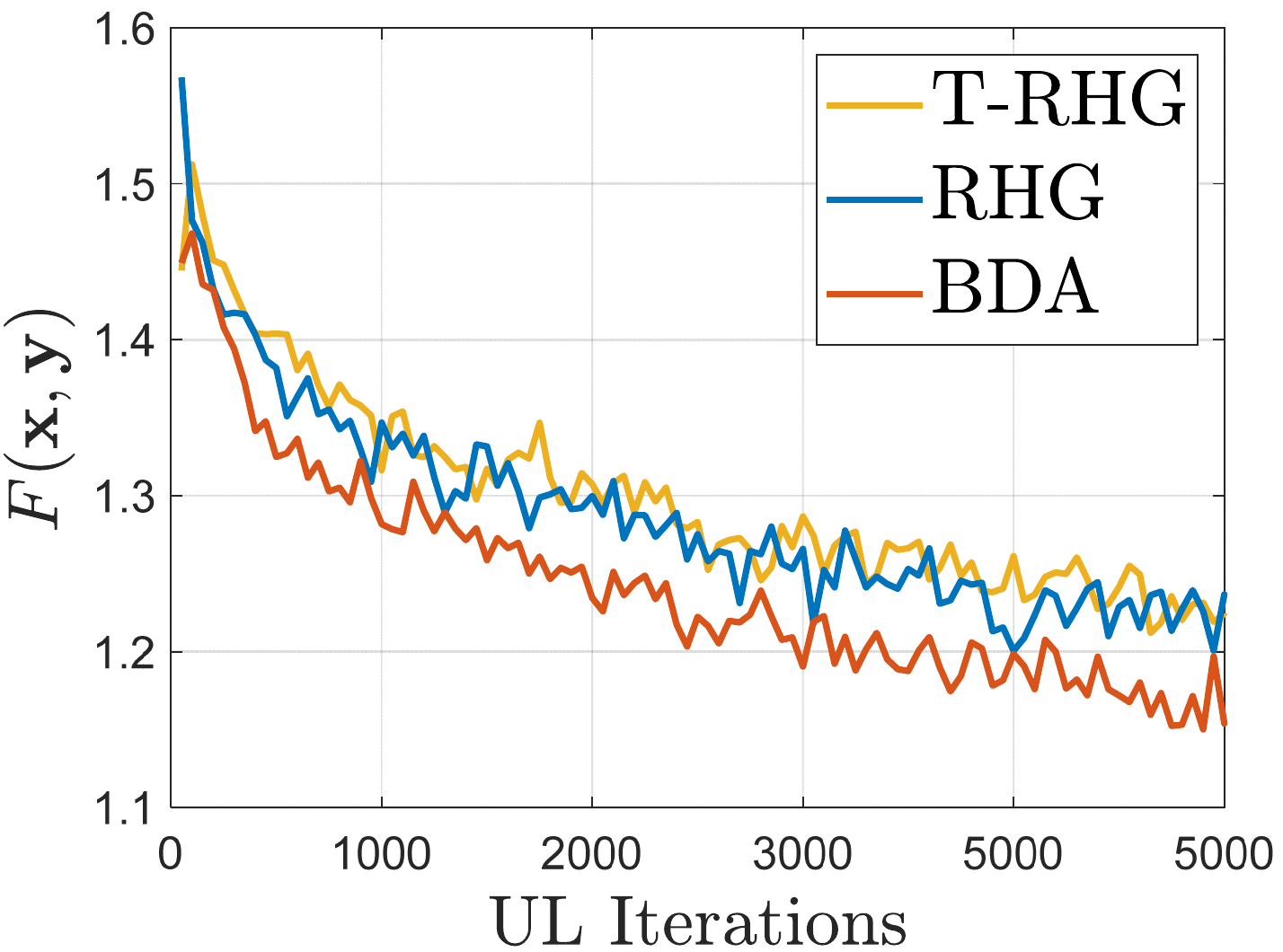}\\
	\end{tabular}
	\caption{Illustrating the validation loss (i.e., UL objectives  $F(\x,\y)$) for three bi-level based methods on few-shot classification task.} \label{fig:dataset_loss}
\end{figure}

\begin{figure}[t]
	\centering 
	\begin{tabular}{c@{\extracolsep{0.1em}}c}
		\includegraphics[width=0.235\textwidth]{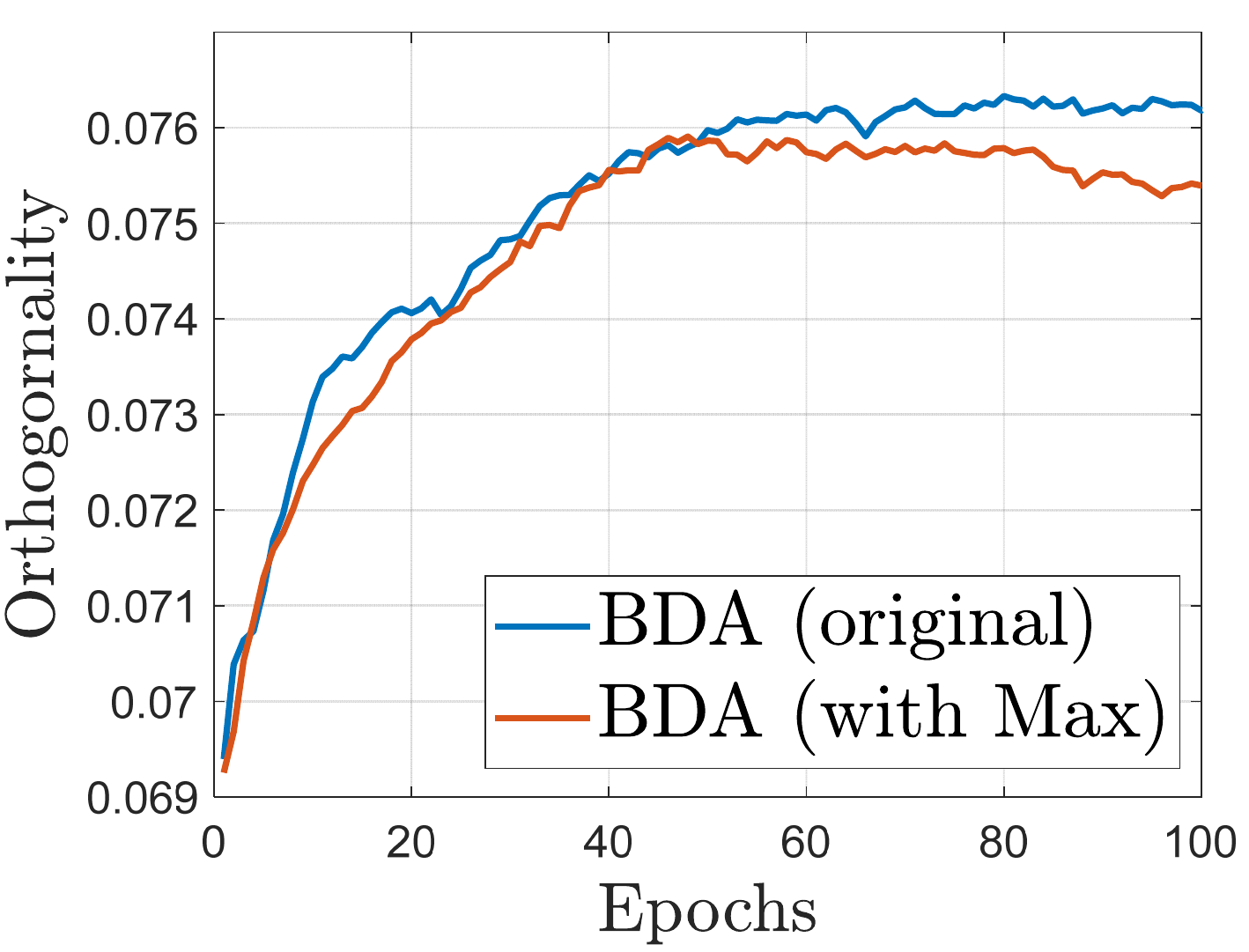}
		&\includegraphics[width=0.235\textwidth]{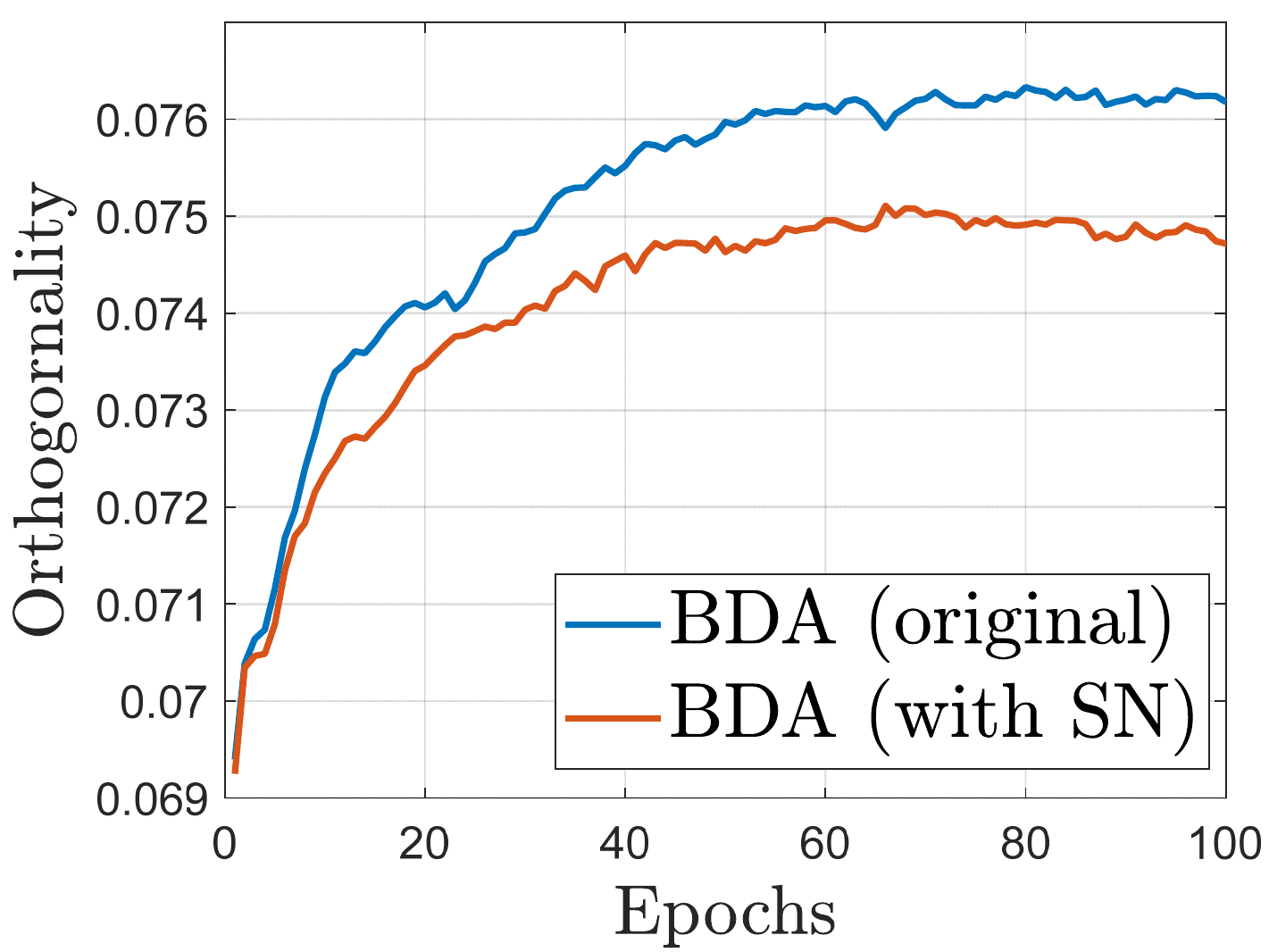}\\
	\end{tabular}
	\caption{Evaluating the orthogonality under different constraints (i.e., with Max-norm regularization and Spectral Normalization (SN for short)) on few-shot application.} \label{fig:about_sn}
\end{figure}

Our experiments are conducted on Ominglot~\cite{lake2015human} and MiniImageNet~\cite{vinyals2016matching} benchmarks. We compared our BDA to several state-of-the-art approaches, such as MAML~\cite{finn2017model}, Meta-SGD~\cite{li2017meta}, Reptile~\cite{nichol2018first}, iMAML~\cite{rajeswaran2019meta}, RHG, and T-RHG. As shown in Table \ref{tab:omniglot}, BDA compared well to these methods and achieved the highest classification accuracy except in the 5-way task. Further, with more complex problems (such as 20-way, 30-way and 40-way), BDA shows significant advantages over other methods. Besides, we evaluate the performance of BDA and bi-level based methods (i.e., RHG and T-RHG) on the more challenging MiniImageNet data set and the corresponding results are listed in Table~\ref{tab:mini}. As shown in the second column of Table \ref{tab:mini} that the developed BDA perform better than RHG and T-RHG. The rightmost two columns demonstrate that BDA needed the fewest iterations to achieve almost the same accuracy ($\approx 44\%$). The corresponding validation loss on Omniglot and MiniImageNet about 5-way 1-shot are shown in Figure \ref{fig:dataset_loss}.

Moreover, to evaluate the effectiveness of the projection operator, we conduct an experiment evaluating orthogonal features of the network with two different strategies (i.e., max-norm regularization and spectral normalization). Note that we compute the orthogonality following \cite{prakash2019repr}. As shown in Figure \ref{fig:about_sn}, BDA with both Max and SN training schemes show the lower orthogonal sum. This experiment implies that the projection operator can help obtain a better network.

\section{Conclusions}

This work established a flexible descent aggregation framework with task-tailored iteration dynamics modules to solve bi-level tasks by formulating BLO in Eq.~\eqref{eq:oblp} from the viewpoint of optimistic bi-level. We provided a new algorithmic framework to handle the LLS issue. Then, this work strictly proved the convergence of the developed framework without the LLS assumption and the strong convexity in the UL objective. Focusing on different solution qualities (namely, global, local, and stationarity), this work elaborated the convergence results respectively. Finally, extensive experiments justified our theoretical results and demonstrated the superiority of the proposed algorithm for hyper-parameter optimization and meta-learning.

\section*{Acknowledgements}

This work is partially supported by the National Key R\&D Program of China (2020YFB1313503), the National Natural Science Foundation of China (Nos. 61922019 and 11971220), the Fundamental Research Funds for the Central Universities, the Shenzhen Science and Technology Program (No. RCYX20200714114700072), the Stable Support Plan Program of Shenzhen Natural Science Fund (No. 20200925152128002), the Guangdong Basic and Applied Basic Research Foundation and the Pacific Institute for the Mathematical Sciences (PIMS).

\nocite{langley00}

\bibliographystyle{IEEEtran}
\bibliography{reference}

\begin{thebibliography}{10}
\providecommand{\url}[1]{#1}
\csname url@samestyle\endcsname
\providecommand{\newblock}{\relax}
\providecommand{\bibinfo}[2]{#2}
\providecommand{\BIBentrySTDinterwordspacing}{\spaceskip=0pt\relax}
\providecommand{\BIBentryALTinterwordstretchfactor}{4}
\providecommand{\BIBentryALTinterwordspacing}{\spaceskip=\fontdimen2\font plus
\BIBentryALTinterwordstretchfactor\fontdimen3\font minus
  \fontdimen4\font\relax}
\providecommand{\BIBforeignlanguage}[2]{{%
\expandafter\ifx\csname l@#1\endcsname\relax
\typeout{** WARNING: IEEEtran.bst: No hyphenation pattern has been}%
\typeout{** loaded for the language `#1'. Using the pattern for}%
\typeout{** the default language instead.}%
\else
\language=\csname l@#1\endcsname
\fi
#2}}
\providecommand{\BIBdecl}{\relax}
\BIBdecl

\bibitem{liu2020generic}
R.~Liu, P.~Mu, X.~Yuan, S.~Zeng, and J.~Zhang, ``A generic first-order
  algorithmic framework for bi-level programming beyond lower-level
  singleton,'' in \emph{ICML}, 2020, pp. 6305--6315.

\bibitem{rajeswaran2019meta}
A.~Rajeswaran, C.~Finn, S.~M. Kakade, and S.~Levine, ``Meta-learning with
  implicit gradients,'' in \emph{NeurIPS}, 2019, pp. 113--124.

\bibitem{mackay2019self}
M.~MacKay, P.~Vicol, J.~Lorraine, D.~Duvenaud, and R.~Grosse, ``Self-tuning
  networks: Bilevel optimization of hyperparameters using structured
  best-response functions,'' \emph{ICLR}, 2019.

\bibitem{kunisch2013bilevel}
K.~Kunisch and T.~Pock, ``A bilevel optimization approach for parameter
  learning in variational models,'' \emph{SIAM Journal on Imaging Sciences},
  vol.~6, no.~2, pp. 938--983, 2013.

\bibitem{liu2021investigating}
R.~Liu, J.~Gao, J.~Zhang, D.~Meng, and Z.~Lin, ``Investigating bi-level
  optimization for learning and vision from a unified perspective: A survey and
  beyond,'' \emph{CoRR, abs/2101.11517}, 2021.

\bibitem{dempe2018bilevel}
S.~Dempe, \emph{Bilevel optimization: theory, algorithms and
  applications}.\hskip 1em plus 0.5em minus 0.4em\relax TU Bergakademie
  Freiberg Mining Academy and Technical University, 2018.

\bibitem{dempe2019two}
S.~Dempe, B.~S. Mordukhovich, and A.~B. Zemkoho, ``Two-level value function
  approach to non-smooth optimistic and pessimistic bilevel programs,''
  \emph{Optimization}, vol.~68, no. 2-3, pp. 433--455, 2019.

\bibitem{dempe2007new}
S.~Dempe, J.~Dutta, and B.~Mordukhovich, ``New necessary optimality conditions
  in optimistic bilevel programming,'' \emph{Optimization}, vol.~56, no. 5-6,
  pp. 577--604, 2007.

\bibitem{kohli2012optimality}
B.~Kohli, ``Optimality conditions for optimistic bilevel programming problem
  using convexifactors,'' \emph{Journal of Optimization Theory and
  Applications}, vol. 152, no.~3, pp. 632--651, 2012.

\bibitem{lampariello2020numerically}
L.~Lampariello, S.~Sagratella \emph{et~al.}, ``Numerically tractable optimistic
  bilevel problems.'' \emph{Computational Optimization and Applications},
  vol.~76, no.~2, pp. 277--303, 2020.

\bibitem{Liu2021TowardsGB}
R.~Liu, Y.~Liu, S.~Zeng, and J.~Zhang, ``Towards gradient-based bilevel
  optimization with non-convex followers and beyond,'' \emph{CoRR,
  abs/2110.00455}, 2021.

\bibitem{franceschi2018bilevel}
L.~Franceschi, P.~Frasconi, S.~Salzo, R.~Grazzi, and M.~Pontil, ``Bilevel
  programming for hyperparameter optimization and meta-learning,'' in
  \emph{ICML}, 2018, pp. 1563--1572.

\bibitem{zugner2019adversarial}
D.~Z{\"u}gner and S.~G{\"u}nnemann, ``Adversarial attacks on graph neural
  networks via meta learning,'' \emph{ICLR}, 2019.

\bibitem{franceschi2017forward}
L.~Franceschi, M.~Donini, P.~Frasconi, and M.~Pontil, ``Forward and reverse
  gradient-based hyperparameter optimization,'' in \emph{ICML}, 2017, pp.
  1165--1173.

\bibitem{okuno2018hyperparameter}
T.~Okuno, A.~Takeda, and A.~Kawana, ``Hyperparameter learning via bilevel
  nonsmooth optimization,'' \emph{CoRR, abs/1806.01520}, 2018.

\bibitem{yang2019provably}
Z.~Yang, Y.~Chen, M.~Hong, and Z.~Wang, ``Provably global convergence of
  actor-critic: A case for linear quadratic regulator with ergodic cost,'' in
  \emph{NeurIPS}, 2019, pp. 8351--8363.

\bibitem{liu2018darts}
H.~Liu, K.~Simonyan, and Y.~Yang, ``Darts: Differentiable architecture
  search,'' in \emph{ICLR}, 2019.

\bibitem{wu2019fbnet}
B.~Wu, X.~Dai, P.~Zhang, Y.~Wang, F.~Sun, Y.~Wu, Y.~Tian, P.~Vajda, Y.~Jia, and
  K.~Keutzer, ``Fbnet: Hardware-aware efficient convnet design via
  differentiable neural architecture search,'' in \emph{IEEE CVPR}, 2019, pp.
  10\,734--10\,742.

\bibitem{nakai2020att}
K.~Nakai, T.~Matsubara, and K.~Uehara, ``Att-darts: Differentiable neural
  architecture search for attention,'' in \emph{IEEE IJCNN}, 2020, pp. 1--8.

\bibitem{hu2020tf}
Y.~Hu, X.~Wu, and R.~He, ``Tf-nas: Rethinking three search freedoms of
  latency-constrained differentiable neural architecture search,'' in
  \emph{ECCV}, 2020, pp. 123--139.

\bibitem{de2017bilevel}
J.~C. De~los Reyes, C.-B. Sch{\"o}nlieb, and T.~Valkonen, ``Bilevel parameter
  learning for higher-order total variation regularisation models,''
  \emph{Journal of Mathematical Imaging and Vision}, vol.~57, no.~1, pp. 1--25,
  2017.

\bibitem{chen2020flexible}
J.~Chen, P.~Mu, R.~Liu, X.~Fan, and Z.~Luo, ``Flexible bilevel image layer
  modeling for robust deraining,'' in \emph{IEEE ICME}, 2020, pp. 1--6.

\bibitem{Liu2020investigating}
R.~Liu, P.~Mu, J.~Chen, X.~Fan, and Z.~Luo, ``Investigating task-driven latent
  feasibility for nonconvex image modeling,'' \emph{IEEE Transactions on Image
  Processing}, vol.~29, pp. 7629--7640, 2020.

\bibitem{jeroslow1985polynomial}
R.~G. Jeroslow, ``The polynomial hierarchy and a simple model for competitive
  analysis,'' \emph{Mathematical Programming}, vol.~32, no.~2, pp. 146--164,
  1985.

\bibitem{weinan2017proposal}
E.~Weinan, ``A proposal on machine learning via dynamical systems,''
  \emph{Communications in Mathematics and Statistics}, vol.~5, no.~1, pp.
  1--11, 2017.

\bibitem{maclaurin2015gradient}
D.~Maclaurin, D.~Duvenaud, and R.~Adams, ``Gradient-based hyperparameter
  optimization through reversible learning,'' in \emph{ICML}, 2015, pp.
  2113--2122.

\bibitem{shaban2018truncated}
A.~Shaban, C.-A. Cheng, N.~Hatch, and B.~Boots, ``Truncated back-propagation
  for bilevel optimization,'' in \emph{AISTATS}, 2019, pp. 1723--1732.

\bibitem{finn2017model}
C.~Finn, P.~Abbeel, and S.~Levine, ``Model-agnostic meta-learning for fast
  adaptation of deep networks,'' in \emph{ICML}, 2017, pp. 1126--1135.

\bibitem{nichol2018first}
A.~Nichol, J.~Achiam, and J.~Schulman, ``On first-order meta-learning
  algorithms,'' \emph{CoRR, abs/1803.02999}, 2018.

\bibitem{lorraine2018stochastic}
J.~Lorraine and D.~Duvenaud, ``Stochastic hyperparameter optimization through
  hypernetworks,'' \emph{CoRR, abs/1802.09419}, 2018.

\bibitem{grazzi2020iteration}
R.~Grazzi, L.~Franceschi, M.~Pontil, and S.~Salzo, ``On the iteration
  complexity of hypergradient computation,'' in \emph{ICML}, 2020, pp.
  3748--3758.

\bibitem{lorraine2020optimizing}
J.~Lorraine, P.~Vicol, and D.~Duvenaud, ``Optimizing millions of
  hyperparameters by implicit differentiation,'' in \emph{International
  Conference on Artificial Intelligence and Statistics}.\hskip 1em plus 0.5em
  minus 0.4em\relax PMLR, 2020, pp. 1540--1552.

\bibitem{Bertrand2020implicit}
Q.~Bertrand, Q.~Klopfenstein, M.~Blondel, S.~Vaiter, A.~Gramfort, and
  J.~Salmon, ``Implicit differentiation of lasso-type models for hyperparameter
  optimization,'' in \emph{ICML}, 2020, pp. 810--821.

\bibitem{liu2021value}
R.~Liu, X.~Liu, X.~Yuan, S.~Zeng, and J.~Zhang, ``A value-function-based
  interior-point method for non-convex bi-level optimization,'' \emph{CoRR,
  abs/2110.04974}.

\bibitem{pedregosa2016hyperparameter}
F.~Pedregosa, ``Hyperparameter optimization with approximate gradient,'' in
  \emph{ICML}, 2016, pp. 737--746.

\bibitem{fallah2020convergence}
A.~Fallah, A.~Mokhtari, and A.~Ozdaglar, ``On the convergence theory of
  gradient-based model-agnostic meta-learning algorithms,'' in
  \emph{International Conference on Artificial Intelligence and
  Statistics}.\hskip 1em plus 0.5em minus 0.4em\relax PMLR, 2020, pp.
  1082--1092.

\bibitem{bolte2020mathematical}
J.~Bolte and E.~Pauwels, ``A mathematical model for automatic differentiation
  in machine learning,'' in \emph{NeruIPS}, pp. 10\,809--10\,819.

\bibitem{nocedal2006numerical}
J.~Nocedal and S.~Wright, \emph{Numerical optimization}.\hskip 1em plus 0.5em
  minus 0.4em\relax Springer Science \& Business Media, 2006.

\bibitem{wright1992interior}
M.~H. Wright, ``Interior methods for constrained optimization,'' \emph{Acta
  numerica}, vol.~1, pp. 341--407, 1992.

\bibitem{Heinz-MonotoneOperator-2011}
H.~H. Bauschke and P.~L. Combettes, \emph{Convex Analysis and Monotone Operator
  Theory in Hilbert Spaces}.\hskip 1em plus 0.5em minus 0.4em\relax Springer,
  2011, vol. 408.

\bibitem{cabot2005proximal}
A.~{Cabot}, ``Proximal point algorithm controlled by a slowly vanishing term:
  Applications to hierarchical minimization,'' \emph{SIAM Journal on
  Optimization}, vol.~15, no.~2, pp. 555--572, 2005.

\bibitem{beck2017first}
A.~Beck, \emph{First-order methods in optimization}.\hskip 1em plus 0.5em minus
  0.4em\relax SIAM, 2017.

\bibitem{baydin2017automatic}
A.~G. Baydin, B.~A. Pearlmutter, A.~A. Radul, and J.~M. Siskind, ``Automatic
  differentiation in machine learning: a survey,'' \emph{The Journal of Machine
  Learning Research}, vol.~18, no.~1, pp. 5595--5637, 2017.

\bibitem{lecun1998gradient}
Y.~LeCun, L.~Bottou, Y.~Bengio, P.~Haffner \emph{et~al.}, ``Gradient-based
  learning applied to document recognition,'' \emph{Proceedings of the IEEE},
  vol.~86, no.~11, pp. 2278--2324, 1998.

\bibitem{xiao2017fashion}
H.~Xiao, K.~Rasul, and R.~Vollgraf, ``Fashion-mnist: a novel image dataset for
  benchmarking machine learning algorithms,'' \emph{CoRR, abs/1708.07747},
  2017.

\bibitem{vinyals2016matching}
O.~Vinyals, C.~Blundell, T.~Lillicrap, D.~Wierstra \emph{et~al.}, ``Matching
  networks for one shot learning,'' in \emph{NeurIPS}, 2016, pp. 3630--3638.

\bibitem{qiao2018few}
S.~Qiao, C.~Liu, W.~Shen, and A.~L. Yuille, ``Few-shot image recognition by
  predicting parameters from activations,'' in \emph{CVPR}, 2018, pp.
  7229--7238.

\bibitem{lake2015human}
B.~M. Lake, R.~Salakhutdinov, and J.~B. Tenenbaum, ``Human-level concept
  learning through probabilistic program induction,'' \emph{Science}, vol. 350,
  no. 6266, pp. 1332--1338, 2015.

\bibitem{li2017meta}
Z.~Li, F.~Zhou, F.~Chen, and H.~Li, ``Meta-sgd: Learning to learn quickly for
  few shot learning,'' \emph{CoRR, abs/1707.09835}, 2017.

\bibitem{prakash2019repr}
A.~Prakash, J.~Storer, D.~Florencio, and C.~Zhang, ``Repr: Improved training of
  convolutional filters,'' in \emph{IEEE CVPR}, 2019, pp. 10\,666--10\,675.

\end{thebibliography}

%
%
%
%
%

\vspace{-1cm}
\begin{IEEEbiography}[{\includegraphics[width=1in,height=1.25in,clip,keepaspectratio]{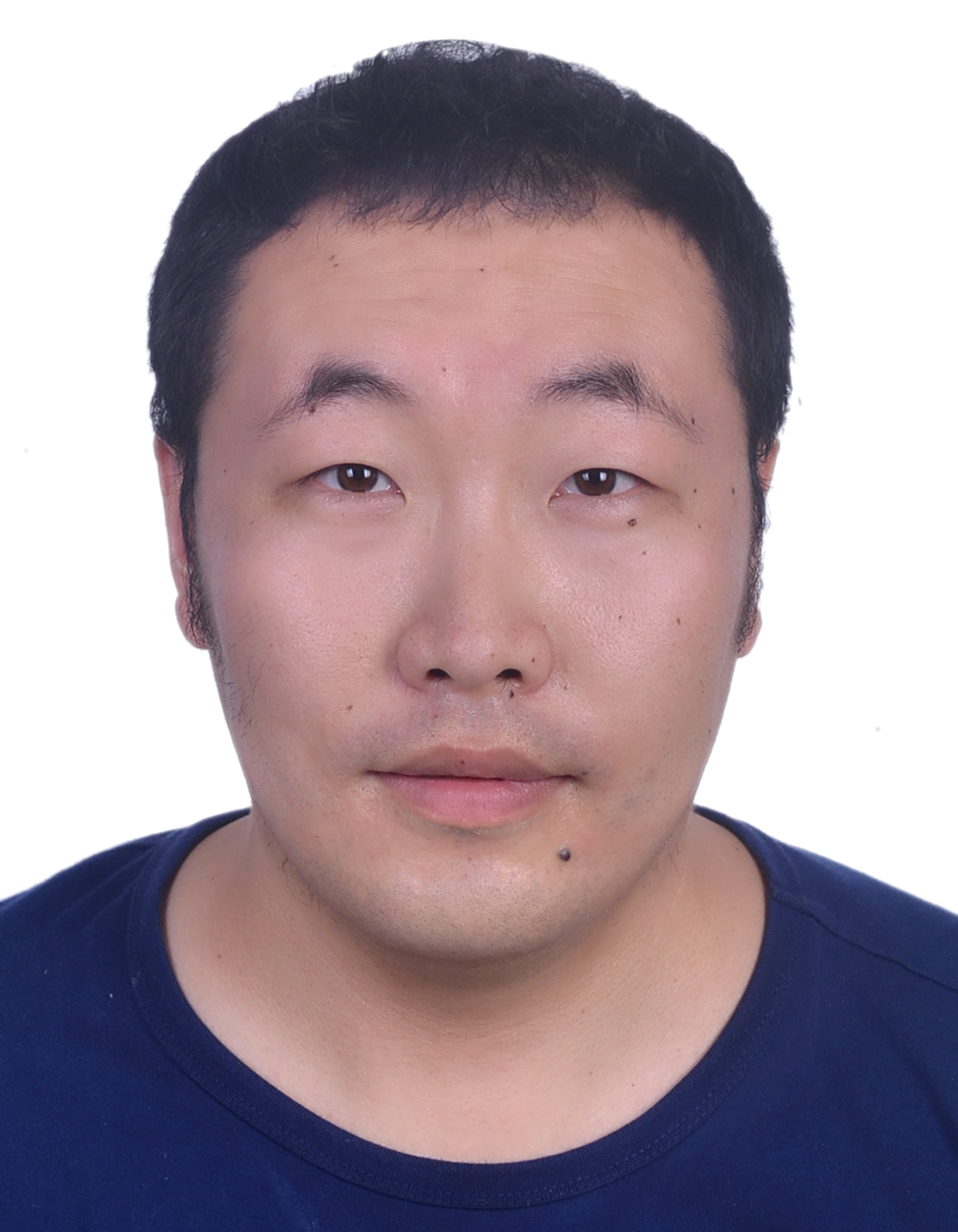}}]{Risheng Liu} received the B.S. and Ph.D. degrees both in mathematics from the Dalian University of Technology in 2007 and 2012, respectively. He was a visiting scholar in the Robotic Institute of Carnegie Mellon University from 2010 to 2012. He served as Hong Kong Scholar Research Fellow at the Hong Kong Polytechnic University from 2016 to 2017. He is currently a professor with DUT-RU International School of Information Science \& Engineering, Dalian University of Technology. 
His research interests include machine learning, optimization, computer vision and multimedia. 
\end{IEEEbiography}
\vspace{-1cm}
\begin{IEEEbiography}[{\includegraphics[width=1in,height=1.25in,clip,keepaspectratio]{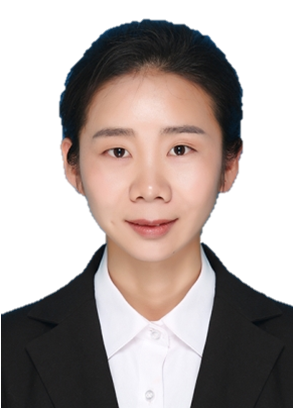}}]{Pan Mu} received the B.S. degree in Applied Mathematics from Henan University, China, in 2014, the M.S. degree in Operational Research and Cybernetics from Dalian University of Technology, China, in 2017. She received the Ph.D. degrees in mathematics from the Dalian University of Technology in 2021. She is currently a lecturer with the College of Computer Science and Technology, Zhejiang University of Technology. Her research interests include computer vision, machine learning, control and optimization. 
\end{IEEEbiography}
\vspace{-1cm}
\begin{IEEEbiography}[{\includegraphics[width=1in,height=1.25in,clip,keepaspectratio]{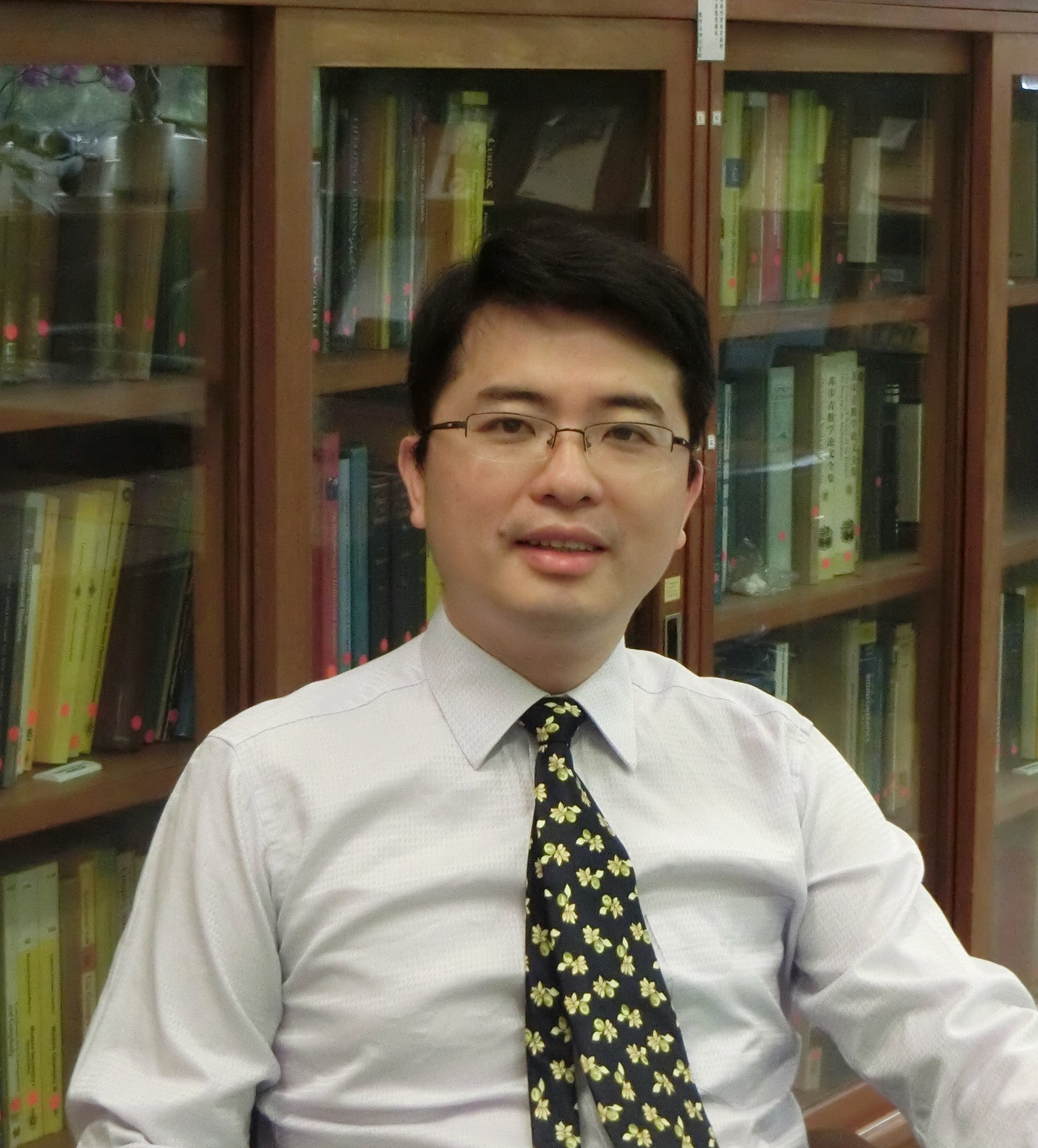}}]{Xiaoming Yuan} is Professor at Department of Mathematics, The University of Hong Kong. His main research interests include numerical optimization, scientific computing and optimal control. Recently, he is particularly interested in optimization problems in various AI and cloud computing areas.
\end{IEEEbiography}
\vspace{-1cm}
\begin{IEEEbiography}[{\includegraphics[width=1in,height=1.25in,clip,keepaspectratio]{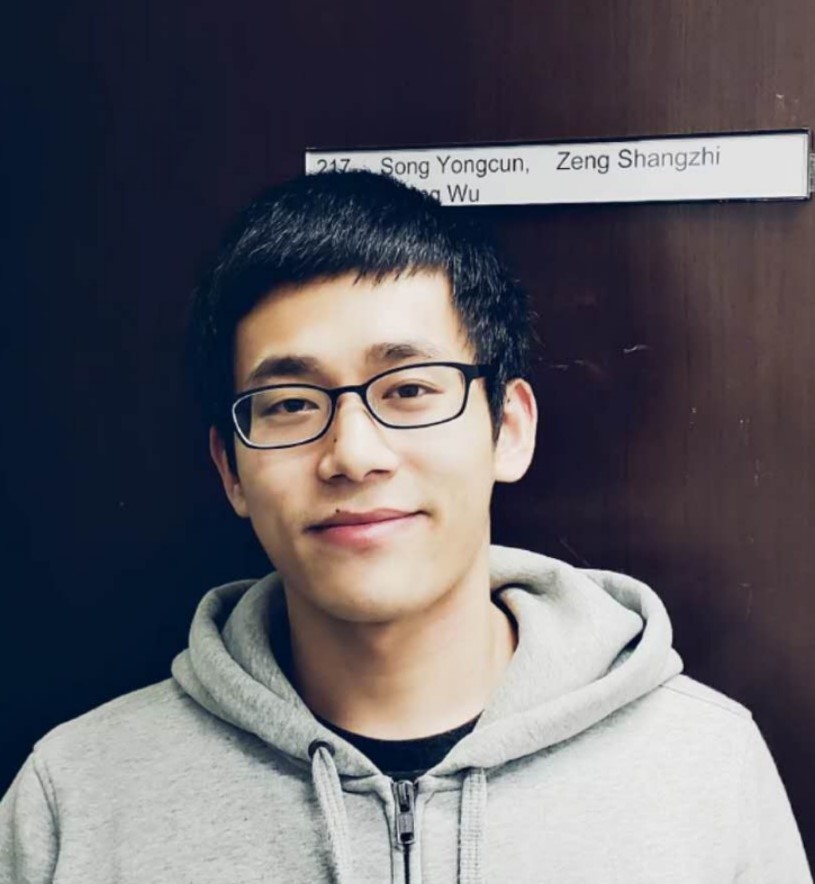}}]{Shangzhi Zeng} received the B.Sc. degree in Mathematics and Applied Mathematics from Wuhan University in 2015, the M.Phil. degree from Hong Kong Baptist University in 2017, and the Ph.D. degree from the University of Hong Kong in 2021. He is currently a PIMS postdoctoral fellow in the Department of Mathematics and Statistics at University of Victoria. His current research interests include variational analysis and bilevel optimization.
\end{IEEEbiography}
\vspace{-1cm}
\begin{IEEEbiography}[{\includegraphics[width=1in,height=1.32in,clip,keepaspectratio]{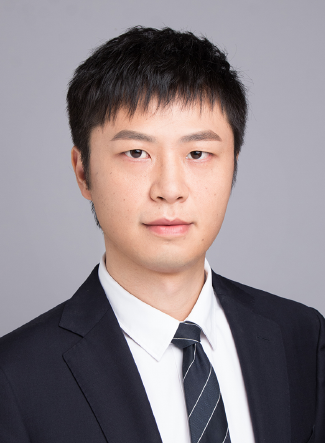}}]{Jin Zhang} received the B.A. degree in Journalism from the Dalian University of Technology in 2007. He received the M.S. degree in mathematics from the Dalian University of Technology, China, in 2010, and the PhD degree in Applied Mathematics from University of Victoria, Canada, in 2015. After working in Hong Kong Baptist University for 3 years, he joined Southern University of Science and Technology as a tenure-track assistant professor in the Department of Mathematics. His research interests include optimization, variational analysis and their applications in economics, engineering and data science.
\end{IEEEbiography}

\end{document}